\documentclass{article}

\usepackage{arxiv}

\usepackage[utf8]{inputenc}
\usepackage[T1]{fontenc}
\usepackage[scaled=.98]{XCharter}
\usepackage[scaled=.95]{helvet}
\usepackage[scaled=1.1]{zlmtt}
\usepackage{amsmath}
\usepackage{amssymb}
\usepackage[uprightscript,charter,vvarbb,scaled=1.05]{newtxmath}
\usepackage[hyphens]{url}
\usepackage[table]{xcolor}
\definecolor{LinkBlue}{HTML}{2457C5}
\definecolor{AbstractGray}{HTML}{F3F5F7}
\definecolor{AbstractBorder}{HTML}{DDE3EA}
\usepackage[
  colorlinks=true,
  linkcolor=LinkBlue,
  citecolor=LinkBlue,
  urlcolor=LinkBlue
]{hyperref}
\usepackage{graphicx}
\usepackage{fontawesome5}
\usepackage{booktabs}
\usepackage{tabularx}
\usepackage{microtype}
\usepackage{enumitem}
\usepackage{etoolbox}
\usepackage[numbers,sort&compress]{natbib}
\usepackage{doi}
\usepackage[most]{tcolorbox}

\newtcolorbox{preprintabstract}{
  enhanced,
  breakable,
  colback=AbstractGray,
  colframe=AbstractBorder,
  boxrule=0.35pt,
  arc=3pt,
  left=12pt,
  right=12pt,
  top=5pt,
  bottom=9pt,
  before skip=10pt,
  after skip=16pt,
  title={Abstract},
  fonttitle=\large\bfseries\sffamily,
  coltitle=black,
  colbacktitle=AbstractGray,
  halign title=center,
  titlerule=0pt
}

\renewenvironment{abstract}
  {\begin{preprintabstract}\fontsize{10pt}{12.4pt}\selectfont}
  {\end{preprintabstract}}

\makeatletter
\renewenvironment{table}
  {\@float{table}}
  {\end@float}
\makeatother

\pdftrailerid{}

\newcommand{\NFastmatmulWThreeOhNineQThirtyTwoOutlierRatio}{1057.4}
\newcommand{\NFastmatmulWThreeOhNineQThirtyTwoGaugeCount}{512}
\newcommand{\NFastmatmulWThreeOhNineQThirtyTwoCatastrophicCount}{8}
\newcommand{\NFastmatmulWThreeOhNineMistralSevenOutlierRatio}{1553.4}
\newcommand{\NFastmatmulWThreeOhNineMistralSevenGaugeCount}{512}
\newcommand{\NFastmatmulWThreeOhNineMistralSevenCatastrophicCount}{0}
\newcommand{\NFastmatmulWThreeOhNineLlamaTwoSevenOutlierRatio}{1760.5}
\newcommand{\NFastmatmulWThreeOhNineLlamaTwoSevenGaugeCount}{512}
\newcommand{\NFastmatmulWThreeOhNineLlamaTwoSevenCatastrophicCount}{0}
\newcommand{\NIclrLadderQfourteenAccPctCintThirtyOne}{35.00}
\newcommand{\NIclrTwinLlEightEqual}{36,288}
\newcommand{\NIclrTwinOlSevenEqual}{36,288}
\newcommand{\NIclrTwinQThreeEightEqual}{40,824}
\newcommand{\NIclrTwinQFourteenEqual}{15,456}
\newcommand{\NIclrWThreeEightNineQThreeEightTstarClean}{1.30}
\newcommand{\NIclrWThreeEightNineQThreeEightNllCleanTOne}{2.129}
\newcommand{\NIclrActTwoQOneFourObqaFillerDnatLibrary}{37.92}
\newcommand{\NIclrActTwoQOneFourObqaFillerDnatLibraryLo}{31.67}
\newcommand{\NIclrActTwoQOneFourObqaFillerDnatLibraryHi}{44.17}
\newcommand{\NIclrActTwoQOneFourObqaFillerDnatOptTwo}{7.92}
\newcommand{\NIclrActTwoQOneFourObqaFillerDnatOptTwoLo}{4.58}
\newcommand{\NIclrActTwoQOneFourObqaFillerDnatOptTwoHi}{11.67}
\newcommand{\NIclrActTwoQOneFourObqaFillerDnatSxc}{7.08}
\newcommand{\NIclrActTwoQOneFourObqaFillerDnatSxcLo}{4.17}
\newcommand{\NIclrActTwoQOneFourObqaFillerDnatSxcHi}{10.42}
\newcommand{\NIclrActTwoQOneFourObqaFillerDnatClean}{0.00}
\newcommand{\NIclrActTwoQOneFourObqaFillerDnatCubic}{0.00}
\newcommand{\NIclrActTwoQOneFourObqaPlausibleDnatLibrary}{34.58}
\newcommand{\NIclrActTwoQOneFourObqaPlausibleDnatLibraryLo}{28.75}
\newcommand{\NIclrActTwoQOneFourObqaPlausibleDnatLibraryHi}{40.83}
\newcommand{\NIclrActTwoQOneFourObqaPlausibleDnatSxc}{5.83}
\newcommand{\NIclrActTwoQOneFourObqaPlausibleDnatSxcLo}{2.92}
\newcommand{\NIclrActTwoQOneFourObqaPlausibleDnatSxcHi}{8.75}
\newcommand{\NIclrActTwoQOneFourObqaPlausibleDnatClean}{0.00}
\newcommand{\NIclrActTwoQOneFourObqaPlausibleDnatCubic}{0.00}
\newcommand{\NIclrActTwoQOneFourArcFillerDnatLibrary}{60.00}
\newcommand{\NIclrActTwoQOneFourArcFillerDnatLibraryLo}{53.75}
\newcommand{\NIclrActTwoQOneFourArcFillerDnatLibraryHi}{66.25}
\newcommand{\NIclrActTwoQOneFourArcFillerDnatOptTwo}{12.08}
\newcommand{\NIclrActTwoQOneFourArcFillerDnatOptTwoLo}{8.33}
\newcommand{\NIclrActTwoQOneFourArcFillerDnatOptTwoHi}{16.25}
\newcommand{\NIclrActTwoQOneFourArcFillerDnatSxc}{6.67}
\newcommand{\NIclrActTwoQOneFourArcFillerDnatSxcLo}{3.75}
\newcommand{\NIclrActTwoQOneFourArcFillerDnatSxcHi}{10.00}
\newcommand{\NIclrActTwoQOneFourArcFillerDnatClean}{0.00}
\newcommand{\NIclrActTwoQOneFourArcFillerDnatCubic}{0.00}
\newcommand{\NIclrActOneQOneFourOrbitSpread}{772.4}
\newcommand{\NIclrActOneQOneFourOrbitMin}{1.367}
\newcommand{\NIclrActOneQOneFourOrbitMedian}{1.425}
\newcommand{\NIclrActOneQOneFourOrbitMax}{1056.0}
\newcommand{\NIclrActOneQOneFourOrbitIdentity}{63.11}
\newcommand{\NIclrActOneQOneFourOrbitN}{512}
\newcommand{\NIclrActOneQOneFourEdgeMaxFactor}{761.0}
\newcommand{\NIclrActOneSevenQThreebMin}{1.124}
\newcommand{\NIclrActOneSevenQThreebMedian}{1.191}
\newcommand{\NIclrActOneSevenQThreebMax}{68.10}
\newcommand{\NIclrActOneSevenQThreebSpread}{60.59}
\newcommand{\NIclrActOneSevenQThreebClassicalFpEightNone}{1.0068}
\newcommand{\NIclrActOneSevenQThreebNAboveTwo}{41}
\newcommand{\NIclrActOneSevenQSevenbMin}{1.099}
\newcommand{\NIclrActOneSevenQSevenbMedian}{1.128}
\newcommand{\NIclrActOneSevenQSevenbMax}{1.18}
\newcommand{\NIclrActOneSevenQSevenbSpread}{1.07}
\newcommand{\NIclrActOneSevenQSevenbClassicalFpEightNone}{1.0054}
\newcommand{\NIclrActOneSevenQSevenbNAboveTwo}{0}
\newcommand{\NIclrActOneSevenQOneFourbClassicalFpEightNone}{1.0138}
\newcommand{\NIclrActOneSevenQOneFourbNAboveTwo}{137}
\newcommand{\NIclrActOneSevenQThreexEightbMin}{1.040}
\newcommand{\NIclrActOneSevenQThreexEightbMedian}{1.105}
\newcommand{\NIclrActOneSevenQThreexEightbMax}{6.15}
\newcommand{\NIclrActOneSevenQThreexEightbSpread}{5.91}
\newcommand{\NIclrActOneSevenQThreexEightbClassicalFpEightNone}{1.0014}
\newcommand{\NIclrActOneSevenQThreexEightbNAboveTwo}{14}
\newcommand{\NIclrActOneSevenLlEightbMin}{1.262}
\newcommand{\NIclrActOneSevenLlEightbMedian}{1.377}
\newcommand{\NIclrActOneSevenLlEightbMax}{2.89}
\newcommand{\NIclrActOneSevenLlEightbSpread}{2.29}
\newcommand{\NIclrActOneSevenLlEightbClassicalFpEightNone}{1.0060}
\newcommand{\NIclrActOneSevenLlEightbNAboveTwo}{5}
\newcommand{\NIclrActOneSevenOlSevenbMin}{1.559}
\newcommand{\NIclrActOneSevenOlSevenbMedian}{1.629}
\newcommand{\NIclrActOneSevenOlSevenbMax}{1.73}
\newcommand{\NIclrActOneSevenOlSevenbSpread}{1.11}
\newcommand{\NIclrActOneSevenOlSevenbClassicalFpEightNone}{1.0023}
\newcommand{\NIclrActOneSevenOlSevenbNAboveTwo}{0}
\newcommand{\NIclrActOneSevenYiSixbMin}{1.120}
\newcommand{\NIclrActOneSevenYiSixbMedian}{1.164}
\newcommand{\NIclrActOneSevenYiSixbMax}{1.25}
\newcommand{\NIclrActOneSevenYiSixbSpread}{1.12}
\newcommand{\NIclrActOneSevenYiSixbClassicalFpEightNone}{1.0009}
\newcommand{\NIclrActOneSevenYiSixbNAboveTwo}{0}
\newcommand{\NIclrActTwoQOneFourObqaPlausibleDnatOptTwo}{10.00}
\newcommand{\NIclrActTwoQOneFourObqaPlausibleDnatOptTwoLo}{6.25}
\newcommand{\NIclrActTwoQOneFourObqaPlausibleDnatOptTwoHi}{13.75}
\newcommand{\NIclrActTwoQOneFourObqaPlausibleWThreeRate}{7.2}
\newcommand{\NIclrActTwoQOneFourObqaPlausibleWThreeNEqual}{69}
\newcommand{\NIclrActTwoQOneFourObqaPlausibleWThreeNScored}{960}
\newcommand{\NIclrActTwoQOneFourObqaItems}{240}
\newcommand{\NIclrActTwoQOneFourObqaPlausibleAccPctClean}{34.58}
\newcommand{\NIclrActTwoQOneFourObqaPlausibleAccPctSxc}{35.42}
\newcommand{\NIclrActTwoQOneFourObqaPlausibleAccPctOptTwo}{32.92}
\newcommand{\NIclrActFourStrassenTwoHeadroomLB}{4}
\newcommand{\NIclrActFourStrassenTwoHeadroomBB}{31}
\newcommand{\NIclrActFourStrassenTwoHeadroomLW}{16}
\newcommand{\NIclrActFourStrassenTwoExactTrials}{200}
\newcommand{\NIclrActFourStrassenTwoExactMaxLeafKIntThreeTwo}{8729}
\newcommand{\NIclrActFourStrassenTwoHeadroomLA}{4}
\newcommand{\NIclrActFourStrassenTwoHeadroomBA}{31}
\newcommand{\NIclrActFourStrassenTwoGaugesN}{512}
\newcommand{\NIclrActFourStrassenTwoGaugesIntDistinct}{1}
\newcommand{\NIclrActFourStrassenTwoGaugesEFourmThreeExact}{0}
\newcommand{\NIclrActFourStrassenTwoGaugesEFourmThreeLotteryRatio}{3.17}
\newcommand{\NIclrActFourStrassenTwoScheduleGauges}{64}
\newcommand{\NIclrActFourStrassenTwoScheduleIntDistinct}{1}
\newcommand{\NIclrActFourStrassenTwoPrefixPairs}{496}
\newcommand{\NIclrActFourStrassenTwoPrefixMovedInt}{0}
\newcommand{\NIclrActFourStrassenTwoPrefixMovedFpEight}{48}
\newcommand{\NIclrActFourStrassenTwoPrefixMovedFinegroupControl}{496}
\newcommand{\NIclrActFourStrassenTwoMultsRatio}{0.7656}
\newcommand{\NIclrActFourStrassenTwoMultsFast}{49}
\newcommand{\NIclrActFourStrassenTwoMultsClassical}{64}
\newcommand{\NIclrCostStrassenOneQGOneTwoEight}{2.04}
\newcommand{\NIclrCostStrassenOneQGTwoFiveSix}{3.17}
\newcommand{\NIclrCostStrassenOneQGFiveOneTwo}{2.09}
\newcommand{\NIclrCostTwoLevelVsClassical}{27.9}
\newcommand{\NIclrCostExecutorFloor}{2.25}
\newcommand{\NIclrCostOrderBandCxSPooled}{1.97}
\newcommand{\NIclrCostOrderBandSxSPooled}{1.46}
\newcommand{\NIclrCostRegretCxS}{1.000}
\newcommand{\NIclrCostRegretSxS}{1.000}
\newcommand{\NIclrCostRegretNFresh}{200}
\newcommand{\NIclrActOneCFourEightEightRatio}{0.985}
\newcommand{\NIclrActOneCFourEightEightCiLo}{0.863}
\newcommand{\NIclrActOneCFourEightEightCiHi}{1.141}
\newcommand{\NIclrActOneCFourSevenSevenSpan}{14.90}
\newcommand{\NIclrActOneCFourSevenSevenMin}{4.6080}
\newcommand{\NIclrActOneCFourSevenSevenMax}{68.6687}
\newcommand{\NIclrActOneCThreeNineTwoSpread}{1.00459}
\newcommand{\NIclrActTwoExemplarsNMoved}{14}
\newcommand{\NIclrActTwoExemplarsControlsQuiet}{14}
\newcommand{\NIclrActTwoExemplarsNaturalIsGold}{3}
\newcommand{\NIclrActTwoExemplarsNeitherIsGold}{9}
\newcommand{\NIclrActTwoExemplarsBlindIsGold}{2}
\newcommand{\NIclrCostAnchorVendorUs}{815}
\newcommand{\NIclrCostAnchorClassicalUsMin}{563}
\newcommand{\NIclrCostAnchorClassicalUsMax}{774}
\newcommand{\NIclrCostCorrOneTwoSevenSparseMin}{0.79}
\newcommand{\NIclrCostCorrOneTwoSevenSparseMax}{0.82}
\newcommand{\NIclrCostCorrOneTwoSevenGatheredMin}{1.43}
\newcommand{\NIclrCostCorrOneTwoSevenGatheredMax}{1.71}
\newcommand{\NIclrActOneOrbitFourFourFourSplitHalfPmOne}{0.121}
\newcommand{\NIclrActOneOrbitFourFourFourSplitHalfZeroOne}{0.9997}
\newcommand{\NIclrActOneRotationUniformDcPlain}{2.137}
\newcommand{\NIclrActOneRotationUniformDcRotated}{1.300}
\newcommand{\NIclrActOneRotationOutlierPlain}{1.240}
\newcommand{\NIclrActOneRotationOutlierRotated}{1.317}
\newcommand{\NIclrActFourStepRatioStrassenSqPredicted}{4.0968}
\newcommand{\NIclrActTwoQOneFourObqaPlausibleAccPctCubic}{33.75}
\newcommand{\NIclrActTwoQOneFourObqaFillerAccPctLibrary}{18.75}
\newcommand{\NIclrActTwoQOneFourArcFillerAccPctClean}{84.17}
\newcommand{\NIclrActTwoQOneFourArcFillerAccPctCubic}{83.33}
\newcommand{\NIclrActTwoQOneFourArcFillerAccPctLibrary}{39.58}
\newcommand{\NIclrActTwoQOneFourArcFillerAccPctOptTwo}{80.83}
\newcommand{\NIclrActTwoQOneFourArcFillerAccPctSxc}{82.50}
\newcommand{\NIclrActTwoCFourSixThreeUnrepairedFlipPctMin}{18.2}
\newcommand{\NIclrActTwoCFourSixThreeUnrepairedFlipPctMax}{40.9}
\newcommand{\NIclrActTwoCFourSixThreeRepairedFlipPctMin}{7.8}
\newcommand{\NIclrActTwoCFourSixThreeRepairedFlipPctMax}{15.2}
\newcommand{\NIclrActTwoCFourSixThreeControlsFlips}{0}
\newcommand{\NIclrActTwoCFourSixThreeControlsPositions}{1,305,600}
\newcommand{\NIclrActTwoCFourSixFourPairs}{768}
\newcommand{\NIclrActTwoCFourSixFourFastTvChangedMin}{768}
\newcommand{\NIclrActTwoCFourSixFourRepairedGreedyPctMin}{11.9}
\newcommand{\NIclrActTwoCFourSixFourRepairedGreedyPctMax}{13.0}
\newcommand{\NIclrActFourCatalogueN}{14,236}
\newcommand{\NIclrActFourCatalogueLAMin}{7}
\newcommand{\NIclrActFourCatalogueBABest}{18}
\newcommand{\NIclrWThreeEightNineQFourteenNllCleanTOne}{1.311}
\newcommand{\NIclrWThreeEightNineLlEightNllCleanTOne}{1.871}
\newcommand{\NIclrWThreeEightNineOlSevenNllCleanTOne}{2.078}
\newcommand{\NIclrQualityQFourteenCubicRaw}{13.7}
\newcommand{\NIclrQualityQFourteenCubicAdj}{13.6}
\newcommand{\NIclrQualityQFourteenCintThirtyOneRaw}{60.7}
\newcommand{\NIclrQualityQFourteenCintThirtyOneAdj}{59.8}
\newcommand{\NIclrQualityQFourteenCintOneTwentySevenRaw}{3.2}
\newcommand{\NIclrQualityQFourteenCintOneTwentySevenAdj}{3.3}
\newcommand{\NIclrQualityLlEightCubicRaw}{5.9}
\newcommand{\NIclrQualityLlEightCubicAdj}{5.9}
\newcommand{\NIclrQualityLlEightCintThirtyOneRaw}{14.5}
\newcommand{\NIclrQualityLlEightCintThirtyOneAdj}{14.5}
\newcommand{\NIclrQualityLlEightCintOneTwentySevenRaw}{1.8}
\newcommand{\NIclrQualityLlEightCintOneTwentySevenAdj}{1.8}
\newcommand{\NIclrQualityOlSevenCubicRaw}{2.3}
\newcommand{\NIclrQualityOlSevenCubicAdj}{1.9}
\newcommand{\NIclrQualityOlSevenCintThirtyOneRaw}{11.0}
\newcommand{\NIclrQualityOlSevenCintThirtyOneAdj}{13.3}
\newcommand{\NIclrQualityOlSevenCintOneTwentySevenRaw}{-0.4}
\newcommand{\NIclrQualityOlSevenCintOneTwentySevenAdj}{-0.1}
\newcommand{\NIclrQualityQThreeEightCubicRaw}{1.4}
\newcommand{\NIclrQualityQThreeEightCubicAdj}{2.5}
\newcommand{\NIclrQualityQThreeEightCintThirtyOneRaw}{-5.5}
\newcommand{\NIclrQualityQThreeEightCintThirtyOneAdj}{18.9}
\newcommand{\NIclrQualityQThreeEightCintOneTwentySevenRaw}{0.0}
\newcommand{\NIclrQualityQThreeEightCintOneTwentySevenAdj}{-0.3}
\newcommand{\NIclrActOneOrbitFourFourFourPhi}{1,024}
\newcommand{\NIclrActOneOrbitFourFourFourNnz}{432}
\newcommand{\NIclrActOneOrbitFourFourFourMaxAbs}{1}
\newcommand{\NIclrActOneSevenQThreebN}{512}
\newcommand{\NIclrActOneSevenQSevenbN}{508}
\newcommand{\NIclrActOneSevenQOneFourbN}{512}
\newcommand{\NIclrActOneSevenQThreexEightbN}{509}
\newcommand{\NIclrActOneSevenLlEightbN}{512}
\newcommand{\NIclrActOneSevenOlSevenbN}{512}
\newcommand{\NIclrActOneSevenYiSixbN}{512}
\newcommand{\NIclrActOneCFourEightZeroRelSpread}{0.0}
\newcommand{\NIclrActTwoCFourSixThreeQOneFourUnrepairedFlipPct}{22.1}
\newcommand{\NIclrActTwoCFourSixThreeQOneFourRepairedOneFlipPct}{8.2}
\newcommand{\NIclrActTwoCFourSixThreeQOneFourRepairedTwoFlipPct}{12.7}
\newcommand{\NIclrActTwoCFourSixThreeQThreeUnrepairedFlipPct}{21.3}
\newcommand{\NIclrActTwoCFourSixThreeQThreeRepairedOneFlipPct}{9.9}
\newcommand{\NIclrActTwoCFourSixThreeQThreeRepairedTwoFlipPct}{14.6}
\newcommand{\NIclrActTwoCFourSixThreeLlEightUnrepairedFlipPct}{24.8}
\newcommand{\NIclrActTwoCFourSixThreeLlEightRepairedOneFlipPct}{9.4}
\newcommand{\NIclrActTwoCFourSixThreeLlEightRepairedTwoFlipPct}{15.2}
\newcommand{\NIclrActTwoCFourSixThreeOlSevenUnrepairedFlipPct}{40.9}
\newcommand{\NIclrActTwoCFourSixThreeOlSevenRepairedOneFlipPct}{10.0}
\newcommand{\NIclrActTwoCFourSixThreeOlSevenRepairedTwoFlipPct}{15.0}
\newcommand{\NIclrActTwoCFourSixThreeQThreeEightUnrepairedFlipPct}{18.2}
\newcommand{\NIclrActTwoCFourSixThreeQThreeEightRepairedOneFlipPct}{7.8}
\newcommand{\NIclrActTwoCFourSixThreeQThreeEightRepairedTwoFlipPct}{12.7}
\newcommand{\NIclrActTwoCFourSixFourQOneFourPairs}{208}
\newcommand{\NIclrActTwoCFourSixFourQOneFourUnrepairedGreedyPct}{18.3}
\newcommand{\NIclrActTwoCFourSixFourQOneFourRepairedTwoGreedyPct}{13.0}
\newcommand{\NIclrActTwoCFourSixFourOlSevenPairs}{240}
\newcommand{\NIclrActTwoCFourSixFourOlSevenUnrepairedGreedyPct}{29.6}
\newcommand{\NIclrActTwoCFourSixFourOlSevenRepairedTwoGreedyPct}{12.9}
\newcommand{\NIclrActTwoCFourSixFourQThreePairs}{320}
\newcommand{\NIclrActTwoCFourSixFourQThreeUnrepairedGreedyPct}{20.9}
\newcommand{\NIclrActTwoCFourSixFourQThreeRepairedTwoGreedyPct}{11.9}
\newcommand{\NIclrActOneOrbitFourFourFourSpreadPmOne}{1.005}
\newcommand{\NIclrActOneOrbitFourFourFourSpreadZeroOne}{2.59}
\newcommand{\NIclrTokObqaGreedyRepairOnePct}{80.14}
\newcommand{\NIclrTokObqaGreedyRepairOneMean}{0.186}
\newcommand{\NIclrTokObqaGreedyRepairOneMedian}{0.097}
\newcommand{\NIclrTokObqaGreedyRepairOnePninety}{0.491}
\newcommand{\NIclrTokObqaGreedyRepairTwoPct}{89.00}
\newcommand{\NIclrTokObqaGreedyRepairTwoMean}{0.297}
\newcommand{\NIclrTokObqaGreedyRepairTwoMedian}{0.179}
\newcommand{\NIclrTokObqaGreedyRepairTwoPninety}{0.727}
\newcommand{\NIclrTokObqaFillerRepairOnePct}{90.15}
\newcommand{\NIclrTokObqaFillerRepairOneMean}{0.190}
\newcommand{\NIclrTokObqaFillerRepairOneMedian}{0.101}
\newcommand{\NIclrTokObqaFillerRepairOnePninety}{0.491}
\newcommand{\NIclrTokObqaFillerRepairTwoPct}{99.88}
\newcommand{\NIclrTokObqaFillerRepairTwoMean}{0.310}
\newcommand{\NIclrTokObqaFillerRepairTwoMedian}{0.187}
\newcommand{\NIclrTokObqaFillerRepairTwoPninety}{0.762}
\newcommand{\NIclrTokArcFillerRepairOnePct}{99.72}
\newcommand{\NIclrTokArcFillerRepairOneMean}{0.459}
\newcommand{\NIclrTokArcFillerRepairOneMedian}{0.204}
\newcommand{\NIclrTokArcFillerRepairOnePninety}{1.245}
\newcommand{\NIclrTokArcFillerRepairTwoPct}{99.72}
\newcommand{\NIclrTokArcFillerRepairTwoMean}{0.695}
\newcommand{\NIclrTokArcFillerRepairTwoMedian}{0.320}
\newcommand{\NIclrTokArcFillerRepairTwoPninety}{1.884}
\newcommand{\NIclrTokObqaCount}{3308}
\newcommand{\NIclrTokArcCount}{4315}
\newcommand{\NIclrTokControlChanged}{0}
\newcommand{\NIclrAccQfourteenIntFull}{34.17}
\newcommand{\NIclrAccLlamaBf}{35.42}
\newcommand{\NIclrAccLlamaFp}{35.00}
\newcommand{\NIclrAccLlamaIntLow}{35.00}
\newcommand{\NIclrAccLlamaIntFull}{35.00}
\newcommand{\NIclrAccOlmoBf}{37.92}
\newcommand{\NIclrAccOlmoFp}{39.17}
\newcommand{\NIclrAccOlmoIntLow}{37.92}
\newcommand{\NIclrAccOlmoIntFull}{38.75}
\newcommand{\NIclrAccQthreeBf}{31.67}
\newcommand{\NIclrAccQthreeFp}{33.75}
\newcommand{\NIclrAccQthreeIntLow}{32.08}
\newcommand{\NIclrAccQthreeIntFull}{31.67}
\newcommand{\NIclrRebKOneKTwoFiveSixSustMin}{0.892}
\newcommand{\NIclrRebKOneKTwoFiveSixSustMax}{0.916}
\newcommand{\NIclrRebKOneKTwoFiveSixSustNshapes}{8}
\newcommand{\NIclrRebKOneKOneTwoEightCSevenZeroSixMedian}{0.925}
\newcommand{\NIclrRebKOneBatchEightBFastprojMin}{0.885}
\newcommand{\NIclrRebKOneBatchEightBFastprojMax}{0.892}
\newcommand{\NIclrRebKOneBatchEightBPrepExcessPctMin}{2.3}
\newcommand{\NIclrRebKOneBatchEightBPrepExcessPctMax}{2.5}
\newcommand{\NIclrRebKOneBatchEightBTotalVsFpEightMin}{0.945}
\newcommand{\NIclrRebKOneBatchEightBTotalVsFpEightMax}{0.959}
\newcommand{\NIclrRebKTwoLayerMin}{0.81}
\newcommand{\NIclrRebKTwoLayerMax}{0.86}
\newcommand{\NIclrRebKTwoNcells}{34}
\newcommand{\NIclrRebHOneZeroZeroKOneMedian}{1.32}
\newcommand{\NIclrRebHOneZeroZeroKOneLSevenZero}{1.38}
\newcommand{\NIclrRebHOneZeroZeroKOneNcells}{13}
\newcommand{\NIclrRebSglBfOneSixAttnQThreePrefillGainPct}{5.0}
\newcommand{\NIclrRebSglBfOneSixAttnQThreeLongGainPct}{3.7}
\newcommand{\NIclrRebSglBfOneSixAttnQThreeChatGainPct}{2.5}
\newcommand{\NIclrRebSglBfOneSixAttnQThreeAgentGainPct}{1.5}
\newcommand{\NIclrRebSglBfOneSixAttnLlEightPrefillGainPct}{5.5}
\newcommand{\NIclrRebSglBfOneSixAttnLlEightLongGainPct}{3.8}
\newcommand{\NIclrRebSglBfOneSixAttnLlEightChatGainPct}{4.7}
\newcommand{\NIclrRebSglBfOneSixAttnLlEightAgentGainPct}{1.4}
\newcommand{\NIclrRebSglBfOneSixAttnQTwoSevenPrefillGainPct}{2.0}
\newcommand{\NIclrRebSglBfOneSixAttnQTwoSevenLongGainPct}{2.1}
\newcommand{\NIclrRebSglBfOneSixAttnQTwoSevenChatGainPct}{3.5}
\newcommand{\NIclrRebSglBfOneSixAttnQTwoSevenAgentGainPct}{1.2}
\newcommand{\NIclrRebSglBfOneSixAttnLlSevenZeroPrefillGainPct}{5.0}
\newcommand{\NIclrRebSglBfOneSixAttnLlSevenZeroLongGainPct}{4.3}
\newcommand{\NIclrRebSglBfOneSixAttnLlSevenZeroChatGainPct}{5.6}
\newcommand{\NIclrRebSglBfOneSixAttnLlSevenZeroAgentGainPct}{1.7}
\newcommand{\NIclrRebSglBfOneSixAttnGainPctMin}{1.2}
\newcommand{\NIclrRebSglBfOneSixAttnGainPctMax}{5.6}
\newcommand{\NIclrRebSglFpEightAttnQThreePrefillGainPct}{4.5}
\newcommand{\NIclrRebSglFpEightAttnQThreeLongGainPct}{3.7}
\newcommand{\NIclrRebSglFpEightAttnQThreeChatGainPct}{2.3}
\newcommand{\NIclrRebSglFpEightAttnQThreeAgentGainPct}{1.9}
\newcommand{\NIclrRebSglFpEightAttnLlEightPrefillGainPct}{6.1}
\newcommand{\NIclrRebSglFpEightAttnLlEightLongGainPct}{5.3}
\newcommand{\NIclrRebSglFpEightAttnLlEightChatGainPct}{5.5}
\newcommand{\NIclrRebSglFpEightAttnLlEightAgentGainPct}{3.2}
\newcommand{\NIclrRebSglFpEightAttnLlSevenZeroPrefillGainPct}{4.4}
\newcommand{\NIclrRebSglFpEightAttnLlSevenZeroLongGainPct}{4.0}
\newcommand{\NIclrRebSglFpEightAttnLlSevenZeroChatGainPct}{5.5}
\newcommand{\NIclrRebSglFpEightAttnLlSevenZeroAgentGainPct}{2.6}
\newcommand{\NIclrRebSglFpEightAttnGainPctMin}{1.9}
\newcommand{\NIclrRebSglFpEightAttnGainPctMax}{6.1}
\newcommand{\NIclrRebNllBfOneSixAttnQThreeIntEight}{0.0020}
\newcommand{\NIclrRebNllBfOneSixAttnQThreeFpEight}{0.0058}
\newcommand{\NIclrRebNllBfOneSixAttnLlEightIntEight}{0.0025}
\newcommand{\NIclrRebNllBfOneSixAttnLlEightFpEight}{0.0139}
\newcommand{\NIclrRebNllBfOneSixAttnQTwoSevenIntEight}{0.0014}
\newcommand{\NIclrRebNllBfOneSixAttnQTwoSevenFpEight}{0.0044}
\newcommand{\NIclrRebNllBfOneSixAttnLlSevenZeroIntEight}{0.0119}
\newcommand{\NIclrRebNllBfOneSixAttnLlSevenZeroFpEight}{0.0168}
\newcommand{\NIclrRebNllBfOneSixAttnLlSevenZeroIntEightOverFpEight}{0.71}
\newcommand{\NIclrRebNllBfOneSixAttnIntEightOverFpEightMinEightBTwoSevenB}{0.18}
\newcommand{\NIclrRebNllBfOneSixAttnIntEightOverFpEightMaxEightBTwoSevenB}{0.34}
\newcommand{\NIclrRebNllFpEightAttnQThreeIntEight}{0.0018}
\newcommand{\NIclrRebNllFpEightAttnQThreeFpEight}{0.0037}
\newcommand{\NIclrRebNllFpEightAttnLlEightIntEight}{0.0116}
\newcommand{\NIclrRebNllFpEightAttnLlEightFpEight}{0.0224}
\newcommand{\NIclrRebNllFpEightAttnLlSevenZeroIntEight}{0.0182}
\newcommand{\NIclrRebNllFpEightAttnLlSevenZeroFpEight}{0.0233}
\newcommand{\NIclrRebETwoQOneFourFpEightRaw}{13.73}
\newcommand{\NIclrRebETwoQOneFourFpEightAdj}{13.59}
\newcommand{\NIclrRebETwoQOneFourRSixThreeMainRaw}{11.85}
\newcommand{\NIclrRebETwoQOneFourRSixThreeMainAdj}{11.78}
\newcommand{\NIclrRebETwoLlEightFpEightRaw}{5.94}
\newcommand{\NIclrRebETwoLlEightFpEightAdj}{5.94}
\newcommand{\NIclrRebETwoLlEightRSixThreeMainRaw}{4.12}
\newcommand{\NIclrRebETwoLlEightRSixThreeMainAdj}{4.12}
\newcommand{\NIclrRebETwoOlSevenFpEightRaw}{2.32}
\newcommand{\NIclrRebETwoOlSevenFpEightAdj}{1.95}
\newcommand{\NIclrRebETwoOlSevenRSixThreeMainRaw}{0.30}
\newcommand{\NIclrRebETwoOlSevenRSixThreeMainAdj}{0.90}
\newcommand{\NIclrRebETwoQThreeEightFpEightRaw}{1.38}
\newcommand{\NIclrRebETwoQThreeEightFpEightAdj}{2.49}
\newcommand{\NIclrRebETwoQThreeEightRSixThreeMainRaw}{1.26}
\newcommand{\NIclrRebETwoQThreeEightRSixThreeMainAdj}{1.14}
\newcommand{\NIclrRebETwoPTwoRSixThreeMainCalls}{4,736}
\newcommand{\NIclrRebEEightComparisons}{112}
\newcommand{\NIclrRebEEightLlEightBfOneSixCIn}{0}
\newcommand{\NIclrRebEEightLlEightBfOneSixCOut}{0}
\newcommand{\NIclrRebEEightLlEightFpEightCIn}{0}
\newcommand{\NIclrRebEEightLlEightFpEightCOut}{0}
\newcommand{\NIclrRebEEightLlEightQspkCIn}{0}
\newcommand{\NIclrRebEEightLlEightQspkCOut}{0}
\newcommand{\NIclrRebEEightLlEightScCIn}{112}
\newcommand{\NIclrRebEEightLlEightScCOut}{79}
\newcommand{\NIclrRebEEightQThreeBfOneSixCIn}{0}
\newcommand{\NIclrRebEEightQThreeBfOneSixCOut}{0}
\newcommand{\NIclrRebEEightQThreeFpEightCIn}{0}
\newcommand{\NIclrRebEEightQThreeFpEightCOut}{0}
\newcommand{\NIclrRebEEightQThreeQspkCIn}{0}
\newcommand{\NIclrRebEEightQThreeQspkCOut}{0}
\newcommand{\NIclrRebEEightQThreeScCIn}{112}
\newcommand{\NIclrRebEEightQThreeScCOut}{86}
\newcommand{\NIclrRebEFourQOneFourEntry}{0.866}
\newcommand{\NIclrRebEFourQOneFourRows}{2.04}
\newcommand{\NIclrRebEFourQOneFourCols}{2.54}
\newcommand{\NIclrRebEFourQOneFourOfA}{2.6}
\newcommand{\NIclrRebEFourQOneFourOfB}{9.5}
\newcommand{\NIclrRebEFourLlEightEntry}{0.877}
\newcommand{\NIclrRebEFourLlEightRows}{2.14}
\newcommand{\NIclrRebEFourLlEightCols}{2.68}
\newcommand{\NIclrRebEFourLlEightOfA}{2.7}
\newcommand{\NIclrRebEFourLlEightOfB}{9.7}
\newcommand{\NIclrRebEFourOlSevenEntry}{0.885}
\newcommand{\NIclrRebEFourOlSevenRows}{2.19}
\newcommand{\NIclrRebEFourOlSevenCols}{2.72}
\newcommand{\NIclrRebEFourOlSevenOfA}{6.0}
\newcommand{\NIclrRebEFourOlSevenOfB}{10.1}
\newcommand{\NIclrRebEFourQThreeEightEntry}{0.863}
\newcommand{\NIclrRebEFourQThreeEightRows}{2.01}
\newcommand{\NIclrRebEFourQThreeEightCols}{2.51}
\newcommand{\NIclrRebEFourQThreeEightOfA}{2.4}
\newcommand{\NIclrRebEFourQThreeEightOfB}{9.3}
\newcommand{\NIclrRebEFourModelEntryMin}{0.863}
\newcommand{\NIclrRebEFourModelEntryMax}{0.885}
\newcommand{\NIclrRebEFourModelRowsMin}{2.01}
\newcommand{\NIclrRebEFourModelRowsMax}{2.19}
\newcommand{\NIclrRebEFourLayersN}{1,036}
\newcommand{\NIclrRebEFourLayerEntryMin}{0.786}
\newcommand{\NIclrRebEFourLayerEntryMedian}{0.863}
\newcommand{\NIclrRebEFourLayerEntryMax}{0.905}
\newcommand{\NIclrRebEFourLayerRowsMedian}{2.01}
\newcommand{\NIclrRebEFourConcTopOneColsA}{12.2}
\newcommand{\NIclrRebEFourConcTopOneRowsA}{7.5}
\newcommand{\NIclrRebEFourGaugeGainMedian}{0.0002}
\newcommand{\NIclrRebEFourGaugeGainMax}{0.011}
\newcommand{\NIclrRebEFourSlicesPTenMin}{0.843}
\newcommand{\NIclrRebEFourSlicesPTenMax}{0.848}
\newcommand{\NIclrRebEThreeSTwoFull}{27.60}
\newcommand{\NIclrRebEThreeSTwoCxc}{2.22}
\newcommand{\NIclrRebEThreeSTwoNoenc}{2.35}
\newcommand{\NIclrRebEThreeSTwoNodec}{20.65}
\newcommand{\NIclrRebEThreeSTwoNoencNodec}{1.91}
\newcommand{\NIclrRebEThreeSTwoFullOverNoenc}{11.8}
\newcommand{\NIclrRebEThreeSTwoFullOverNodec}{1.3}
\newcommand{\NIclrRebEThreeSOneFull}{3.11}
\newcommand{\NIclrRebEThreeSOneCxc}{1.15}
\newcommand{\NIclrRebEThreeSOneNoenc}{0.98}
\newcommand{\NIclrRebEThreeSOneNodec}{2.48}
\newcommand{\NIclrRebEThreeSOneNoencNodec}{1.03}
\newcommand{\NIclrRebEThreeSTwoFullSpills}{2,272}
\newcommand{\NIclrRebEThreeOctTwoOverCl}{0.894}
\newcommand{\NIclrRebEThreeOctOneOverCl}{0.925}
\newcommand{\NIclrRebESevenK}{32}
\newcommand{\NIclrRebESevenRepOneTopFiveBits}{5.00}
\newcommand{\NIclrRebESevenRepOneTopFiveAcc}{100.0}
\newcommand{\NIclrRebESevenRepOneGreedyBits}{0.14}
\newcommand{\NIclrRebESevenRepOneGreedyAcc}{4.5}
\newcommand{\NIclrRebESevenRepTwoTopFiveAcc}{100.0}
\newcommand{\NIclrRebESevenRepTwoGreedyBits}{0.28}
\newcommand{\NIclrRebESevenRepTwoGreedyAcc}{5.3}
\newcommand{\NIclrRebESevenBfRowAcc}{3.1}
\newcommand{\NIclrRebEOneContFgThreeChanged}{7.50}
\newcommand{\NIclrRebEOneContFgThreeChangedLo}{4.58}
\newcommand{\NIclrRebEOneContFgThreeChangedHi}{10.83}
\newcommand{\NIclrRebEOneContFgThreeAcc}{31.67}
\newcommand{\NIclrRebEOneContFgThreeTokpos}{89.1}
\newcommand{\NIclrRebEOneContFgFChanged}{4.17}
\newcommand{\NIclrRebEOneContFgFChangedLo}{1.67}
\newcommand{\NIclrRebEOneContFgFChangedHi}{7.08}
\newcommand{\NIclrRebEOneContFgFAcc}{32.92}
\newcommand{\NIclrRebEOneContFgFTokpos}{80.2}
\newcommand{\NIclrRebEOneFillFgThreeChanged}{7.50}
\newcommand{\NIclrRebEOneFillFgThreeChangedLo}{4.58}
\newcommand{\NIclrRebEOneFillFgThreeChangedHi}{10.83}
\newcommand{\NIclrRebEOneFillFgFChanged}{4.58}
\newcommand{\NIclrRebEOneFillFgFChangedLo}{2.08}
\newcommand{\NIclrRebEOneFillFgFChangedHi}{7.50}
\newcommand{\NIclrRebEOneArcFgThreeChanged}{10.83}
\newcommand{\NIclrRebEOneArcFgThreeChangedLo}{7.08}
\newcommand{\NIclrRebEOneArcFgThreeChangedHi}{15.00}
\newcommand{\NIclrRebEOneArcFgThreeAcc}{80.00}
\newcommand{\NIclrRebEOneArcFgFChanged}{14.17}
\newcommand{\NIclrRebEOneArcFgFChangedLo}{10.00}
\newcommand{\NIclrRebEOneArcFgFChangedHi}{18.75}
\newcommand{\NIclrRebEOneArcFgFAcc}{78.33}
\newcommand{\NIclrRebEOnePackedQOneFourFgThreeTvPairs}{208}
\newcommand{\NIclrRebEOnePackedQOneFourFgThreeFlip}{12.0}
\newcommand{\NIclrRebEOnePackedQOneFourFgFTvPairs}{208}
\newcommand{\NIclrRebEOnePackedQOneFourFgFFlip}{8.2}
\newcommand{\NIclrRebEOnePackedQOneFourPairs}{208}
\newcommand{\NIclrRebEOnePackedOlSevenFgThreeTvPairs}{240}
\newcommand{\NIclrRebEOnePackedOlSevenFgThreeFlip}{8.8}
\newcommand{\NIclrRebEOnePackedOlSevenFgFTvPairs}{240}
\newcommand{\NIclrRebEOnePackedOlSevenFgFFlip}{7.1}
\newcommand{\NIclrRebEOnePackedOlSevenPairs}{240}
\newcommand{\NIclrRebEOnePackedQThreeFgThreeTvPairs}{320}
\newcommand{\NIclrRebEOnePackedQThreeFgThreeFlip}{8.4}
\newcommand{\NIclrRebEOnePackedQThreeFgFTvPairs}{320}
\newcommand{\NIclrRebEOnePackedQThreeFgFFlip}{6.6}
\newcommand{\NIclrRebEOnePackedQThreePairs}{320}
\newcommand{\NIclrRebEOnePackedFlipMin}{6.6}
\newcommand{\NIclrRebEOnePackedFlipMax}{12.0}
\newcommand{\NIclrRebEFiveADmeanMin}{-0.023}
\newcommand{\NIclrRebEFiveADmeanMax}{0.005}
\newcommand{\NIclrRebEFiveAPosShareMin}{49}
\newcommand{\NIclrRebEFiveAPosShareMax}{52}
\newcommand{\NIclrRebESixAShareMin}{97.3}
\newcommand{\NIclrRebESixAJsMeanMin}{0.007}
\newcommand{\NIclrRebESixAJsMeanMax}{0.024}
\newcommand{\NIclrRebESixBRepOneTextChanged}{40.2}
\newcommand{\NIclrRebESixBRepOneTextChangedLo}{32.9}
\newcommand{\NIclrRebESixBRepOneTextChangedHi}{47.6}
\newcommand{\NIclrRebESixBRepOneFirstStep}{100}
\newcommand{\NIclrRebESixBRepOneCorrChanged}{4.3}
\newcommand{\NIclrRebESixBRepOneCorrChangedLo}{1.2}
\newcommand{\NIclrRebESixBRepOneCorrChangedHi}{7.9}
\newcommand{\NIclrRebESixBRepOneEmBOne}{31.7}
\newcommand{\NIclrRebESixBRepOneEmBTwo}{29.9}
\newcommand{\NIclrRebESixBRepTwoTextChanged}{51.8}
\newcommand{\NIclrRebESixBRepTwoTextChangedLo}{43.9}
\newcommand{\NIclrRebESixBRepTwoTextChangedHi}{59.8}
\newcommand{\NIclrRebESixBRepTwoCorrChanged}{6.7}
\newcommand{\NIclrRebESixBRepTwoCorrChangedLo}{3.0}
\newcommand{\NIclrRebESixBRepTwoCorrChangedHi}{11.0}
\newcommand{\NIclrRebESixBRepTwoEmBOne}{27.4}
\newcommand{\NIclrRebESixBRepTwoEmBTwo}{30.5}
\newcommand{\NIclrRebESixBItems}{164}
\newcommand{\NIclrRebESixBCleanEm}{30.5}
\newcommand{\NIclrRebESixBRepOneFirstTokenN}{25}
\newcommand{\NIclrRebESixBRepOneTextChangedN}{66}
\newcommand{\NIclrRebESixBRepTwoFirstTokenN}{45}
\newcommand{\NIclrRebESixBRepTwoTextChangedN}{85}
\newcommand{\NIclrRebENineGenFgThreeTextChanged}{43.9}
\newcommand{\NIclrRebENineGenFgThreeTextChangedLo}{36.6}
\newcommand{\NIclrRebENineGenFgThreeTextChangedHi}{51.2}
\newcommand{\NIclrRebENineGenFgThreeCorrChanged}{5.5}
\newcommand{\NIclrRebENineGenFgThreeCorrChangedLo}{2.4}
\newcommand{\NIclrRebENineGenFgThreeCorrChangedHi}{9.1}
\newcommand{\NIclrRebENineGenFgThreeEmBOne}{29.3}
\newcommand{\NIclrRebENineGenFgThreeEmBTwo}{29.9}
\newcommand{\NIclrRebENineGenFgFTextChanged}{39.6}
\newcommand{\NIclrRebENineGenFgFTextChangedLo}{32.3}
\newcommand{\NIclrRebENineGenFgFTextChangedHi}{47.0}
\newcommand{\NIclrRebENineGenFgFCorrChanged}{7.9}
\newcommand{\NIclrRebENineGenFgFCorrChangedLo}{4.3}
\newcommand{\NIclrRebENineGenFgFCorrChangedHi}{12.2}
\newcommand{\NIclrRebENineGenFgFEmBOne}{29.3}
\newcommand{\NIclrRebENineGenFgFEmBTwo}{29.9}
\newcommand{\NIclrRebENineGenTextChangedMin}{39.6}
\newcommand{\NIclrRebENineGenTextChangedMax}{43.9}
\newcommand{\NIclrRebGenCorrChangedMin}{4.3}
\newcommand{\NIclrRebGenCorrChangedMax}{9.8}
\newcommand{\NIclrRebENineIdentFgThreeORowHeldout}{100.0}
\newcommand{\NIclrRebENineIdentFgThreeTopFiveHeldout}{99.8}
\newcommand{\NIclrRebENineIdentFgThreeGreedyBits}{0.30}
\newcommand{\NIclrRebENineIdentFgFORowHeldout}{100.0}
\newcommand{\NIclrRebENineIdentFgFTopFiveHeldout}{99.4}
\newcommand{\NIclrRebENineIdentFgFGreedyBits}{0.24}
\newcommand{\NIclrRebENineIdentTopFiveHeldoutMin}{99.4}
\newcommand{\NIclrRebESixCRepOneTextChanged}{25.6}
\newcommand{\NIclrRebESixCRepOneTextChangedLo}{18.9}
\newcommand{\NIclrRebESixCRepOneTextChangedHi}{32.3}
\newcommand{\NIclrRebESixCRepOneCorrChanged}{4.3}
\newcommand{\NIclrRebESixCRepOneCorrChangedLo}{1.8}
\newcommand{\NIclrRebESixCRepOneCorrChangedHi}{7.3}
\newcommand{\NIclrRebESixCRepOneEmBOne}{32.9}
\newcommand{\NIclrRebESixCRepOneEmBTwo}{31.1}
\newcommand{\NIclrRebESixCRepTwoTextChanged}{42.1}
\newcommand{\NIclrRebESixCRepTwoTextChangedLo}{34.8}
\newcommand{\NIclrRebESixCRepTwoTextChangedHi}{50.0}
\newcommand{\NIclrRebESixCRepTwoCorrChanged}{9.8}
\newcommand{\NIclrRebESixCRepTwoCorrChangedLo}{5.5}
\newcommand{\NIclrRebESixCRepTwoCorrChangedHi}{14.6}
\newcommand{\NIclrRebESixCRepTwoEmBOne}{32.9}
\newcommand{\NIclrRebESixCRepTwoEmBTwo}{29.3}
\newcommand{\NIclrRebESixCCleanEm}{37.8}
\newcommand{\NIclrRebGenRepTextChangedMin}{25.6}
\newcommand{\NIclrRebGenRepTextChangedMax}{51.8}
\newcommand{\NIclrRebEFivePresTokMin}{-0.0108}
\newcommand{\NIclrRebEFivePresTokMax}{-0.0002}
\newcommand{\NIclrRebETenSecondHalfChangedMin}{41.9}
\newcommand{\NIclrRebETenPartnerChangedMin}{41.9}
\newcommand{\NIclrRebETenMaxChangePct}{1.1}
\newcommand{\NIclrRebETenNVariants}{8}
\newcommand{\NIclrRebMechQFourteenIsolateRemoved}{86}
\newcommand{\NIclrRebMechQFourteenEquilSpread}{1.078}
\newcommand{\NIclrRebMechQFourteenEquilNcat}{0}
\newcommand{\NIclrRebMechQFourteenRowNat}{10.8}
\newcommand{\NIclrRebMechQFourteenRowAoff}{12.9}
\newcommand{\NIclrRebMechQFourteenRowWoff}{1.03}
\newcommand{\NIclrRebMechQThreeNativeNcat}{3}
\newcommand{\NIclrRebMechQThreeEquilNcat}{2}
\newcommand{\NIclrRebMechQThirtyEightSoloPpl}{1.10}
\newcommand{\NIclrRebMechQThirtyEightFullPpl}{6.15}
\newcommand{\NIclrRebHOneZeroZeroSplitEightTriton}{1.10}
\newcommand{\NIclrRebHOneZeroZeroSplitEightCublaslt}{1.13}
\newcommand{\NIclrRebEElevenQOneFourCfpEightRaw}{14.3}
\newcommand{\NIclrRebEElevenQOneFourCfpEightAdj}{14.1}
\newcommand{\NIclrRebEElevenQOneFourCintThirtyOneWbRaw}{106.3}
\newcommand{\NIclrRebEElevenQOneFourCintThirtyOneWbAdj}{104.9}
\newcommand{\NIclrRebEElevenQOneFourCintOneTwentySevenWbRaw}{5.8}
\newcommand{\NIclrRebEElevenQOneFourCintOneTwentySevenWbAdj}{5.7}
\newcommand{\NIclrRebEElevenLlEightCfpEightRaw}{6.5}
\newcommand{\NIclrRebEElevenLlEightCfpEightAdj}{6.5}
\newcommand{\NIclrRebEElevenLlEightCintThirtyOneWbRaw}{39.5}
\newcommand{\NIclrRebEElevenLlEightCintThirtyOneWbAdj}{39.5}
\newcommand{\NIclrRebEElevenLlEightCintOneTwentySevenWbRaw}{3.8}
\newcommand{\NIclrRebEElevenLlEightCintOneTwentySevenWbAdj}{3.8}
\newcommand{\NIclrRebEElevenOlSevenCfpEightRaw}{1.9}
\newcommand{\NIclrRebEElevenOlSevenCfpEightAdj}{2.3}
\newcommand{\NIclrRebEElevenOlSevenCintThirtyOneWbRaw}{24.0}
\newcommand{\NIclrRebEElevenOlSevenCintThirtyOneWbAdj}{29.5}
\newcommand{\NIclrRebEElevenOlSevenCintOneTwentySevenWbRaw}{2.5}
\newcommand{\NIclrRebEElevenOlSevenCintOneTwentySevenWbAdj}{3.5}
\newcommand{\NIclrRebEElevenQThreeEightCfpEightRaw}{-0.6}
\newcommand{\NIclrRebEElevenQThreeEightCfpEightAdj}{1.9}
\newcommand{\NIclrRebEElevenQThreeEightCintThirtyOneWbRaw}{15.4}
\newcommand{\NIclrRebEElevenQThreeEightCintThirtyOneWbAdj}{41.2}
\newcommand{\NIclrRebEElevenQThreeEightCintOneTwentySevenWbRaw}{-3.8}
\newcommand{\NIclrRebEElevenQThreeEightCintOneTwentySevenWbAdj}{1.3}
\newcommand{\NIclrRebEElevenWitnessSitesMin}{224}
\newcommand{\NIclrRebEElevenSchedDiffMin}{-0.3}
\newcommand{\NIclrRebEElevenWitnessSitesMax}{336}
\newcommand{\NIclrRebEElevenSchedDiffMax}{1.8}
\newcommand{\NIclrRebEOneTwoQOneFourFpEightptRotRaw}{25.7}
\newcommand{\NIclrRebEOneTwoQOneFourFpEightptRotAdj}{25.2}
\newcommand{\NIclrRebEOneTwoQOneFourFpEightptGptqRaw}{26.1}
\newcommand{\NIclrRebEOneTwoQOneFourFpEightptGptqAdj}{26.0}
\newcommand{\NIclrRebEOneTwoQOneFourFpEightptRotGptqRaw}{25.1}
\newcommand{\NIclrRebEOneTwoQOneFourFpEightptRotGptqAdj}{24.8}
\newcommand{\NIclrRebEOneTwoQOneFourFpEightblkRotRaw}{24.0}
\newcommand{\NIclrRebEOneTwoQOneFourFpEightblkRotAdj}{23.6}
\newcommand{\NIclrRebEOneTwoQOneFourFpEightblkGptqRaw}{22.1}
\newcommand{\NIclrRebEOneTwoQOneFourFpEightblkGptqAdj}{21.8}
\newcommand{\NIclrRebEOneTwoQOneFourFpEightblkRotGptqRaw}{21.6}
\newcommand{\NIclrRebEOneTwoQOneFourFpEightblkRotGptqAdj}{21.6}
\newcommand{\NIclrRebEOneTwoLlEightFpEightptRotRaw}{12.4}
\newcommand{\NIclrRebEOneTwoLlEightFpEightptRotAdj}{12.4}
\newcommand{\NIclrRebEOneTwoLlEightFpEightptGptqRaw}{10.7}
\newcommand{\NIclrRebEOneTwoLlEightFpEightptGptqAdj}{10.7}
\newcommand{\NIclrRebEOneTwoLlEightFpEightptRotGptqRaw}{5.9}
\newcommand{\NIclrRebEOneTwoLlEightFpEightptRotGptqAdj}{5.9}
\newcommand{\NIclrRebEOneTwoLlEightFpEightblkRotRaw}{11.3}
\newcommand{\NIclrRebEOneTwoLlEightFpEightblkRotAdj}{11.3}
\newcommand{\NIclrRebEOneTwoLlEightFpEightblkGptqRaw}{6.8}
\newcommand{\NIclrRebEOneTwoLlEightFpEightblkGptqAdj}{6.8}
\newcommand{\NIclrRebEOneTwoLlEightFpEightblkRotGptqRaw}{6.5}
\newcommand{\NIclrRebEOneTwoLlEightFpEightblkRotGptqAdj}{6.5}
\newcommand{\NIclrRebEOneTwoOlSevenFpEightptRotRaw}{2.0}
\newcommand{\NIclrRebEOneTwoOlSevenFpEightptRotAdj}{3.5}
\newcommand{\NIclrRebEOneTwoOlSevenFpEightptGptqRaw}{2.9}
\newcommand{\NIclrRebEOneTwoOlSevenFpEightptGptqAdj}{3.7}
\newcommand{\NIclrRebEOneTwoOlSevenFpEightptRotGptqRaw}{1.1}
\newcommand{\NIclrRebEOneTwoOlSevenFpEightptRotGptqAdj}{1.8}
\newcommand{\NIclrRebEOneTwoOlSevenFpEightblkRotRaw}{5.7}
\newcommand{\NIclrRebEOneTwoOlSevenFpEightblkRotAdj}{4.5}
\newcommand{\NIclrRebEOneTwoOlSevenFpEightblkGptqRaw}{6.8}
\newcommand{\NIclrRebEOneTwoOlSevenFpEightblkGptqAdj}{6.3}
\newcommand{\NIclrRebEOneTwoOlSevenFpEightblkRotGptqRaw}{1.1}
\newcommand{\NIclrRebEOneTwoOlSevenFpEightblkRotGptqAdj}{1.8}
\newcommand{\NIclrRebEOneTwoQThreeEightFpEightptRotRaw}{13.3}
\newcommand{\NIclrRebEOneTwoQThreeEightFpEightptRotAdj}{8.6}
\newcommand{\NIclrRebEOneTwoQThreeEightFpEightptGptqRaw}{6.2}
\newcommand{\NIclrRebEOneTwoQThreeEightFpEightptGptqAdj}{6.5}
\newcommand{\NIclrRebEOneTwoQThreeEightFpEightptRotGptqRaw}{4.2}
\newcommand{\NIclrRebEOneTwoQThreeEightFpEightptRotGptqAdj}{6.3}
\newcommand{\NIclrRebEOneTwoQThreeEightFpEightblkRotRaw}{4.0}
\newcommand{\NIclrRebEOneTwoQThreeEightFpEightblkRotAdj}{8.7}
\newcommand{\NIclrRebEOneTwoQThreeEightFpEightblkGptqRaw}{1.9}
\newcommand{\NIclrRebEOneTwoQThreeEightFpEightblkGptqAdj}{3.9}
\newcommand{\NIclrRebEOneTwoQThreeEightFpEightblkRotGptqRaw}{7.5}
\newcommand{\NIclrRebEOneTwoQThreeEightFpEightblkRotGptqAdj}{6.9}
\newcommand{\NIclrRebEOneTwoQOneFourCubicRaw}{13.7}
\newcommand{\NIclrRebEOneTwoQOneFourCubicAdj}{13.6}
\newcommand{\NIclrRebEOneTwoQOneFourFpEightptRaw}{24.7}
\newcommand{\NIclrRebEOneTwoQOneFourFpEightptAdj}{24.2}
\newcommand{\NIclrRebEOneTwoQOneFourRSixThreeMainRaw}{11.8}
\newcommand{\NIclrRebEOneTwoQOneFourRSixThreeMainAdj}{11.8}
\newcommand{\NIclrRebEOneTwoLlEightCubicRaw}{5.9}
\newcommand{\NIclrRebEOneTwoLlEightCubicAdj}{5.9}
\newcommand{\NIclrRebEOneTwoLlEightFpEightptRaw}{7.4}
\newcommand{\NIclrRebEOneTwoLlEightFpEightptAdj}{7.4}
\newcommand{\NIclrRebEOneTwoLlEightRSixThreeMainRaw}{4.1}
\newcommand{\NIclrRebEOneTwoLlEightRSixThreeMainAdj}{4.1}
\newcommand{\NIclrRebEOneTwoOlSevenCubicRaw}{2.3}
\newcommand{\NIclrRebEOneTwoOlSevenCubicAdj}{1.9}
\newcommand{\NIclrRebEOneTwoOlSevenFpEightptRaw}{5.3}
\newcommand{\NIclrRebEOneTwoOlSevenFpEightptAdj}{3.8}
\newcommand{\NIclrRebEOneTwoOlSevenRSixThreeMainRaw}{0.3}
\newcommand{\NIclrRebEOneTwoOlSevenRSixThreeMainAdj}{0.9}
\newcommand{\NIclrRebEOneTwoQThreeEightCubicRaw}{1.4}
\newcommand{\NIclrRebEOneTwoQThreeEightCubicAdj}{2.5}
\newcommand{\NIclrRebEOneTwoQThreeEightFpEightptRaw}{5.0}
\newcommand{\NIclrRebEOneTwoQThreeEightFpEightptAdj}{6.4}
\newcommand{\NIclrRebEOneTwoQThreeEightRSixThreeMainRaw}{1.3}
\newcommand{\NIclrRebEOneTwoQThreeEightRSixThreeMainAdj}{1.1}
\newcommand{\NIclrRebEOneThreeQThreeFpEightMean}{0.0058}
\newcommand{\NIclrRebEOneThreeQThreeFpEightgMean}{0.0054}
\newcommand{\NIclrRebEOneThreeQThreeFpEightrgMean}{0.0090}
\newcommand{\NIclrRebEOneThreeQThreeQspkMean}{0.0020}
\newcommand{\NIclrRebEOneThreeLlEightFpEightMean}{0.0139}
\newcommand{\NIclrRebEOneThreeLlEightFpEightgMean}{0.0112}
\newcommand{\NIclrRebEOneThreeLlEightFpEightrgMean}{0.0076}
\newcommand{\NIclrRebEOneThreeLlEightQspkMean}{0.0025}
\newcommand{\NIclrRebEOneFivePrefillLongGainPctMin}{4.3}
\newcommand{\NIclrRebEOneFivePrefillLongGainPctMax}{6.7}
\newcommand{\NIclrRebEOneFiveChatAgentGainPctMin}{1.4}
\newcommand{\NIclrRebEOneFiveChatAgentGainPctMax}{2.9}
\newcommand{\NIclrRebEOneFiveChatAgentIntEightSpreadPctMax}{3.5}
\newcommand{\NIclrRebEOneThreebQThreeFpEightgAMean}{0.0038}
\newcommand{\NIclrRebEOneThreebQThreeFpEightgHMean}{0.0050}
\newcommand{\NIclrRebEOneThreebQThreeFpEightrgAMean}{0.0093}
\newcommand{\NIclrRebEOneThreebQThreeFpEightrgHMean}{0.0090}
\newcommand{\NIclrRebEOneThreebQThreeFpEightbestOverQspk}{1.9}
\newcommand{\NIclrRebEOneThreebLlEightFpEightgAMean}{0.0103}
\newcommand{\NIclrRebEOneThreebLlEightFpEightgHMean}{0.0107}
\newcommand{\NIclrRebEOneThreebLlEightFpEightrgAMean}{0.0091}
\newcommand{\NIclrRebEOneThreebLlEightFpEightrgHMean}{0.0099}
\newcommand{\NIclrRebEOneThreebLlEightFpEightbestOverQspk}{3.0}
\newcommand{\NIclrRebEOneTwobQOneFourFpEightptGptqARaw}{27.6}
\newcommand{\NIclrRebEOneTwobQOneFourFpEightptGptqAAdj}{27.6}
\newcommand{\NIclrRebEOneTwobQOneFourFpEightptGptqHRaw}{26.2}
\newcommand{\NIclrRebEOneTwobQOneFourFpEightptGptqHAdj}{25.3}
\newcommand{\NIclrRebEOneTwobQOneFourFpEightptRotGptqARaw}{16.3}
\newcommand{\NIclrRebEOneTwobQOneFourFpEightptRotGptqAAdj}{16.3}
\newcommand{\NIclrRebEOneTwobQOneFourFpEightptRotGptqHRaw}{19.9}
\newcommand{\NIclrRebEOneTwobQOneFourFpEightptRotGptqHAdj}{19.7}
\newcommand{\NIclrRebEOneTwobQOneFourFpEightblkGptqARaw}{16.9}
\newcommand{\NIclrRebEOneTwobQOneFourFpEightblkGptqAAdj}{16.8}
\newcommand{\NIclrRebEOneTwobQOneFourFpEightblkGptqHRaw}{20.3}
\newcommand{\NIclrRebEOneTwobQOneFourFpEightblkGptqHAdj}{20.1}
\newcommand{\NIclrRebEOneTwobQOneFourFpEightblkRotGptqARaw}{19.7}
\newcommand{\NIclrRebEOneTwobQOneFourFpEightblkRotGptqAAdj}{19.4}
\newcommand{\NIclrRebEOneTwobQOneFourFpEightblkRotGptqHRaw}{21.8}
\newcommand{\NIclrRebEOneTwobQOneFourFpEightblkRotGptqHAdj}{21.6}
\newcommand{\NIclrRebEOneTwobLlEightFpEightptGptqARaw}{7.8}
\newcommand{\NIclrRebEOneTwobLlEightFpEightptGptqAAdj}{7.8}
\newcommand{\NIclrRebEOneTwobLlEightFpEightptGptqHRaw}{8.1}
\newcommand{\NIclrRebEOneTwobLlEightFpEightptGptqHAdj}{8.1}
\newcommand{\NIclrRebEOneTwobLlEightFpEightptRotGptqARaw}{7.0}
\newcommand{\NIclrRebEOneTwobLlEightFpEightptRotGptqAAdj}{7.0}
\newcommand{\NIclrRebEOneTwobLlEightFpEightptRotGptqHRaw}{8.0}
\newcommand{\NIclrRebEOneTwobLlEightFpEightptRotGptqHAdj}{8.0}
\newcommand{\NIclrRebEOneTwobLlEightFpEightblkGptqARaw}{6.9}
\newcommand{\NIclrRebEOneTwobLlEightFpEightblkGptqAAdj}{6.9}
\newcommand{\NIclrRebEOneTwobLlEightFpEightblkGptqHRaw}{5.4}
\newcommand{\NIclrRebEOneTwobLlEightFpEightblkGptqHAdj}{5.4}
\newcommand{\NIclrRebEOneTwobLlEightFpEightblkRotGptqARaw}{6.5}
\newcommand{\NIclrRebEOneTwobLlEightFpEightblkRotGptqAAdj}{6.5}
\newcommand{\NIclrRebEOneTwobLlEightFpEightblkRotGptqHRaw}{5.8}
\newcommand{\NIclrRebEOneTwobLlEightFpEightblkRotGptqHAdj}{5.8}
\newcommand{\NIclrRebEOneTwobOlSevenFpEightptGptqARaw}{6.8}
\newcommand{\NIclrRebEOneTwobOlSevenFpEightptGptqAAdj}{6.1}
\newcommand{\NIclrRebEOneTwobOlSevenFpEightptGptqHRaw}{2.6}
\newcommand{\NIclrRebEOneTwobOlSevenFpEightptGptqHAdj}{2.9}
\newcommand{\NIclrRebEOneTwobOlSevenFpEightptRotGptqARaw}{2.4}
\newcommand{\NIclrRebEOneTwobOlSevenFpEightptRotGptqAAdj}{2.7}
\newcommand{\NIclrRebEOneTwobOlSevenFpEightptRotGptqHRaw}{-0.1}
\newcommand{\NIclrRebEOneTwobOlSevenFpEightptRotGptqHAdj}{1.0}
\newcommand{\NIclrRebEOneTwobOlSevenFpEightblkGptqARaw}{5.0}
\newcommand{\NIclrRebEOneTwobOlSevenFpEightblkGptqAAdj}{4.1}
\newcommand{\NIclrRebEOneTwobOlSevenFpEightblkGptqHRaw}{-3.1}
\newcommand{\NIclrRebEOneTwobOlSevenFpEightblkGptqHAdj}{-0.7}
\newcommand{\NIclrRebEOneTwobOlSevenFpEightblkRotGptqARaw}{2.4}
\newcommand{\NIclrRebEOneTwobOlSevenFpEightblkRotGptqAAdj}{3.0}
\newcommand{\NIclrRebEOneTwobOlSevenFpEightblkRotGptqHRaw}{0.1}
\newcommand{\NIclrRebEOneTwobOlSevenFpEightblkRotGptqHAdj}{1.5}
\newcommand{\NIclrRebEOneTwobQThreeEightFpEightptGptqARaw}{6.7}
\newcommand{\NIclrRebEOneTwobQThreeEightFpEightptGptqAAdj}{7.8}
\newcommand{\NIclrRebEOneTwobQThreeEightFpEightptGptqHRaw}{2.1}
\newcommand{\NIclrRebEOneTwobQThreeEightFpEightptGptqHAdj}{4.0}
\newcommand{\NIclrRebEOneTwobQThreeEightFpEightptRotGptqARaw}{4.3}
\newcommand{\NIclrRebEOneTwobQThreeEightFpEightptRotGptqAAdj}{6.2}
\newcommand{\NIclrRebEOneTwobQThreeEightFpEightptRotGptqHRaw}{5.8}
\newcommand{\NIclrRebEOneTwobQThreeEightFpEightptRotGptqHAdj}{5.9}
\newcommand{\NIclrRebEOneTwobQThreeEightFpEightblkGptqARaw}{2.3}
\newcommand{\NIclrRebEOneTwobQThreeEightFpEightblkGptqAAdj}{3.2}
\newcommand{\NIclrRebEOneTwobQThreeEightFpEightblkGptqHRaw}{5.9}
\newcommand{\NIclrRebEOneTwobQThreeEightFpEightblkGptqHAdj}{7.3}
\newcommand{\NIclrRebEOneTwobQThreeEightFpEightblkRotGptqARaw}{1.2}
\newcommand{\NIclrRebEOneTwobQThreeEightFpEightblkRotGptqAAdj}{4.0}
\newcommand{\NIclrRebEOneTwobQThreeEightFpEightblkRotGptqHRaw}{4.3}
\newcommand{\NIclrRebEOneTwobQThreeEightFpEightblkRotGptqHAdj}{6.7}
\newcommand{\NIclrRebEOneTwobQOneFourSignflipPVsFbestTwo}{0.24}
\newcommand{\NIclrRebEOneTwobQOneFourSignflipPVsCubic}{0.40}
\newcommand{\NIclrRebEOneTwobLlEightSignflipPVsFbestTwo}{0.54}
\newcommand{\NIclrRebEOneTwobLlEightSignflipPVsCubic}{0.36}
\newcommand{\NIclrRebEOneTwobOlSevenSignflipPVsFbestTwo}{0.09}
\newcommand{\NIclrRebEOneTwobOlSevenSignflipPVsCubic}{0.32}
\newcommand{\NIclrRebEOneTwobQThreeEightSignflipPVsFbestTwo}{0.98}
\newcommand{\NIclrRebEOneTwobQThreeEightSignflipPVsCubic}{0.96}
\newcommand{\NIclrRebEOneThreebQThreeQspkMinusFpEightbestCiNineFiveLo}{-0.0027}
\newcommand{\NIclrRebEOneThreebQThreeQspkMinusFpEightbestCiNineFiveHi}{-0.0010}
\newcommand{\NIclrRebEOneThreebLlEightQspkMinusFpEightbestCiNineFiveLo}{-0.0065}
\newcommand{\NIclrRebEOneThreebLlEightQspkMinusFpEightbestCiNineFiveHi}{-0.0038}

\newcommand{\rev}[1]{{#1}}

\usepackage{array}
\usepackage{caption}
\usepackage{longtable}
\usepackage{placeins}
\makeatletter
\renewenvironment{table*}{\@float{table}}{\end@float}
\renewenvironment{figure*}{\@float{figure}}{\end@float}
\makeatother

\definecolor{exmuted}{gray}{0.42}
\newcommand{\exans}[1]{\textit{#1}}
\newcommand{\exok}{\,\checkmark}
\newcommand{\exgroup}[1]{\textit{#1}}
\newlength{\exitemw}
\newlength{\excorrw}
\newlength{\exmidw}
\newcommand{\exspan}[1]{\multicolumn{2}{>{\raggedright\arraybackslash}p{\dimexpr\exmidw+\excorrw+2\tabcolsep\relax}@{}}{#1}}
\newtheorem{theorem}{Theorem}
\newtheorem{proposition}{Proposition}
\newtheorem{lemma}{Lemma}
\newenvironment{proof}[1][Proof]{\par\noindent\textit{#1.}\ \ignorespaces}{\unskip\nobreak\hfill$\square$\par\medskip}

\newcommand{\paperstatus}{Preprint}
\newcommand{\dezhicorrespondenceemail}{dezhiran@pku.edu.cn}
\newcommand{\taocorrespondenceemail}{taoxie@pku.edu.cn}
\title{\normalfont\sffamily\bfseries Beyond Accuracy: Prefix-Invariant Realizations of \\Low-Precision Fast Matrix Multiplication}
\newcommand{\authornamefont}{\fontsize{11.8pt}{14.2pt}\selectfont\sffamily\bfseries}
\newcommand{\affiliationfont}{\fontsize{10.2pt}{12.4pt}\selectfont\normalfont}
\newlength{\authorrowbreakheight}
\makeatletter
\patchcmd{\@maketitle}
  {\begin{tabular}[t]{c}\bf\rule{\z@}{24\p@}\ignorespaces}
  {\begin{tabular}[t]{c}\authornamefont\rule{\z@}{24\p@}\ignorespaces}
  {}{\PackageWarning{arxiv-paper-template}{Could not patch the first author-row break}}
\patchcmd{\@maketitle}
  {\begin{tabular}[t]{c}\bf\rule{\z@}{24\p@}\ignorespaces}
  {\begin{tabular}[t]{c}\authornamefont\rule{\z@}{\authorrowbreakheight}\ignorespaces}
  {}{\PackageWarning{arxiv-paper-template}{Could not patch the second author-row break}}
\makeatother
\newcommand{\resourcelink}[2]{%
  \href{#1}{\normalfont\mdseries\nolinkurl{#2}}%
}

\author{
  \authornamefont
  \textbf{Shuxiao Xie}\textsuperscript{1,2}\thanks{Equal contribution.}
  \quad
  \textbf{Shuyang Xie}\textsuperscript{3}\footnotemark[1]
  \quad
  \textbf{Yuan Cao}\textsuperscript{1,4}
  \AND
  \textbf{Dezhi Ran}\textsuperscript{1}\thanks{Corresponding authors:
  \resourcelink{mailto:\dezhicorrespondenceemail}{\dezhicorrespondenceemail} and
  \resourcelink{mailto:\taocorrespondenceemail}{\taocorrespondenceemail}.}
  \quad
  \textbf{Wei Yang}\textsuperscript{2}
  \quad
  \textbf{Tao Xie}\textsuperscript{1,2,4,5}\footnotemark[2]
  \\[0.7em]
  \affiliationfont
  \textsuperscript{1}Beijing Tongming Lake Information Technology Application
  Innovation Center (TLAIC), China
  \\
  \textsuperscript{2}Fudan University Institute of Systems for Advanced Computing, China
  \\
  \textsuperscript{3}Harbin Institute of Technology, China
  \\
  \textsuperscript{4}Key Lab of HCST (PKU), MOE; SCS, Peking University, Beijing, China
  \\
  \textsuperscript{5}Shanghai Institute of Systems for Open Computing, China
}

\date{}
\renewcommand{\headeright}{\paperstatus}
\renewcommand{\undertitle}{\paperstatus}
\renewcommand{\shorttitle}{Beyond Accuracy}
\hypersetup{pdftitle={Beyond Accuracy: Prefix-Invariant Realizations of Low-Precision Fast Matrix Multiplication},pdfsubject={Public preprint},pdfauthor={Shuxiao Xie, Shuyang Xie, Yuan Cao, Dezhi Ran, Wei Yang, Tao Xie}}

\begin{document}

\maketitle

\begin{abstract}
Fast matrix multiplication saves multiplications through exact cancellation, but rounding sums that mix token rows can leave contributions from later tokens in earlier language model outputs. This threatens prefix invariance, which \rev{likelihood scoring and batched inference rely on: an output} must depend only on its allowed prefix. \rev{On Qwen2.5-14B-Instruct, two fast FP8 realizations repaired to near-bf16 accuracy, and the existing fast kernel FalconGEMM, change answers when only text outside the allowed prefix is replaced: on \NIclrRebEOneContFgFChanged{}--\NIclrActTwoQOneFourObqaPlausibleDnatOptTwo{}\% of OpenBookQA items scored by likelihood and \NIclrRebENineGenTextChangedMin{}--\NIclrRebESixBRepTwoTextChanged{}\% of NQ-open questions answered in a batch.} Accuracy thus does not certify prefix invariance, and the stability criteria we analyze cannot tell realizations apart: across all \NIclrActOneQOneFourOrbitN{} sign variants of two-level Strassen \rev{the criteria} stay constant while teacher-forced perplexities span a \NIclrActOneQOneFourOrbitSpread{}$\times$ range on the same model. We therefore construct certified realizations of two-level Strassen on bounded integer codes that quantize token rows independently, then mix and cancel exactly before rescaling, using \NIclrActFourStrassenTwoMultsFast{} block multiplications instead of \NIclrActFourStrassenTwoMultsClassical{}. Our certificate guarantees bitwise equality to a prescribed row-local classical int8 operator at the same quantization specification, so every certified realization inherits its prefix invariance. Certification thus turns realization choice into a pure cost decision: which certified realization runs can no longer change \rev{any output}. \rev{With optimized kernels, certified realizations run faster than classical int8 on an NVIDIA H20, and in the SGLang serving engine they raise throughput over its FP8 path by \NIclrRebSglBfOneSixAttnGainPctMin--\NIclrRebSglBfOneSixAttnGainPctMax\% on four models.}
\end{abstract}

\section{Introduction}\label{sec:intro}

Fast matrix multiplication saves multiplications: it multiplies sums of matrix blocks and lets the unwanted cross terms cancel exactly \citep{strassen1969}.
Low-precision kernels for large language model (LLM) inference now run such algorithms \citep{Jangda_2026,zhu2026falcongemm}, and in a linear layer the rows of the activation matrix are token positions, so these block sums add rows of different tokens.
In exact arithmetic the cross terms cancel and each output row depends only on its own input row; after rounding they can fail to cancel \citep{malik2019randomization}, and a later token can then contribute to an earlier token's output.
Such a contribution arises inside a linear layer, where the causal attention mask has no effect.
It violates prefix invariance \citep{kim2026mask}: what a model computes at a position must depend only on its allowed prefix, the text of its own sequence up to that position.
\rev{Likelihood scoring and batched inference rely on this property, since both put text outside a position's allowed prefix into the same forward pass: the later tokens of a scored answer, or other requests in the batch.}

\begin{figure}[t]
\centering
\begingroup
\makeatletter
\let\Gread@transgrouptrue\Gread@transgroupfalse
\makeatother
\includegraphics[width=\textwidth]{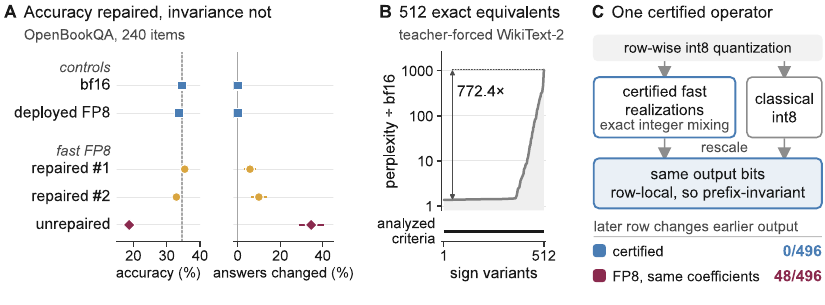}
\endgroup
\caption{\textbf{A}~Accuracy on the original input, and answers changed (95\% intervals) when the text after the allowed prefix is replaced by the bf16 model's greedy continuation; \#1, \#2: two fast algorithms, each with both repairs. \textbf{B}~Perplexity of the \NIclrActOneQOneFourOrbitN{} sign variants of two-level Strassen under one FP8 schedule, sorted, and the criteria we analyze. \textbf{C}~A certified realization and classical int8; strip: synthetic row pairs in which a later row changes an earlier output (FP8: our schedule, same coefficients). A, B: Qwen2.5-14B-Instruct.}
\label{fig:overview}
\end{figure}

The natural first check on a fast low-precision realization, an algorithm's coefficients together with how it is quantized and executed, is the model's accuracy.
In 8-bit floating point (FP8), fast algorithms can lose much of that accuracy (Figure~\ref{fig:overview}A), and two repairs adapted from prior work aim to restore it by changing how the algorithm is quantized \citep{Dumitrescu_1998,ballard2015improving,nagel2020down}.
With both repairs, FP8 realizations of two different fast algorithms bring Qwen2.5-14B-Instruct \citep{qwen25} to \NIclrActTwoQOneFourObqaPlausibleAccPctSxc\% and \NIclrActTwoQOneFourObqaPlausibleAccPctOptTwo\% accuracy on OpenBookQA \citep{mihaylov2018can}, against \NIclrActTwoQOneFourObqaPlausibleAccPctClean\% for the model in bfloat16 (bf16).
Yet when only the text after the allowed prefix is replaced by the bf16 model's own greedy continuation of that prefix, with the original answer still the one scored, they change the answer chosen by likelihood on \NIclrActTwoQOneFourObqaPlausibleDnatSxc\% and \NIclrActTwoQOneFourObqaPlausibleDnatOptTwo\% of \NIclrActTwoQOneFourObqaItems{} items, respectively (Figure~\ref{fig:overview}A)\ifhmode\unskip\fi.
\rev{The existing fast FP8 kernel FalconGEMM \citep{zhu2026falcongemm} changes the answer on \NIclrRebEOneContFgFChanged--\NIclrRebEOneContFgThreeChanged\% of them.}
Two row-local controls, the bf16 model and a deployed FP8 matrix multiplication kernel, change none.
The allowed prefix, input shapes and execution plan are unchanged, so the changed answers expose a dependence on text the answer must not see.
The same kind of dependence crosses requests packed into one forward pass\rev{: when only another request's text is replaced, one request's generated answer changes on \NIclrRebENineGenTextChangedMin--\NIclrRebESixBRepTwoTextChanged\% of \NIclrRebESixBItems{} NQ-open questions \citep{kwiatkowski2019natural,lee2019latent} under the same realizations (Section~\ref{sec:reach})}.

Since accuracy did not reveal these violations and rounding can produce them, the natural next check is the numerical analysis used to choose among fast algorithms.
It compares them by stability criteria computed from their coefficients, such as the parameters of a bound on rounding error \citep{ballard2015improving} or a measure of expected error \citep{xie2026fast}.
The criteria we analyze read only coefficient magnitudes.
Changing an algorithm's coefficient signs in ways that cancel out leaves the exact product unchanged \citep{malik2019randomization}; two-level Strassen, Strassen's algorithm applied to its own blocks, has \NIclrActOneQOneFourOrbitN{} such sign variants, and each criterion gives all of them one value \citep{malik2019randomization,dumas-automated-2025}.
Under the same FP8 schedule with block scaling, however, their teacher-forced perplexity on the same model spans a \NIclrActOneQOneFourOrbitSpread{}$\times$ range (Figure~\ref{fig:overview}B).

Accuracy did not flag the violations above, and the criteria we analyze did not even separate realizations whose quantized behavior differs widely.
A test like ours can expose a violation but not prove its absence, so a guarantee that a realization keeps prefix invariance has to come from how it is built.
The most direct construction forbids every intermediate product that feeds an output row from depending on a later row, but we prove that this forces the classical number of multiplications (Section~\ref{sec:obstruction}), which gives up the saving.
A realization that keeps the saving must therefore contain products that mix later rows into earlier ones, and prefix invariance can then hold only for its final output: the mixing must be undone exactly by the time that output is formed.

We therefore construct realizations of two-level Strassen that undo the mixing exactly, by carrying it out on bounded integer codes, whose sums do not round.
Where EmuGEMM uses this exactness to save multiplications through cancellation without rounding error \citep{ozaki2025ozaki,lu2026emugemm}, we use it to make a fast realization compute exactly what classical int8 matrix multiplication computes.
Each token row is quantized to integer codes on its own, the algorithm mixes and cancels these codes, and the scales are applied only afterwards.
Our certificate checks that the integer computation stays exact and that the scales can still be applied as in classical int8.
Every realization that passes computes, bit for bit, the classical int8 operator with the same codes and scales, rescaled and accumulated in the same order (Figure~\ref{fig:overview}C)\ifhmode\unskip\fi; since that operator is row-local, every certified realization inherits its prefix invariance.
All \NIclrActFourStrassenTwoGaugesN{} sign variants of two-level Strassen pass the certificate and compute the same output bits.

The certified realizations match classical int8 bit for bit while using \NIclrActFourStrassenTwoMultsFast{} block multiplications instead of \NIclrActFourStrassenTwoMultsClassical{}.
Certification \rev{thus} makes the choice between them a decision about cost alone: taking the saving no longer risks prefix invariance.
\rev{Turning the saving into speed is a separate task, which Section~\ref{sec:cost} takes up: its kernels run faster than classical int8 on an NVIDIA H20 and raise the throughput of the SGLang serving engine over its FP8 path.}

\section{Realizations and prefix invariance}\label{sec:contract}

\rev{Table~\ref{tab:terms} in Appendix~\ref{app:terms} collects the terms defined here and in Sections~\ref{sec:separation} and~\ref{sec:construction}.}

\subsection{Fast algorithms and their sign variants}\label{sec:signvariants}

A fast matrix multiplication algorithm for $C=AB$ splits $A$ into an $m\times k$ grid of blocks $A_{il}$ and $B$ into a $k\times n$ grid of blocks $B_{lj}$, multiplies $R$ pairs of block sums and combines the results:
\begin{equation}\label{eq:bilinear}
M_r=\Bigl(\sum_{i,l}u_{r,il}\,A_{il}\Bigr)\Bigl(\sum_{l,j}v_{r,lj}\,B_{lj}\Bigr)\quad(r=1,\dots,R),\qquad C_{ij}=\sum_{r=1}^{R}w_{r,ij}\,M_r .
\end{equation}
The coefficients $(U,V,W)$ make every unwanted cross term cancel; the classical algorithm uses $R=mkn$ block multiplications, and two-level Strassen, with $m=k=n=4$, uses \NIclrActFourStrassenTwoMultsFast{} instead of \NIclrActFourStrassenTwoMultsClassical{}.
Different coefficients can compute the same product.
Let $D$, $E$ and $F$ be diagonal sign matrices whose entries $d_i$, $e_l$ and $f_j$ flip whole block rows or columns.
Since $AB=D\,(DAE)(EBF)\,F$, running the algorithm on $DAE$ and $EBF$ and flipping the result back computes $AB$ exactly \citep{malik2019randomization}.
Absorbing the signs gives the coefficients $d_ie_lu_{r,il}$, $e_lf_jv_{r,lj}$ and $d_if_jw_{r,ij}$, with the same number of block multiplications; these coefficient sets are the algorithm's sign variants.
Negating all of $D$, $E$ or $F$ leaves every weighted product $w_{r,ij}M_r$ unchanged, so we fix $d_1=e_1=f_1=1$; two-level Strassen then has \NIclrActOneQOneFourOrbitN{} sign variants.

Quantized and executed the same way, the sign variants are exactly equivalent realizations, yet in low precision they can still differ in accuracy.
Choosing among such realizations is an old problem, from the fixed-point digital filters of \citet{Mullis_1976} to fast matrix multiplication, where the sign variants are a small part of a larger set of equivalences \citep{degroote1978varieties2} that \citet{dumas-automated-2025} search for a smaller bound on the rounding error.
\citet{malik2019randomization} randomize exactly these sign changes and find that random signs can lower the error on synthetic matrices.

\subsection{Quantization and the reference operator}\label{sec:spec}

Sign changes can alter the error because they change the values a low-precision realization rounds, and how it rounds them is set by how it is quantized and executed; for integer codes, as in classical int8 and our construction, a quantization specification $c$ spells these rules out.
On the way in, $c$ fixes how the operands become integer codes: the code bounds $b_A$ and $b_B$ (the largest code magnitudes); the length $g$ of the groups along the inner dimension that share a scale (block scaling); the rows or columns each scale belongs to; the rule that sets it; and the rounding to codes.
On the way out, it fixes how the group results become the output: the precision and exact order of operations in which they are rescaled and summed, and the output format.
We treat $c$ as part of the operator a realization computes: changing any of these, even the order of the final additions, can change the output bits.

Our reference, which a certified realization must reproduce, is the classical int8 operator at $c$: it multiplies the codes of each group exactly and rescales and sums the group results as $c$ prescribes.
Because $c$ gives each row of $A$ and each column of $B$ its own scale per group, set from that group's entries, row $t$ of the output reads only row $t$ of $A$, its codes and its scales: the operator is row-local.

\subsection{Prefix invariance}\label{sec:prefix}

Row locality matters for a property that accuracy does not certify: prefix invariance.
It requires that for any two inputs that agree on a position's allowed prefix, the text of its own sequence up to that position, the outputs at that position be the same.
The causal mask keeps attention from breaking it, but \citet{kim2026mask} show that the mask alone does not guarantee it: other operators that combine positions can break it too.
A linear layer keeps it when its matrix multiplication is row-local, like the reference operator.
The fast FP8 schedules of Section~\ref{sec:intro} are not: they round the block sums of $A$, and because the rows of $A$ are token positions, a block sum over different $i$ mixes tokens, and in a packed batch it can mix requests.

Prefix invariance differs from deterministic inference, which batch-invariant kernels now provide for LLM serving \citep{he2025defeating,lmsys2025deterministic}: deterministic inference asks that a request's output not change with how it is executed, for example with its batch size, while prefix invariance asks which text an output may depend on at a fixed execution plan and input shape.
Batch-invariant kernels fix the order of each sum, and order-independent summation \citep{ahrens2020algorithms} makes the result independent of it; both act on how a fixed set of numbers is added, while a fast algorithm changes what is added, so neither, by itself, keeps later text or another request's text out of an output.

\section{Exactly equivalent realizations are not interchangeable}\label{sec:separation}

Two checks are natural for a fast low-precision realization, the accuracy of the quantized model and the stability criteria computed from an algorithm's coefficients, and neither certifies prefix invariance.
Realizations repaired to accuracy close to the bf16 model's still change answers when only text they must not see is replaced (Sections~\ref{sec:support} and~\ref{sec:reach}), and the criteria we analyze do not even separate sign variants whose errors differ widely once quantized (Section~\ref{sec:magnitude}).

\subsection{Accuracy repair does not restore prefix invariance}\label{sec:support}

We run two-level Strassen in every linear layer of Qwen2.5-14B-Instruct \citep{qwen25} except the output layer, rounding its block sums to FP8 with the scales that the deployed FP8 matrix multiplication kernel of DeepGEMM uses for its own operands \citep{deepgemm2025}.
The model then answers only \NIclrActTwoQOneFourObqaFillerAccPctLibrary\% of the first \NIclrActTwoQOneFourObqaItems{} OpenBookQA test items \citep{mihaylov2018can} correctly, against \NIclrActTwoQOneFourObqaPlausibleAccPctClean\% in bf16 (Table~\ref{tab:support}).
We build two fast realizations whose accuracy comes back close to the bf16 model's, each by changing both the algorithm and how its block sums are rounded.
Repaired~\#1 runs one level of Strassen, and repaired~\#2 runs two levels of the seven-multiplication algorithm that the measure of expected error $\Phi$ of \citet{xie2026fast} ranks most accurate.
Both rescale each token row by a power of two before rows are summed, the outside scaling of \citet{Dumitrescu_1998} as \citet{ballard2015improving} describe it, and round the weight block sums jointly to reduce the output error, adapting \citet{nagel2020down} (rules in Appendix~\ref{app:scoring}).
The three fast realizations, unrepaired and repaired, are our compositions of DeepGEMM's FP8 routines\rev{; we also run an existing fast kernel as released, FalconGEMM \citep{zhu2026falcongemm}: one level of Strassen in FP8, in both of its variants, which combine the products into the output in a separate kernel or fused into the matrix product (Appendix~\ref{app:published})}.
We compare them with two row-local controls: the bf16 model and the deployed FP8 kernel, which scales each token row on its own and forms no block sums.

\rev{Accuracy close to the bf16 model's} does not show that the repaired realizations keep prefix invariance, so we test the property directly, following the audit design of \citet{kim2026mask}.
Likelihood scoring reads each token of a candidate answer at the position before it, in one forward pass, so an answer's later tokens are in the input when an earlier one is scored.
For each answer token we build a second input that keeps the token's allowed prefix and replaces the token itself and all later text.
The original token is still scored, from the position before it, in both inputs (Figure~\ref{fig:leak}A), and an answer's score under replacement sums these token scores.
The two inputs have the same shape and run under one execution plan, so a computation that keeps prefix invariance scores them identically.
The replacement is either filler, unrelated text shared by all answers of an item, or the bf16 model's greedy continuation of the allowed prefix, an ordinary future rather than an adverse one.
On all but \NIclrActTwoQOneFourObqaPlausibleWThreeRate\% of the answers, the continuation already differs from the real text at the answer's first token.
A changed answer is not run-to-run variation: on an identical input the unrepaired realization and repaired~\#2 repeat their logits exactly, and an independent second run \rev{in another container instance differs in some log probabilities but} reproduces every chosen answer of repaired~\#1 and of both controls.
\begin{figure}[t]
\centering
\includegraphics[width=\textwidth]{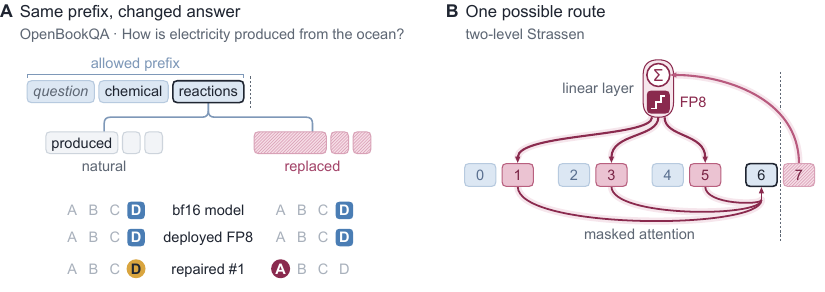}
\caption{\textbf{A}~One OpenBookQA item on Qwen2.5-14B-Instruct: the answer chosen with the natural text and with the bf16 model's greedy continuation after an illustrative cut; ``produced'' is scored at the outlined token. \textbf{B}~In two-level Strassen, block sums rounded to FP8 carry row 7 into rows 1, 3 and 5 (repaired~\#1 only into 3), and masked attention carries those rows on to position 6.}
\label{fig:leak}
\end{figure}

The repaired realizations reach \NIclrActTwoQOneFourObqaPlausibleAccPctSxc\% and \NIclrActTwoQOneFourObqaPlausibleAccPctOptTwo\% accuracy on OpenBookQA, against \NIclrActTwoQOneFourObqaPlausibleAccPctClean\% in bf16, yet change the answer chosen by likelihood on \NIclrActTwoQOneFourObqaPlausibleDnatSxc\% [\NIclrActTwoQOneFourObqaPlausibleDnatSxcLo, \NIclrActTwoQOneFourObqaPlausibleDnatSxcHi] and \NIclrActTwoQOneFourObqaPlausibleDnatOptTwo\% [\NIclrActTwoQOneFourObqaPlausibleDnatOptTwoLo, \NIclrActTwoQOneFourObqaPlausibleDnatOptTwoHi] of the items when the text after the allowed prefix is replaced by the greedy continuation (Table~\ref{tab:support})\ifhmode\unskip\fi.
\rev{FalconGEMM's separate and fused variants change \NIclrRebEOneContFgThreeChanged\% and \NIclrRebEOneContFgFChanged\% of the items, at \NIclrRebEOneContFgThreeAcc\% and \NIclrRebEOneContFgFAcc\% accuracy: an existing kernel fails the test as our realizations do.}
The row-local controls change none, under either replacement.
Under the same continuation, target-token log probabilities change at \NIclrTokObqaGreedyRepairOnePct\% and \NIclrTokObqaGreedyRepairTwoPct\% of scored positions for repaired~\#1 and \#2, respectively (Table~\ref{tab:tokenchanges})\ifhmode\unskip\fi.
The deployed FP8 kernel rounds to the same format, so the contrast lies within one FP8 format: unlike the controls, the repaired realizations \rev{and FalconGEMM} round sums that mix token rows.
Table~\ref{tab:support} shows the changes persisting under the filler, both on OpenBookQA and on ARC-Easy \citep{clark2018arc}.

\begin{table}[t]
\caption{Accuracy (\%) on the original input, and answers changed (\%) when the text after the allowed prefix is replaced by the bf16 model's greedy continuation or by filler, with 95\% bootstrap intervals over items for the answers changed; Qwen2.5-14B-Instruct, the first \NIclrActTwoQOneFourObqaItems{} test items of each task. \rev{FalconGEMM: an existing FP8 Strassen kernel, as released, combining the products into the output in a separate kernel or fused into the matrix product.}}
\label{tab:support}
\centering
\small
\setlength{\tabcolsep}{3.5pt}
\begin{tabular}{@{}l r r@{\hspace{2pt}}l r@{\hspace{2pt}}l @{\hspace{12pt}} r r@{\hspace{2pt}}l@{}}
\toprule
 & \multicolumn{5}{c}{OpenBookQA} & \multicolumn{3}{c}{ARC-Easy} \\
\cmidrule(lr){2-6}\cmidrule(l){7-9}
 & & \multicolumn{4}{c}{answers changed} & & \multicolumn{2}{c}{changed} \\
\cmidrule(lr){3-6}\cmidrule(l){8-9}
 & accuracy & \multicolumn{2}{c}{continuation} & \multicolumn{2}{c}{filler} & accuracy & \multicolumn{2}{c}{filler} \\
\midrule
bf16 model & \NIclrActTwoQOneFourObqaPlausibleAccPctClean & \NIclrActTwoQOneFourObqaPlausibleDnatClean & & \NIclrActTwoQOneFourObqaFillerDnatClean & & \NIclrActTwoQOneFourArcFillerAccPctClean & \NIclrActTwoQOneFourArcFillerDnatClean & \\
deployed FP8 kernel & \NIclrActTwoQOneFourObqaPlausibleAccPctCubic & \NIclrActTwoQOneFourObqaPlausibleDnatCubic & & \NIclrActTwoQOneFourObqaFillerDnatCubic & & \NIclrActTwoQOneFourArcFillerAccPctCubic & \NIclrActTwoQOneFourArcFillerDnatCubic & \\
\midrule
repaired \#1 & \NIclrActTwoQOneFourObqaPlausibleAccPctSxc & \NIclrActTwoQOneFourObqaPlausibleDnatSxc & {\footnotesize[\NIclrActTwoQOneFourObqaPlausibleDnatSxcLo,\,\NIclrActTwoQOneFourObqaPlausibleDnatSxcHi]} & \NIclrActTwoQOneFourObqaFillerDnatSxc & {\footnotesize[\NIclrActTwoQOneFourObqaFillerDnatSxcLo,\,\NIclrActTwoQOneFourObqaFillerDnatSxcHi]} & \NIclrActTwoQOneFourArcFillerAccPctSxc & \NIclrActTwoQOneFourArcFillerDnatSxc & {\footnotesize[\NIclrActTwoQOneFourArcFillerDnatSxcLo,\,\NIclrActTwoQOneFourArcFillerDnatSxcHi]} \\
repaired \#2 & \NIclrActTwoQOneFourObqaPlausibleAccPctOptTwo & \NIclrActTwoQOneFourObqaPlausibleDnatOptTwo & {\footnotesize[\NIclrActTwoQOneFourObqaPlausibleDnatOptTwoLo,\,\NIclrActTwoQOneFourObqaPlausibleDnatOptTwoHi]} & \NIclrActTwoQOneFourObqaFillerDnatOptTwo & {\footnotesize[\NIclrActTwoQOneFourObqaFillerDnatOptTwoLo,\,\NIclrActTwoQOneFourObqaFillerDnatOptTwoHi]} & \NIclrActTwoQOneFourArcFillerAccPctOptTwo & \NIclrActTwoQOneFourArcFillerDnatOptTwo & {\footnotesize[\NIclrActTwoQOneFourArcFillerDnatOptTwoLo,\,\NIclrActTwoQOneFourArcFillerDnatOptTwoHi]} \\
unrepaired & \NIclrActTwoQOneFourObqaFillerAccPctLibrary & \NIclrActTwoQOneFourObqaPlausibleDnatLibrary & {\footnotesize[\NIclrActTwoQOneFourObqaPlausibleDnatLibraryLo,\,\NIclrActTwoQOneFourObqaPlausibleDnatLibraryHi]} & \NIclrActTwoQOneFourObqaFillerDnatLibrary & {\footnotesize[\NIclrActTwoQOneFourObqaFillerDnatLibraryLo,\,\NIclrActTwoQOneFourObqaFillerDnatLibraryHi]} & \NIclrActTwoQOneFourArcFillerAccPctLibrary & \NIclrActTwoQOneFourArcFillerDnatLibrary & {\footnotesize[\NIclrActTwoQOneFourArcFillerDnatLibraryLo,\,\NIclrActTwoQOneFourArcFillerDnatLibraryHi]} \\
\midrule
\rev{FalconGEMM} & \rev{\NIclrRebEOneContFgThreeAcc} & \rev{\NIclrRebEOneContFgThreeChanged} & {\footnotesize\rev{[\NIclrRebEOneContFgThreeChangedLo,\,\NIclrRebEOneContFgThreeChangedHi]}} & \rev{\NIclrRebEOneFillFgThreeChanged} & {\footnotesize\rev{[\NIclrRebEOneFillFgThreeChangedLo,\,\NIclrRebEOneFillFgThreeChangedHi]}} & \rev{\NIclrRebEOneArcFgThreeAcc} & \rev{\NIclrRebEOneArcFgThreeChanged} & {\footnotesize\rev{[\NIclrRebEOneArcFgThreeChangedLo,\,\NIclrRebEOneArcFgThreeChangedHi]}} \\
\rev{FalconGEMM, fused} & \rev{\NIclrRebEOneContFgFAcc} & \rev{\NIclrRebEOneContFgFChanged} & {\footnotesize\rev{[\NIclrRebEOneContFgFChangedLo,\,\NIclrRebEOneContFgFChangedHi]}} & \rev{\NIclrRebEOneFillFgFChanged} & {\footnotesize\rev{[\NIclrRebEOneFillFgFChangedLo,\,\NIclrRebEOneFillFgFChangedHi]}} & \rev{\NIclrRebEOneArcFgFAcc} & \rev{\NIclrRebEOneArcFgFChanged} & {\footnotesize\rev{[\NIclrRebEOneArcFgFChangedLo,\,\NIclrRebEOneArcFgFChangedHi]}} \\
\bottomrule
\end{tabular}
\end{table}

For repaired~\#1 on OpenBookQA, the changes under the greedy continuation barely move accuracy: of its \NIclrActTwoExemplarsNMoved{} changed answers, \NIclrActTwoExemplarsNeitherIsGold{} are wrong before and after the change, \NIclrActTwoExemplarsNaturalIsGold{} go from correct to wrong and \NIclrActTwoExemplarsBlindIsGold{} from wrong to correct.
Figure~\ref{fig:leak}A shows one of them, a switch from one wrong answer to another, while both controls keep theirs.
Accuracy is therefore blind to most of these changes, and so is a count of changes in correctness, such as the flips that \citet{dutta2024accuracy} use to compare a compressed model with its original.

\subsection{How the violation can arise and how far it extends}\label{sec:reach}

Later text can reach a scored position through the linear layers, where the causal mask has no effect.
Two-level Strassen splits the rows of a linear layer's input, which are token positions, into four contiguous blocks after padding with zeros, and its block sums add rows at the same offset in different blocks (Figure~\ref{fig:leak}B).
In exact arithmetic the other rows cancel from every output row; after FP8 rounding they need not \citep{malik2019randomization}, so an output row can depend on the later rows at its offset, as a test of a single multiplication confirms (Appendix~\ref{app:rowmap}).
Repaired~\#2 adds the same rows and repaired~\#1, with one level, adds pairs of them; rescaling and joint rounding change the scale and rounding of these sums, not which rows they add.
Later text can thus reach a scored token directly, when a later row shares its offset, or indirectly: it can change earlier rows, and correctly masked attention can carry those rows forward to the scored token.

The dependence is not confined to multiple-choice scoring or to one model: on five models of four families \citep{qwen25,qwen3,llama3,olmo2024olmo}, replacing all but the first quarter of each 1,024-token window changes \NIclrActTwoCFourSixThreeUnrepairedFlipPctMin--\NIclrActTwoCFourSixThreeUnrepairedFlipPctMax\% of the next-token predictions within that quarter under the unrepaired realization, \NIclrActTwoCFourSixThreeRepairedFlipPctMin--\NIclrActTwoCFourSixThreeRepairedFlipPctMax\% under the repaired ones\rev{, whose next-token distributions themselves change at \NIclrRebESixAShareMin\% of those positions or more}, and none under the controls \rev{(Appendix~\ref{app:windows})}.

The same dependence crosses requests: we pack two requests into one forward pass of fixed shape, with attention separated between them and no key--value cache, and replace the second request's text with other text of the same length \rev{(Appendices~\ref{app:packed}--\ref{app:engine})}.
Over \NIclrActTwoCFourSixFourPairs{} such pairs on three models, the first request's next-token distribution changes in every pair under the unrepaired realization\rev{,} repaired~\#2 \rev{and FalconGEMM}, and never under the controls.
Under repaired~\#2 its greedy next token changes in \NIclrActTwoCFourSixFourRepairedGreedyPctMin--\NIclrActTwoCFourSixFourRepairedGreedyPctMax\% of the pairs\rev{, and under FalconGEMM in \NIclrRebEOnePackedFlipMin--\NIclrRebEOnePackedFlipMax\%}.
\rev{With attention separated, linear layers can still connect the two requests:} on Qwen2.5-3B the first request's key and value projections differ at the first layer, before any attention, but not under the controls.

\rev{In batched inference, the answer a request receives can thus depend on the other requests in its batch: when the first request answers \NIclrRebESixBItems{} NQ-open questions, its generated answer changes with the second request's text in \NIclrRebGenRepTextChangedMin--\NIclrRebGenRepTextChangedMax\% of them under the repaired realizations on Qwen2.5-14B-Instruct and Llama-3.1-8B-Instruct, and in \NIclrRebENineGenTextChangedMin--\NIclrRebENineGenTextChangedMax\% under FalconGEMM on Qwen2.5-14B-Instruct; whether the answer is correct changes in \NIclrRebGenCorrChangedMin--\NIclrRebGenCorrChangedMax\%, and under the controls neither changes.
A request's output can even identify another request's text: on Qwen2.5-14B-Instruct, the first request's five most likely tokens, with log probabilities rounded to 0.01, identify which of \NIclrRebESevenK{} candidate texts the second request held in every case under either repaired realization, and in \NIclrRebENineIdentTopFiveHeldoutMin\% of the cases or more under FalconGEMM.
The dependence also arises in a serving engine's own batching: in SGLang, under our fast FP8 Strassen path, replacing only the co-scheduled request's text changes the first request's log probabilities in all \NIclrRebEEightComparisons{} comparisons per model and its 32-token greedy continuation in \NIclrRebEEightLlEightScCOut{} (Llama-3.1-8B-Instruct) and \NIclrRebEEightQThreeScCOut{} (Qwen3-8B), and under SGLang's FP8 path or the certified one of Section~\ref{sec:cost} changes neither.}

\subsection{What the stability criteria cannot see}\label{sec:magnitude}

The second check comes from the numerical analysis of fast algorithms, which studies their rounding error, where violations of prefix invariance can arise, and compares algorithms by stability criteria computed from their coefficients.
The coefficients alone, however, do not decide prefix invariance: with the same coefficients, two-level Strassen keeps it in exact arithmetic and on bounded integer codes, whose sums do not round (Section~\ref{sec:construction}), and breaks it under the FP8 schedule of Section~\ref{sec:support}.

The criteria we analyze are $\Phi$, the number of nonzero coefficients, the largest coefficient magnitude, and the parameters $Q$ and $E$ of the bound on rounding error of \citet{ballard2015improving}, all defined in Appendix~\ref{app:criteria}, and they do not even separate the sign variants of one algorithm by their error.
They read the coefficients only through their absolute values or their nonzero pattern, which a sign change leaves unchanged, so all \NIclrActOneQOneFourOrbitN{} sign variants of two-level Strassen get one value; \citet[Proposition~10]{malik2019randomization} and \citet[Lemma~26]{dumas-automated-2025} prove this for the criteria in their own error bounds: the bounds stay valid, but equal bounds do not promise equal realized error.

Measuring the realized error on random matrices, as the checks that accompany fast kernels and their error analyses do \citep{Jangda_2026,dai2022numerical,lu2026emugemm}, would not separate the sign variants either when the entries are independent.
These checks draw entries symmetric about zero, and for independent entries a sign variant computes, up to exact sign flips of the output, what the original algorithm computes on $DAE$ and $EBF$, which are distributed exactly as $A$ and $B$ (Section~\ref{sec:signvariants}), so all variants' error norms have the same distribution.
With entries uniform on $[0,1]$ the symmetry breaks and the variants' errors differ: how they compare depends on the operands, so we run every variant inside a model, on its own activations.

Under the FP8 schedule of Section~\ref{sec:support}, in every linear layer of Qwen2.5-14B-Instruct except the output layer, we score each of the \NIclrActOneQOneFourOrbitN{} sign variants by its teacher-forced perplexity on WikiText-2 \citep{merity-pointer-2016}, divided by the perplexity of the model in bf16.
This ratio runs from \NIclrActOneQOneFourOrbitMin{} to \NIclrActOneQOneFourOrbitMax{}, a \NIclrActOneQOneFourOrbitSpread{}$\times$ spread (Figure~\ref{fig:overview}B), and the variant with Strassen's original signs, the unrepaired realization of Section~\ref{sec:support}, sits at \NIclrActOneQOneFourOrbitIdentity{}$\times$.
A single sign, which none of the criteria reads, can make nearly all of this difference: two variants that differ only in one of the signs $d_i$, $e_l$, $f_j$ of Section~\ref{sec:signvariants} have ratios \NIclrActOneQOneFourEdgeMaxFactor{} times apart.
Appendix~\ref{app:magnitude} reports the spread on seven models, with three controls.

Teacher forcing reads every position in one forward pass, with the text after a position, including the token it predicts, in the input; under a realization that is not row-local, a position's score can therefore depend on that text.
Replacing that text with filler chosen independently of the predicted token, for four variants across the range (the smallest, the median, Strassen's original signs and the largest), leaves their spread at \NIclrActOneCFourEightEightRatio{} times its teacher-forced value on the same scored positions (95\% confidence interval [\NIclrActOneCFourEightEightCiLo, \NIclrActOneCFourEightEightCiHi], bootstrap over text windows): the spread survives the replacement.

Exactly equivalent realizations can thus differ once quantized, in what their outputs depend on and in how large their error is; accuracy misses the first and the criteria we analyze miss both.
\rev{Section~\ref{sec:construction} removes both differences at once: on bounded integer codes every certified sign variant computes the same operator bit for bit, so which variant runs no longer changes the output.}

\section{Certified realizations on bounded integer codes}\label{sec:construction}

Prefix invariance constrains a realization's outputs, not each product it forms on the way.

\subsection{Prefix-safe products force the classical count}\label{sec:obstruction}

Whether a product keeps later rows out of the outputs it feeds can be read off the coefficients of~(\ref{eq:bilinear}), because the rows of $A$ are token positions, split into $m$ contiguous blocks (Section~\ref{sec:reach}).
A product $M_r$ reads block row $i$ of $A$ if $u_{r,il}\neq0$ for some $l$, and feeds block row $i$ of $C$ if $w_{r,ij}\neq0$ for some $j$.
We call an algorithm prefix-safe product by product if no product reads a block row that comes after a block row it feeds.
Such an algorithm keeps later rows out of earlier outputs even when each row of its block sums is rounded with its own scale, but it cannot keep the saving.

\begin{theorem}\label{thm:prefixsafe}
An algorithm of the form~(\ref{eq:bilinear}) that is prefix-safe product by product uses $R\ge mkn$ block multiplications, and the classical algorithm attains the bound (proof in Appendix~\ref{app:proofs}).
\end{theorem}

Two-level Strassen, with \NIclrActFourStrassenTwoMultsFast{} products instead of \NIclrActFourStrassenTwoMultsClassical{}, therefore has products that carry a later block row into an earlier output, and the contributions of those rows must vanish from every earlier output by the time it is formed, as they do in exact arithmetic.

\subsection{Computing classical int8 with a fast algorithm}\label{sec:certificate}

\begin{figure}[t]
\centering
\includegraphics[width=\textwidth]{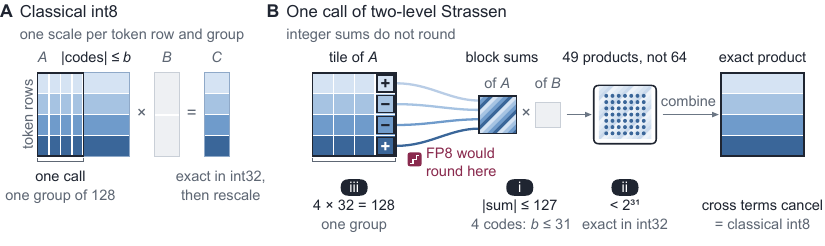}
\caption{\textbf{A}~Classical int8 on one group of 128. \textbf{B}~One call of two-level Strassen on the same codes; (i)--(iii) mark where the conditions of Proposition~\ref{prop:certificate} bind.}
\label{fig:method}
\end{figure}

Our construction builds on classical int8 matrix multiplication, which already multiplies integer codes without rounding.
At a quantization specification $c$ (Section~\ref{sec:spec}), it gives each row of $A$ and each column of $B$ its own scale for every group of $g$ consecutive indices of the inner dimension.
It multiplies the codes of each group exactly, and only then applies the scales and adds the group results in the order $c$ fixes (Figure~\ref{fig:method}).
Our realizations keep all of this except the exact product of the codes, which they compute with the fast algorithm on one submatrix, or tile, of $A$ and one of $B$ at a time; we call each such computation a call.
Their output bits are those of classical int8 if every call returns its integer product exactly and if the scales can still be applied after it.
Rows can keep their own scales, since an exact product builds each output row from its own row's codes alone.
The inner dimension is another matter: every fast algorithm mixes its indices too, since one whose products each read a single block column of $A$ needs $mkn$ of them (Theorem~\ref{thm:onecolumn}), and independent scales along it cannot in general be applied after products that do not depend on them (Appendix~\ref{app:proofs}).

Our certificate consists of three conditions that secure both requirements, and we call a realization that meets them certified.
They involve the largest coefficient 1-norms of the block sums, $L_A=\max_r\sum_{i,l}|u_{r,il}|$ and $L_B=\max_r\sum_{l,j}|v_{r,lj}|$, the largest total weight $L_W=\max_{i,j}\sum_r|w_{r,ij}|$ with which an output block combines products, and the inner length $h$ of a block.

\begin{proposition}\label{prop:certificate}
Let a realization quantize, rescale and accumulate as the classical int8 operator at a quantization specification $c$ does, and compute the integer product of the codes tile by tile, each call with an algorithm of the form~(\ref{eq:bilinear}) with integer coefficients.
Suppose that
(i)~$L_Ab_A\le127$ and $L_Bb_B\le127$;
(ii)~$\max\{L_Wh(L_Ab_A)(L_Bb_B),\,gb_Ab_B\}<2^{31}$ for int32 accumulation, or $<2^{24}$ for exact integers in fp32; and
(iii)~each call lies within one group along the inner dimension.
Then on every input the realization computes the output of the classical int8 operator at $c$, bit for bit.
\end{proposition}

Conditions~(i) and~(ii) secure the first requirement: every block sum of codes stays within int8, so the products run on the same int8 multipliers as classical int8, and every product and partial sum, the running sum over a group's calls included, stays within the accumulator.
Condition~(iii) secures the second, since a call that stays within one group contributes only to that group's integer product, which is then rescaled as in classical int8.
\rev{Conditions~(i) and~(ii) bound the intermediates, as analyses of exact integer fast multiplication do \citep[Theorem~3.1]{dumas2009fflas}; condition~(iii) is what the scales of $c$ add (Lemma~\ref{lem:scales}), and with it the certificate covers the whole quantized operator, its rescaling order and output format included, and not only the integer product.}
The conditions involve only the coefficients, the tiling and $c$, so they are checked once, before any input is seen.
Every certified realization is therefore row-local, like classical int8, and keeps prefix invariance on every input (Section~\ref{sec:prefix}).
Two-level Strassen has $L_A=\NIclrActFourStrassenTwoHeadroomLA$ and $L_B=\NIclrActFourStrassenTwoHeadroomLB$, so condition~(i) admits codes up to $b_A=\NIclrActFourStrassenTwoHeadroomBA$ and $b_B=\NIclrActFourStrassenTwoHeadroomBB$.
With $L_W=\NIclrActFourStrassenTwoHeadroomLW$, int32 accumulation is exact for blocks of inner length up to \NIclrActFourStrassenTwoExactMaxLeafKIntThreeTwo, and with our groups of $g=128$ we use blocks of inner length 32, so the four block columns of a call span one group.
Two-level Strassen is thus certified at code bound \NIclrActFourStrassenTwoHeadroomBA, with \NIclrActFourStrassenTwoMultsFast{} block multiplications, \NIclrActFourStrassenTwoMultsRatio{} of the classical count.
\rev{One level of Strassen, with $L_A=L_B=2$, is certified in the same way at code bound 63 with 7 block multiplications instead of 8, and it is the configuration we deploy (Sections~\ref{sec:quality} and~\ref{sec:cost}).}

\rev{The integer product of two-level Strassen} equals the classical one bit for bit on \NIclrActFourStrassenTwoExactTrials{} random tiles of codes; on a separately drawn tile the \NIclrActFourStrassenTwoGaugesN{} sign variants give \NIclrActFourStrassenTwoGaugesIntDistinct{} distinct integer product, where rounding the block sums to FP8 leaves \NIclrActFourStrassenTwoGaugesEFourmThreeExact{} of them exact and a \NIclrActFourStrassenTwoGaugesEFourmThreeLotteryRatio{}-fold spread in the largest error, and with the scales applied as well, under the full quantization specification on inputs of 32 token rows in four blocks, a subset of \NIclrActFourStrassenTwoScheduleGauges{} variants again gives \NIclrActFourStrassenTwoScheduleIntDistinct{} distinct output.

We then test prefix invariance directly, on separately drawn Gaussian inputs of that shape: changing one token row at a time and comparing every earlier output row moves \NIclrActFourStrassenTwoPrefixMovedInt{} of the \NIclrActFourStrassenTwoPrefixPairs{} pairs of an earlier and a later row\ifhmode\unskip\fi.
The same coefficients move \NIclrActFourStrassenTwoPrefixMovedFpEight{} pairs under an FP8 schedule that rounds the block sums before multiplying, exactly those whose rows share an offset in their blocks (Figure~\ref{fig:leak}B).
A control that breaks condition~(iii), with scales on groups of 32 indices while a call spans 128, rounds each mixed block sum again with one scale, set by its largest magnitude, which any of its rows can change, and all \NIclrActFourStrassenTwoPrefixMovedFinegroupControl{} pairs move.

Inside a model, we run two-level Strassen at code bound \NIclrActFourStrassenTwoHeadroomBA{} in every linear layer except the output layer of four models, Qwen2.5-14B-Instruct, Llama-3.1-8B-Instruct, OLMo-2-7B-Instruct and Qwen3-8B, emulating its int8 arithmetic exactly in fp32, where every intermediate stays below $2^{24}$.
Each time a replaced layer runs, its output matches classical int8's bit for bit (Appendix~\ref{app:construction}), so the remaining model measurements run that classical int8 operator, which gives the same outputs.
On the test that exposed the repaired FP8 realizations, replacing the text after the allowed prefix with the greedy continuation (Section~\ref{sec:support}), the operator changes none of the \NIclrActTwoQOneFourObqaItems{} OpenBookQA answers on any of the four models, like the row-local controls.

\subsection{Full-range codes with an overflow correction}\label{sec:correction}

\rev{Condition~(i) sets the price of certification on int8 multipliers: a realization certified by it} can use codes only up to $\lfloor127/L_A\rfloor$ in $A$ and $\lfloor127/L_B\rfloor$ in $B$, \NIclrActFourStrassenTwoHeadroomBA{} for two-level Strassen (Appendix~\ref{app:construction}).

The overflow behind this price can be corrected instead of avoided.
Each block sum of codes $\hat A$ splits exactly into $\hat A=A_0+256R_A$, with $A_0$ the remainder in the int8 range and $R_A$ nonzero only where $\hat A$ overflows, and likewise $\hat B=B_0+256R_B$, so that
\begin{equation}\label{eq:correction}
\hat A\hat B=A_0B_0+256\,(R_AB_0+A_0R_B)+256^2R_AR_B .
\end{equation}
Two-level Strassen then forms its products on $A_0$ and $B_0$ and adds the correction terms of~(\ref{eq:correction}), which take further products, counted in Section~\ref{sec:cost}.
With this correction it computes classical int8 at code bound 127 bit for bit at the same groups of 128, under a certificate of its own (Appendix~\ref{app:construction}).

\subsection{Model quality and the choice of realization}\label{sec:quality}

\begin{figure}[t]
\centering
\includegraphics[width=\textwidth]{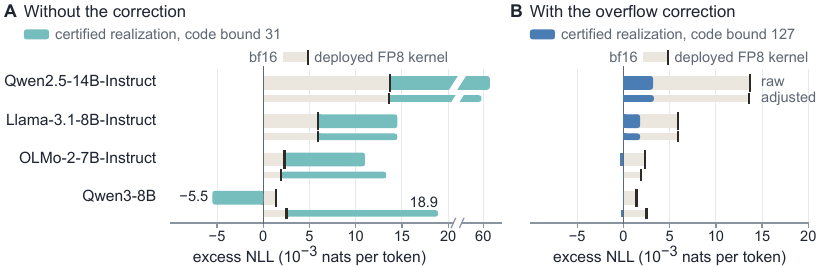}
\caption{Excess NLL per token over bf16 on eight 2048-token chunks of WikiText-2, at code bound \NIclrActFourStrassenTwoHeadroomBA{} (\textbf{A}) and at 127 with the overflow correction (\textbf{B}); thick bars raw, thin bars with temperatures fitted on the other seven chunks, bands ending at the deployed FP8 kernel's value. Certified values are measured on the classical int8 operators the realizations compute (Table~\ref{tab:quality}).}
\label{fig:quality}
\end{figure}

Bitwise identity fixes which operator a certified realization computes, not how close that operator comes to the bf16 model.
We measure that closeness for the operators the realizations compute, classical int8 at code bound \NIclrActFourStrassenTwoHeadroomBA{} and, with the correction, at 127, by their excess negative log-likelihood (NLL) per token over the bf16 model on eight chunks of WikiText-2, against the deployed FP8 kernel (Figure~\ref{fig:quality}).
Since a change of temperature alone can move such a comparison \citep{guo2017calibration}, we report the excess both raw, at temperature 1, and adjusted, with every operator and the bf16 model at a temperature fitted without the scored chunk.
The raw excess, which scores the models as they are run, is the primary reading.
At code bound \NIclrActFourStrassenTwoHeadroomBA{} the excess is larger than the kernel's on three of the four models raw, all but Qwen3-8B, and on all four adjusted.
With the overflow correction, at code bound 127, the excess is at or below the kernel's on all four models, raw and adjusted.
\rev{The configuration that Section~\ref{sec:cost} deploys, one level of Strassen at code bound 63 with an online Hadamard rotation of the activations \citep{ashkboos-quarot-2024} and GPTQ weight codes \citep{frantar2023gptq}, needs no correction and has an excess at or below the kernel's on all four models, raw and adjusted (Appendix~\ref{app:construction}).}
Table~\ref{tab:obqa-quality} also reports ordinary OpenBookQA accuracy for these operators.

\section{\texorpdfstring{Cost and \rev{speed}}{Cost and speed}}\label{sec:cost}

A certified realization and classical int8 differ in two costs, the multiplications they perform and the time they take\rev{, both detailed in Appendix~\ref{app:cost}}.
In multiplications, two-level Strassen at code bound \NIclrActFourStrassenTwoHeadroomBA{} performs \NIclrActFourStrassenTwoMultsRatio{} of classical int8's multiplications\rev{, and one level, which condition~(i) certifies at code bound 63, performs 7/8 of them}.
Raising the code bound to 127 brings the excess NLL to or below the deployed FP8 kernel's on all four models (Section~\ref{sec:quality}), but the correction of Section~\ref{sec:correction} then adds products wherever a block sum overflows.
If the correction multiplies only the entries that overflow, the corrected computation performs \rev{\NIclrRebEFourModelEntryMin--\NIclrRebEFourModelEntryMax{} of classical int8's multiplications over all linear layers of each of the four models but the output layer, and fewer than classical int8 on each of these \NIclrRebEFourLayersN{} layers}.
If it multiplies every row of $R_A$ and column of $R_B$ that holds an overflow, it performs \rev{\NIclrRebEFourModelRowsMin--\NIclrRebEFourModelRowsMax{} times as many as classical int8, because the overflows are spread thinly over rows and columns}.
Whether the full code range keeps the saving thus depends on how the correction is implemented\rev{: applied entry by entry, the corrected computation saves about as much as one level of Strassen at code bound 63, which needs no correction and is the configuration we deploy below}.

\rev{In time, certification leaves the kernel free: any optimization that preserves the quantization specification and remains certified yields the same output bits.
We wrote our first kernels in Triton \citep{tillet2019triton} and timed them on two $4096\times4096$ matrices on an NVIDIA H20 graphics processing unit (GPU).
They took \NIclrCostStrassenOneQGOneTwoEight--\NIclrCostStrassenOneQGTwoFiveSix{} times as long as classical int8 for one level of Strassen and \NIclrCostTwoLevelVsClassical{} times as long for two.
Most of this time went to forming the block sums inside the kernel and to splitting a call into short block products.
Our optimized kernels remove this overhead by forming the block sums in a separate pass and running a call's products back to back on the tensor cores.
We time them on the same GPU at 8,192 token rows, with one scale per token row and per output channel as deployed below.
The one-level kernel reads activation block sums prepared with the quantization, and under sustained load at eight projection shapes of 8B to 70B models it takes \NIclrRebKOneKTwoFiveSixSustMin--\NIclrRebKOneKTwoFiveSixSustMax{} of classical int8's time, against the 7/8 its multiplications set.
The two-level kernel forms these block sums itself and takes \NIclrRebKTwoLayerMin--\NIclrRebKTwoLayerMax{} of classical int8's time per layer on four such models.
Every timed call returns classical int8's output bits.}

\rev{The saving carries through a serving engine.
In SGLang \citep{zheng2024sglang}, we run one level of Strassen at code bound 63 in every projection except the query--key--value projection, which runs classical int8.
As in Section~\ref{sec:quality}, the activations are rotated online by a Hadamard transform, and the weight codes are computed with GPTQ.
Within an 8,192-token prefill batch of an 8B model, the fast projections take \NIclrRebKOneBatchEightBFastprojMin--\NIclrRebKOneBatchEightBFastprojMax{} of their time under classical int8.
On four models with 8B to 70B parameters and four workloads, from long-prompt prefill to multi-turn agent sessions, throughput exceeds that of SGLang's own FP8 path by \NIclrRebSglBfOneSixAttnGainPctMin--\NIclrRebSglBfOneSixAttnGainPctMax\%.
When both paths also use FP8 attention and an FP8 key--value cache, the gain is \NIclrRebSglFpEightAttnGainPctMin--\NIclrRebSglFpEightAttnGainPctMax\% on the 8B and 70B models.
With bf16 attention, the excess NLL per token over bf16 is \NIclrRebNllBfOneSixAttnIntEightOverFpEightMinEightBTwoSevenB--\NIclrRebNllBfOneSixAttnIntEightOverFpEightMaxEightBTwoSevenB{} of the FP8 path's on the 8B and 27B models and \NIclrRebNllBfOneSixAttnLlSevenZeroIntEightOverFpEight{} of it on the 70B model.
Every fast call is certified, so the served model's linear layers stay row-local, as classical int8's are.
Kernels for other GPUs are future work.}

\section{Conclusion}\label{sec:conclusion}

\rev{A fast matrix multiplication inside a language model has to keep prefix invariance, which likelihood scoring and batched inference rely on.
We show that accuracy does not certify it: FP8 realizations repaired to near-bf16 accuracy, and an existing fast kernel, change answers when only text outside the allowed prefix is replaced.
Saving multiplications requires products that mix token rows, but prefix invariance constrains only the output and can therefore be certified without giving up the saving.
We construct certified realizations of Strassen's algorithm on bounded integer codes: they compute the row-local classical int8 operator bit for bit at the same quantization specification, with 7 block multiplications instead of 8 at one level, which we deploy, and \NIclrActFourStrassenTwoMultsFast{} instead of \NIclrActFourStrassenTwoMultsClassical{} at two levels.
Certification thus turns realization choice into a pure cost decision, and on an NVIDIA H20 our certified kernels run faster than classical int8 and raise SGLang's throughput over its FP8 path.}

\bibliography{references}
\bibliographystyle{plainnat}

\appendix
\section{Proofs for Section 4}\label{app:proofs}

\subsection{Prefix-safe products force the classical count}\label{app:prefixsafe}

Within one sequence, the most direct way to guarantee prefix invariance is to require every product to keep later rows out of the outputs it feeds, and Theorem~\ref{thm:prefixsafe} states what that costs: an algorithm of the form~(\ref{eq:bilinear}) that is prefix-safe product by product, so that no product reads a block row of $A$ that comes after a block row of $C$ it feeds, uses $R\ge mkn$ block multiplications, as many as the classical algorithm.
The proof counts, for each block row, the products that both read it and feed it, and shows that prefix safety forbids two block rows to share such a product.

The count depends only on the coefficients, which also give an algorithm for single-entry blocks.
Setting every block $A_{il}=x_{il}X$ and $B_{lj}=y_{lj}Y$, for fixed matrices $X$ and $Y$ with $XY\neq0$, makes every term of~(\ref{eq:bilinear}) a multiple of $XY$, and comparing the multiples shows that the coefficients multiply the $m\times k$ matrix $(x_{il})$ by the $k\times n$ matrix $(y_{lj})$ with $R$ multiplications of numbers.
We may therefore take $A\in\mathbb{R}^{m\times k}$ and $B\in\mathbb{R}^{k\times n}$, so that block rows and block columns become rows and columns, and product $r$ is the number $M_r=f_r(A)\,g_r(B)$ with $f_r(A)=\sum_{i,l}u_{r,il}A_{il}$ and $g_r(B)=\sum_{l,j}v_{r,lj}B_{lj}$.
We write $E_{il}$ for the matrix, of the size the context requires, whose only nonzero entry is a~1 in position $(i,l)$, and $I$ for the identity matrix.

\begin{proof}[Proof of Theorem~\ref{thm:prefixsafe}]
For each row $i$, let $\mathcal{R}_i$ be the set of products that read row $i$ and feed row $i$.
Every algorithm of the form~(\ref{eq:bilinear}), prefix-safe or not, has at least $kn$ products in each $\mathcal{R}_i$.
To see this, take $A=E_{il}$: then $f_r(A)=u_{r,il}$, and row $i$ of $AB$ is row $l$ of $B$, so that
\[
B_{lj}=(AB)_{ij}=\sum_{r=1}^{R}u_{r,il}\,w_{r,ij}\,g_r(B)\qquad\text{for every }B .
\]
A term of this sum can be nonzero only if product $r$ reads row $i$ and feeds it, that is, only if $r\in\mathcal{R}_i$.
Letting $l$ and $j$ vary, each of the $kn$ coordinate functions $B\mapsto B_{lj}$ is therefore a linear combination of the functions $g_r$ with $r\in\mathcal{R}_i$; since the coordinate functions are linearly independent, $\mathcal{R}_i$ has at least $kn$ members.
Prefix safety makes the sets disjoint: a product in both $\mathcal{R}_i$ and $\mathcal{R}_{i'}$, with $i<i'$, feeds row $i$ and reads the later row $i'$.
Hence $R\ge|\mathcal{R}_1|+\dots+|\mathcal{R}_m|\ge mkn$.
The classical algorithm attains the bound, since its product $A_{il}B_{lj}$ reads row $i$ and feeds only row $i$.
\end{proof}

Strassen's algorithm, with $m=k=n=2$, shows the sharing that the bound rules out.
Its seven products, with $C_{11}=M_1+M_4-M_5+M_7$, $C_{12}=M_3+M_5$, $C_{21}=M_2+M_4$ and $C_{22}=M_1-M_2+M_3+M_6$, read and feed these rows:
\begin{center}
\small
\begin{tabular}{@{}l l c c@{}}
\toprule
 & product & reads rows of $A$ & feeds rows of $C$ \\
\midrule
$M_1$ & $(A_{11}+A_{22})(B_{11}+B_{22})$ & 1, 2 & 1, 2 \\
$M_2$ & $(A_{21}+A_{22})\,B_{11}$ & 2 & 2 \\
$M_3$ & $A_{11}\,(B_{12}-B_{22})$ & 1 & 1, 2 \\
$M_4$ & $A_{22}\,(B_{21}-B_{11})$ & 2 & 1, 2 \\
$M_5$ & $(A_{11}+A_{12})\,B_{22}$ & 1 & 1 \\
$M_6$ & $(A_{21}-A_{11})(B_{11}+B_{12})$ & 1, 2 & 2 \\
$M_7$ & $(A_{12}-A_{22})(B_{21}+B_{22})$ & 1, 2 & 1 \\
\bottomrule
\end{tabular}
\end{center}
Four products read and feed row~1, $M_1$, $M_3$, $M_5$ and $M_7$, and four read and feed row~2, $M_1$, $M_2$, $M_4$ and $M_6$, as the proof requires.
The two sets share one product, $M_1$, which is how seven products meet a count of eight.
A product shared by two rows always breaks prefix safety, since it feeds the earlier row and reads the later one: $M_1$ feeds row~1 and reads row~2.

\subsection{The inner dimension and its scales}\label{app:inner}

Theorem~\ref{thm:prefixsafe} forces a fast algorithm's products to mix rows.
They must also mix indices of the inner dimension, where classical int8 applies the scales of its groups.
A product reads block column $l$ of $A$ if $u_{r,il}\neq0$ for some $i$, and block row $l$ of $B$ if $v_{r,lj}\neq0$ for some $j$.

\begin{theorem}\label{thm:onecolumn}
An algorithm of the form~(\ref{eq:bilinear}) in which no product reads two block columns of $A$ uses $R\ge mkn$ block multiplications, and so does one in which no product reads two block rows of $B$.
\end{theorem}

This time the count runs over the output: on inputs that use only one index of the inner dimension, only the products that read it contribute, and they must produce each of the $mn$ entries of $AB$.

\begin{proof}
With the single-entry blocks of Appendix~\ref{app:prefixsafe}, write $W_r$ for the $m\times n$ matrix of the weights $w_{r,ij}$, so that~(\ref{eq:bilinear}) reads $AB=\sum_r f_r(A)\,g_r(B)\,W_r$.
Let $\mathcal{Q}_l$ be the set of products that read column $l$ of $A$; under the first hypothesis the sets $\mathcal{Q}_1,\dots,\mathcal{Q}_k$ are disjoint.
For $A=E_{il}$ and $B=E_{lj}$, $AB=E_{ij}$, while $f_r(A)\,g_r(B)=u_{r,il}\,v_{r,lj}$ is zero unless $r\in\mathcal{Q}_l$.
Each of the $mn$ matrices $E_{ij}$ is therefore a linear combination of the $W_r$ with $r\in\mathcal{Q}_l$, and since the $E_{ij}$ are linearly independent, $\mathcal{Q}_l$ has at least $mn$ members.
Summing over the $k$ disjoint sets gives $R\ge mkn$.
The sets of products that read row $l$ of $B$ give the second statement in the same way.
\end{proof}

An algorithm with fewer than $mkn$ block multiplications therefore has a product that reads two block columns of $A$, and one that reads two block rows of $B$: some of its block sums add entries from different indices of the inner dimension.
Under $c$, those indices carry scales: an output entry's scale is its row's scale times its column's scale, shared by the $g$ indices of each group, and classical int8 applies it after multiplying the group's codes.
Suppose that each block column of a call lay in a group of its own, as in the control of Section~\ref{sec:certificate}, whose groups of 32 indices give each block column of a call spanning 128 indices its own scale.
The call forms its products from the codes, the scales of its block columns vary independently of one another, and to reproduce classical int8 it would have to apply these scales after its products.
In the lemma below the products are fixed, while the coefficients that combine them may depend on the scales in any way; applied with single-entry blocks, it shows that for arbitrary codes and scales this takes the classical count.

\begin{lemma}\label{lem:scales}
Let $M_1,\dots,M_R$ be bilinear forms in $(A,B)\in\mathbb{R}^{m\times k}\times\mathbb{R}^{k\times n}$, such as the products of an algorithm of the form~(\ref{eq:bilinear}).
Suppose that for every diagonal matrix $\Lambda$ with positive diagonal entries there are real numbers $\gamma_{ij,r}(\Lambda)$ such that $(A\Lambda B)_{ij}=\sum_{r=1}^{R}\gamma_{ij,r}(\Lambda)\,M_r(A,B)$ for all $A$, $B$, $i$ and $j$.
Then $R\ge mkn$.
\end{lemma}

\begin{proof}
Both $\Lambda=I$ and $\Lambda=I+E_{ll}$ have positive diagonals, and $(A(I+E_{ll})B)_{ij}-(AB)_{ij}=A_{il}B_{lj}$, so each of the $mkn$ forms $A_{il}B_{lj}$ is a linear combination of $M_1,\dots,M_R$.
These forms are distinct monomials and hence linearly independent, so $R\ge mkn$.
\end{proof}

Condition~(iii) of Proposition~\ref{prop:certificate} keeps each call within one group, where each output entry's scale is constant along the call's inner dimension: the call's exact integer product can then join the group's integer product before the group is rescaled, as in classical int8.

\subsection{Proof of Proposition~\ref{prop:certificate}}\label{app:certificate}

Theorems~\ref{thm:prefixsafe} and~\ref{thm:onecolumn} leave a fast realization no choice but to mix rows and, within a call, indices of the inner dimension.
Proposition~\ref{prop:certificate} shows that on bounded integer codes it can still compute the output bits of classical int8, under three conditions: (i)~$L_Ab_A\le127$ and $L_Bb_B\le127$; (ii)~$\max\{L_Wh(L_Ab_A)(L_Bb_B),\,gb_Ab_B\}<2^{31}$ for int32 accumulation, or $<2^{24}$ for exact integers in fp32; and (iii)~each call lies within one group along the inner dimension.

\begin{proof}
The realization quantizes, rescales and accumulates as the classical int8 operator at $c$ does: it converts each group's integer product into the precision that $c$ fixes, as classical int8 converts the same integer, and rescales and sums the results in the order that $c$ fixes.
Computing the integer product tile by tile means that the tiles partition the inner dimension, and by condition~(iii) each call lies within one group, so the calls of a group add up to that group's integer product, the exact product of the codes over its $g$ indices of the inner dimension.
The realization therefore differs from classical int8 only in how it computes these integers, and it suffices to show that they come out exact.

Consider first one call, on tiles of codes of magnitude at most $b_A$ in $A$ and $b_B$ in $B$, split into blocks of inner length $h$, and write $\alpha=L_Ab_A$ and $\beta=L_Bb_B$.
Entry by entry, for every product $r$, every output block $(i,j)$ and every set $S$ of products,
\[
\begin{gathered}
\Bigl|\sum_{i,l}u_{r,il}A_{il}\Bigr|\le\alpha\le127,\qquad
\Bigl|\sum_{l,j}v_{r,lj}B_{lj}\Bigr|\le\beta\le127,\\
|M_r|\le h\,\alpha\beta,\qquad
\Bigl|\sum_{r\in S}w_{r,ij}M_r\Bigr|\le L_W\,h\,\alpha\beta ,
\end{gathered}
\]
by condition~(i) and the definitions of $L_A$, $L_B$ and $L_W$.
Each bound holds for the sum of the absolute values of the terms it covers (the codes in a block sum, the $h$ products of two block-sum entries in an entry of $M_r$, the weighted products in an output block), so it also holds for every partial sum formed on the way, in any order.
The block sums are therefore exact int8 values, and since integer weights give $L_W\ge1$, every product and every partial sum formed in the call's multiplications and combinations has magnitude at most $L_Wh\alpha\beta$.
Condition~(ii) keeps this below $2^{31}$, so that int32 accumulation is exact, or, for fp32, below $2^{24}$, where fp32 represents every integer and computes exactly any sum or product whose exact value is such an integer.
The call therefore returns its integer product exactly, since its coefficients satisfy~(\ref{eq:bilinear}).

Consider next the calls of one group.
Each returns its exact integer product before that is added to the group's running sum, and for a given output entry the calls cover disjoint sets of the group's $g$ inner indices, so every value of the running sum is a sum of at most $g$ terms $A_{tp}B_{pj}$, with $t$ a row of $A$ and $p$ an index of the inner dimension, of magnitude at most $gb_Ab_B$.
Condition~(ii) keeps this within the same exact range, and every group's integer product therefore comes out exact.

In fp32 the integers are held as floating-point numbers, and a product such as $(-1)\cdot0$ gives $-0$ where the integer is $0$; converted as classical int8 converts the integer, it enters the rescaling as $+0$.
For the operator we run, any difference between the two representations of a zero group result disappears in any case in the sum over groups, which starts from $+0$, rounds each addition separately and, under IEEE round-to-nearest, never holds $-0$.
The realization's output is therefore that of the classical int8 operator at $c$, bit for bit.
\end{proof}

\section{Criteria, models and controls for the sign variants}\label{app:magnitude}

\subsection{What the criteria read, and what the operands decide}\label{app:criteria}

Each stability criterion analyzed in Section~\ref{sec:magnitude} is constant across all sign variants of two-level Strassen.
The criteria are computed from the coefficients of~(\ref{eq:bilinear}); write $u_r$, $v_r$ and $w_r$ for the coefficients $u_{r,il}$, $v_{r,lj}$ and $w_{r,ij}$ of product $r$.
The measure of expected error of \citet{xie2026fast} is $\Phi=\sum_r\|u_r\|_2^2\,\|v_r\|_2^2\,\|w_r\|_2^2$; the nonzero coefficients are counted over $U$, $V$ and $W$ together, and the largest coefficient magnitude is taken over all three.
With $\alpha_r$ and $\beta_r$ the numbers of nonzero entries of $u_r$ and $v_r$, $a_r=\|u_r\|_1$, $b_r=\|v_r\|_1$, and $\gamma_{ij}$ the number of products with $w_{r,ij}\neq0$, the prefactor and the stability factor of \citet[Definitions~4.1 and~4.2]{ballard2015improving} are
\[
Q=\max_{i,j}\Bigl(\gamma_{ij}+\max_{r:\,w_{r,ij}\neq0}(\alpha_r+\beta_r)\Bigr),\qquad E=\max_{i,j}\sum_r a_r\,b_r\,|w_{r,ij}|.
\]
Table~\ref{tab:criteria} lists four of the \NIclrActOneQOneFourOrbitN{} sign variants, chosen across the range of their perplexity ratios on Qwen2.5-14B-Instruct.
Each criterion takes one value on all four, as on all \NIclrActOneQOneFourOrbitN{}, while the ratio runs from \NIclrActOneQOneFourOrbitMin{} to \NIclrActOneQOneFourOrbitMax.
The parameters $Q$ and $E$ of \citet{ballard2015improving} read the same coefficient magnitudes and nonzero pattern, so they too take one value on all \NIclrActOneQOneFourOrbitN.

\begin{table}[h]
\caption{The stability criteria take one value on every sign variant of two-level Strassen; the measured perplexity ratio does not. Four of the \NIclrActOneQOneFourOrbitN{} sign variants on Qwen2.5-14B-Instruct under the FP8 schedule of Section~\ref{sec:support}; the ratio is the WikiText-2 perplexity over that of the bf16 model, and the median row is the upper of the two middle ratios. $\Phi$ is the measure of expected error that \citet{xie2026fast} compute from the coefficients; the nonzero coefficients are counted over $U$, $V$ and $W$.}
\label{tab:criteria}
\centering
\small
\begin{tabular}{@{}l c c c r@{}}
\toprule
sign variant & $\Phi$ & nonzero coefficients & largest $|$coefficient$|$ & perplexity ratio \\
\midrule
smallest ratio & \NIclrActOneOrbitFourFourFourPhi & \NIclrActOneOrbitFourFourFourNnz & \NIclrActOneOrbitFourFourFourMaxAbs & \NIclrActOneQOneFourOrbitMin \\
median ratio & \NIclrActOneOrbitFourFourFourPhi & \NIclrActOneOrbitFourFourFourNnz & \NIclrActOneOrbitFourFourFourMaxAbs & \NIclrActOneQOneFourOrbitMedian \\
Strassen's original signs & \NIclrActOneOrbitFourFourFourPhi & \NIclrActOneOrbitFourFourFourNnz & \NIclrActOneOrbitFourFourFourMaxAbs & \NIclrActOneQOneFourOrbitIdentity \\
largest ratio & \NIclrActOneOrbitFourFourFourPhi & \NIclrActOneOrbitFourFourFourNnz & \NIclrActOneOrbitFourFourFourMaxAbs & \NIclrActOneQOneFourOrbitMax \\
\bottomrule
\end{tabular}
\end{table}

Error differences among sign variants depend on the operand distribution.
In the models, the schedule of Section~\ref{sec:support} rounds the block sums of a layer's input to FP8 (e4m3) with one scale for every 128 consecutive entries of a row, and those of its weight, first stored in bf16, with one scale for every $128\times128$ block; DeepGEMM multiplies them, returns each product in bf16, and the products are combined in fp32.
On synthetic $128\times128$ matrices we use a simplified form, one scale per row of each block sum of $A$ and per column of each block sum of $B$, with the products computed in fp32 on a CPU.
For every one of the \NIclrActOneQOneFourOrbitN{} sign variants we draw 200 pairs of operands with independent entries from each of two distributions and average the variant's relative error in the Frobenius norm separately over the even-numbered and the odd-numbered pairs.
With entries uniform on $[-1,1]$ the variants' average errors span a factor of about \NIclrActOneOrbitFourFourFourSpreadPmOne, and the two halves' averages correlate across the variants at only \NIclrActOneOrbitFourFourFourSplitHalfPmOne, consistent with the symmetry argument of Section~\ref{sec:magnitude}.
With entries uniform on $[0,1]$ they span a factor of \NIclrActOneOrbitFourFourFourSpreadZeroOne, and the halves correlate at \NIclrActOneOrbitFourFourFourSplitHalfZeroOne: the variants differ, and differ reproducibly.
On such matrices \citet{malik2019randomization} also see random signs lower the error, while on Gaussian ones their variants perform about the same.
\citet[Section~5.1]{ballard2015improving} report the same contrast between whole algorithms: among three algorithms of one shape that each minimize one of their criteria, the errors differ clearly on operands uniform on $[0,1]$ and are practically indistinguishable on operands uniform on $[-1,1]$.

\subsection{Seven models, and an outlier test}\label{app:models}

Inside a model the operands are its own activations and weights, so the spread of the sign variants depends on the model.
Table~\ref{tab:sevenmodels} repeats the measurement of Section~\ref{sec:magnitude} on seven models \citep{qwen25,qwen3,llama3,olmo2024olmo,young2024yi}.
The spread runs from \NIclrActOneSevenQSevenbSpread$\times$ on Qwen2.5-7B-Instruct to \NIclrActOneQOneFourOrbitSpread$\times$ on Qwen2.5-14B-Instruct, while the deployed FP8 kernel's ratio ranges only from \NIclrActOneSevenYiSixbClassicalFpEightNone{} to \NIclrActOneSevenQOneFourbClassicalFpEightNone.

\begin{table}[h]
\caption{The spread of the sign variants of two-level Strassen differs by model, while the deployed FP8 kernel stays close to bf16 on all seven. Perplexity ratios under the FP8 schedule of Section~\ref{sec:support}, each the perplexity of eight 2048-token chunks of WikiText-2 over the bf16 model's on the same chunks. Runs whose producer checksum changed during execution are omitted on Qwen2.5-7B-Instruct and Qwen3-8B. ``Variants'' counts the sign variants with a valid measurement; the median is the middle value, the upper of the two middle values when the count is even; the spread is the largest ratio over the smallest; ``$>2$'' counts the variants whose ratio exceeds 2; the last column is the deployed FP8 kernel's ratio.}
\label{tab:sevenmodels}
\centering
\small
\setlength{\tabcolsep}{4.5pt}
\begin{tabular}{@{}l r r r r r r r@{}}
\toprule
model & variants & smallest & median & largest & spread & $>2$ & FP8 kernel \\
\midrule
Qwen2.5-3B & \NIclrActOneSevenQThreebN & \NIclrActOneSevenQThreebMin & \NIclrActOneSevenQThreebMedian & \NIclrActOneSevenQThreebMax & \NIclrActOneSevenQThreebSpread & \NIclrActOneSevenQThreebNAboveTwo & \NIclrActOneSevenQThreebClassicalFpEightNone \\
Qwen2.5-7B-Instruct & \NIclrActOneSevenQSevenbN & \NIclrActOneSevenQSevenbMin & \NIclrActOneSevenQSevenbMedian & \NIclrActOneSevenQSevenbMax & \NIclrActOneSevenQSevenbSpread & \NIclrActOneSevenQSevenbNAboveTwo & \NIclrActOneSevenQSevenbClassicalFpEightNone \\
Qwen2.5-14B-Instruct & \NIclrActOneSevenQOneFourbN & \NIclrActOneQOneFourOrbitMin & \NIclrActOneQOneFourOrbitMedian & \NIclrActOneQOneFourOrbitMax & \NIclrActOneQOneFourOrbitSpread & \NIclrActOneSevenQOneFourbNAboveTwo & \NIclrActOneSevenQOneFourbClassicalFpEightNone \\
Qwen3-8B & \NIclrActOneSevenQThreexEightbN & \NIclrActOneSevenQThreexEightbMin & \NIclrActOneSevenQThreexEightbMedian & \NIclrActOneSevenQThreexEightbMax & \NIclrActOneSevenQThreexEightbSpread & \NIclrActOneSevenQThreexEightbNAboveTwo & \NIclrActOneSevenQThreexEightbClassicalFpEightNone \\
Llama-3.1-8B-Instruct & \NIclrActOneSevenLlEightbN & \NIclrActOneSevenLlEightbMin & \NIclrActOneSevenLlEightbMedian & \NIclrActOneSevenLlEightbMax & \NIclrActOneSevenLlEightbSpread & \NIclrActOneSevenLlEightbNAboveTwo & \NIclrActOneSevenLlEightbClassicalFpEightNone \\
OLMo-2-7B-Instruct & \NIclrActOneSevenOlSevenbN & \NIclrActOneSevenOlSevenbMin & \NIclrActOneSevenOlSevenbMedian & \NIclrActOneSevenOlSevenbMax & \NIclrActOneSevenOlSevenbSpread & \NIclrActOneSevenOlSevenbNAboveTwo & \NIclrActOneSevenOlSevenbClassicalFpEightNone \\
Yi-6B & \NIclrActOneSevenYiSixbN & \NIclrActOneSevenYiSixbMin & \NIclrActOneSevenYiSixbMedian & \NIclrActOneSevenYiSixbMax & \NIclrActOneSevenYiSixbSpread & \NIclrActOneSevenYiSixbNAboveTwo & \NIclrActOneSevenYiSixbClassicalFpEightNone \\
\bottomrule
\end{tabular}
\end{table}

Since the spread varies this much, one may ask whether a simple property of a model's activations predicts it.
A separate test, under the same schedule and on three further models, checked one such rule: that a model with strong enough activation outliers has a catastrophic sign variant, one whose perplexity exceeds ten times that of its own bf16 model.
Its outlier ratio is the largest ratio of the Euclidean norm of a token row to that of the median row, over the inputs of every linear layer in the body of the bf16 model (all but the output layer) and over four 2048-token chunks of WikiText-2.
Mistral-7B-v0.1 \citep{jiang2023mistral} and Llama-2-7b-hf \citep{touvron2023llama} were chosen before their variants were run, as models whose strong activation outliers should give each a catastrophic variant if the rule held.
Their outlier ratios, \NFastmatmulWThreeOhNineMistralSevenOutlierRatio{} and \NFastmatmulWThreeOhNineLlamaTwoSevenOutlierRatio, exceed the \NFastmatmulWThreeOhNineQThirtyTwoOutlierRatio{} of Qwen2.5-32B-Instruct, yet neither has a catastrophic variant (\NFastmatmulWThreeOhNineMistralSevenCatastrophicCount{} of \NFastmatmulWThreeOhNineMistralSevenGaugeCount{} and \NFastmatmulWThreeOhNineLlamaTwoSevenCatastrophicCount{} of \NFastmatmulWThreeOhNineLlamaTwoSevenGaugeCount), while Qwen2.5-32B-Instruct has \NFastmatmulWThreeOhNineQThirtyTwoCatastrophicCount{} of \NFastmatmulWThreeOhNineQThirtyTwoGaugeCount.
So no threshold on the outlier ratio separates these three models into those above it, with a catastrophic variant, and those below it, without one\ifhmode\unskip\fi.

\rev{\subsection{What sets the variants apart on one model}\label{app:mechanism}

On Qwen2.5-14B-Instruct, a catastrophic sign variant differs from a benign one mainly in the error it leaves on a few token rows.
Two-level Strassen's block sums combine each token row with three others, and the first token of each window carries massive activations, so the FP8 rounding error of the encoded weights, multiplied by those activations, reaches the rows combined with it.
Injecting only the difference between a catastrophic and a benign variant on one of those rows, at one linear layer, with every other linear layer on the classical FP8 kernel, raises the perplexity ratio to \NIclrRebMechQFourteenRowNat.
Recomputed with the activations' FP8 rounding switched off, that difference still gives \NIclrRebMechQFourteenRowAoff; with the weights' switched off, it gives \NIclrRebMechQFourteenRowWoff.
Every later token attends to these rows, and the sign variant, by changing the encoded weights, changes their error.
Two interventions on this coupling remove the catastrophic variants.
Moving the at most four rows whose norm exceeds twenty times the median to positions whose partners in the block sums are padding removes all \NIclrRebMechQFourteenIsolateRemoved{} of them.
Rescaling each token row by a power of two before the rows are summed, as the repaired realizations of Section~\ref{sec:support} do, narrows the spread of the \NIclrActOneQOneFourOrbitN{} sign variants from \NIclrActOneQOneFourOrbitSpread{} to \NIclrRebMechQFourteenEquilSpread, with \NIclrRebMechQFourteenEquilNcat{} catastrophic.

The same account does not carry to the other models we examined.
On Qwen2.5-3B, the only other model of Table~\ref{tab:sevenmodels} with catastrophic variants, the same rescaling leaves \NIclrRebMechQThreeEquilNcat{} of its \NIclrRebMechQThreeNativeNcat{} catastrophic.
On Qwen3-8B, the variant with the largest perplexity ratio, \NIclrRebMechQThirtyEightFullPpl, needs the error of one early layer, but run at that layer alone it raises the ratio only to \NIclrRebMechQThirtyEightSoloPpl, and neither the first token's row nor its partner carries that effect.
Which sign variant is safe thus cannot be read off in advance: the coefficient criteria take one value on all of them (Section~\ref{sec:magnitude}), a model-level statistic such as the outlier ratio does not even predict which models have a catastrophic variant (Appendix~\ref{app:models}), and a rescaling that tames one model's variants leaves another's catastrophic.
Only certifying the operator makes every variant safe: every certified variant computes the same row-local operator, bit for bit, on every input (Section~\ref{sec:construction}).}

\subsection{A rotation, other summation orders and repeated runs}\label{app:order}

Changing the operands without changing their product moves the spread, and not in one direction.
Methods for removing outliers, such as QuaRot and SpinQuant \citep{ashkboos-quarot-2024,liu-spinquant-2024}, rotate the operands, $A\mapsto AH$ and $B\mapsto H^{\top}B$ with $H$ orthogonal, which leaves $AB$ unchanged.
With the simplified rounding of Appendix~\ref{app:criteria}, on $128\times128$ products with one level of Strassen and 80 draws of the operands, a randomized Hadamard rotation narrows the spread of the average errors (the largest over the smallest) over its 32 stored sign patterns, which repeat some sign variants of Section~\ref{sec:signvariants}, from \NIclrActOneRotationUniformDcPlain{} to \NIclrActOneRotationUniformDcRotated, when every column of $A$ carries the same offset, and widens it, from \NIclrActOneRotationOutlierPlain{} to \NIclrActOneRotationOutlierRotated, when the columns' offsets have random signs and three percent of them are forty times larger than the rest.

With the coefficients and the input fixed, the order in which the \NIclrActFourStrassenTwoMultsFast{} products are added is one more choice that leaves the product exact.
We ran sixteen realizations of two-level Strassen on Qwen2.5-14B-Instruct, eight sign variants and eight that also permute blocks, each under its original order and seven sampled reorderings of its products; the sign variants of Section~\ref{sec:magnitude} all keep the original order, so their spread does not come from reordering.
A reordering permutes the products throughout, in how they are prepared and launched as well as in how every output block adds them.
On one of them the order alone moves the perplexity ratio from \NIclrActOneCFourSevenSevenMin{} to \NIclrActOneCFourSevenSevenMax, a \NIclrActOneCFourSevenSevenSpan$\times$ span.
On the classical path, reordering the deployed FP8 kernel's scaling groups along the inner dimension, together with their scales, moves its perplexity ratio across eight orders by a factor of only \NIclrActOneCThreeNineTwoSpread.

Within the tested container instance, none of these differences is run-to-run variation.
Eight fresh processes on eight GPUs of one container instance return, for each of the two orders at that realization's extremes, exactly the same perplexity ratio (relative difference \NIclrActOneCFourEightZeroRelSpread).
On Qwen2.5-14B-Instruct, running one of four 1,024-token inputs twice, or again after a different input, leaves every logit unchanged under the unrepaired realization, repaired~\#2 and both row-local controls; for repaired~\#1, Appendix~\ref{app:scoring} reports how its chosen answers reproduce.

The rescore of Section~\ref{sec:magnitude} that replaces the later text with filler uses 280 windows of 1,024 tokens, alternating WikiText-2 test and C4 validation text, with four scored positions per window and two filler texts per position; its interval resamples the windows.

\section{Details of the tests for prefix invariance}\label{app:support}

\subsection{What the two inputs of a comparison share}\label{app:scoring}

The comparisons of Section~\ref{sec:support} hold everything but the replaced text fixed.
Both inputs run in one process with batch size one, through one realization with the same weights and the same order of its products, and they have the same length, so in the fast realizations their token rows are padded with zeros (to a multiple of 16) and split into blocks identically.
An answer is scored by summing the log probabilities of its tokens, each read at the position before it; on the original input a single forward pass over the question and the answer gives all of them, and each candidate answer is scored on its own.
For the answer token at position $p$, the second input keeps positions $1$ to $p-1$, the token's allowed prefix, and replaces position $p$ and every later position with as many new tokens, so the original token is still the one scored, from position $p-1$; an answer's score under replacement therefore adds up scores from separate forward passes, one for each of its tokens.
A realization that keeps prefix invariance gives every answer token the same score from both inputs, and so every item the same choice.

The two replacements differ only in where the new tokens come from.
The filler is text from the validation split of C4 \citep{raffel2020exploring}: each item gets one window of it, shared by all its answers and positions, and each replacement takes that window's first tokens.
The greedy continuation is generated by the bf16 model from the allowed prefix, one most likely token at a time, for exactly the replaced length; each is generated once for its prefix and length and reused for every realization.
OpenBookQA and ARC-Easy each contribute the same first \NIclrActTwoQOneFourObqaItems{} test items in every condition; this subset size follows the compute budget for the separate forward pass required at every scored answer-token position.
On the \NIclrActTwoQOneFourObqaPlausibleWThreeNScored{} answers of OpenBookQA the continuation built at an answer's first token begins with that token \NIclrActTwoQOneFourObqaPlausibleWThreeNEqual{} times.
The answers changed in Table~\ref{tab:support} count the items whose chosen answer differs between the two inputs, including a change from one wrong answer to another, with 95\% intervals from resampling the items.

Repeating the comparison reproduces its outcome.
For the unrepaired realization, repaired~\#2 and both row-local controls, repeating an input within one process gives identical logits (Appendix~\ref{app:order}).
Across two runs of the whole OpenBookQA comparison in different container instances, \rev{some log probabilities differed, but} every chosen answer of repaired~\#1 and of both row-local controls was reproduced.

Table~\ref{tab:tokenchanges} resolves the same comparisons at every scored answer-token position, including positions from items whose chosen answer does not change.

\begin{table}[h]
\caption{Token-level changes on Qwen2.5-14B-Instruct under the replacements of Table~\ref{tab:support}. At each scored candidate position, $\Delta$ is the original target token's log probability on the original input minus its log probability after replacement; all \NIclrTokObqaCount{} OpenBookQA or \NIclrTokArcCount{} ARC-Easy positions are included per realization and replacement. The last three columns summarize $|\Delta|$ in nats; the 90th percentile is the empirical inverse cumulative distribution function. Both row-local controls have \NIclrTokControlChanged{} changed target scores and identical full logit rows under all three comparisons.}
\label{tab:tokenchanges}
\centering
\small
\setlength{\tabcolsep}{4pt}
\begin{tabular}{@{}l l r r r r@{}}
\toprule
condition & realization & changed (\%) & mean & median & 90th pct. \\
\midrule
OpenBookQA, continuation & repaired~\#1 & \NIclrTokObqaGreedyRepairOnePct & \NIclrTokObqaGreedyRepairOneMean & \NIclrTokObqaGreedyRepairOneMedian & \NIclrTokObqaGreedyRepairOnePninety \\
 & repaired~\#2 & \NIclrTokObqaGreedyRepairTwoPct & \NIclrTokObqaGreedyRepairTwoMean & \NIclrTokObqaGreedyRepairTwoMedian & \NIclrTokObqaGreedyRepairTwoPninety \\
\addlinespace[3pt]
OpenBookQA, filler & repaired~\#1 & \NIclrTokObqaFillerRepairOnePct & \NIclrTokObqaFillerRepairOneMean & \NIclrTokObqaFillerRepairOneMedian & \NIclrTokObqaFillerRepairOnePninety \\
 & repaired~\#2 & \NIclrTokObqaFillerRepairTwoPct & \NIclrTokObqaFillerRepairTwoMean & \NIclrTokObqaFillerRepairTwoMedian & \NIclrTokObqaFillerRepairTwoPninety \\
\addlinespace[3pt]
ARC-Easy, filler & repaired~\#1 & \NIclrTokArcFillerRepairOnePct & \NIclrTokArcFillerRepairOneMean & \NIclrTokArcFillerRepairOneMedian & \NIclrTokArcFillerRepairOnePninety \\
 & repaired~\#2 & \NIclrTokArcFillerRepairTwoPct & \NIclrTokArcFillerRepairTwoMean & \NIclrTokArcFillerRepairTwoMedian & \NIclrTokArcFillerRepairTwoPninety \\
\bottomrule
\end{tabular}
\end{table}

\paragraph{The repair rules.}
For the repairs, a token row $x_t$ uses
\[
d_t=2^{\left\lceil \log_2\!\left(\max\{\operatorname{RMS}(x_t),10^{-3}\}\right)\right\rceil},
\]
with the root mean square (RMS) taken over the full row: we divide the original row by $d_t$ before the fast mixing and multiply its decoded output row by $d_t$ afterward.
For each weight-block-sum entry, we take its bf16 prequantization value and its original DeepGEMM block scale, consider the two e4m3 values at that scale whose decoded values bracket it, and initialize at the nearer one.
For the $R$ product terms at the same position, let $e_r$ be the decoded candidate minus the bf16 prequantization block sum and set
\[
H_{rs}=\langle w_r,w_s\rangle\,\langle u_r,u_s\rangle,
\]
where $u_r$ and $w_r$ are the coefficient vectors of product $r$ in~(\ref{eq:bilinear}).
We sweep $r=1,\ldots,R$ in order; a candidate change $\delta$ is accepted when its fp32-computed change in $e^\top He$, $2\delta(He)_r+\delta^2H_{rr}$, is negative.
We stop after a sweep with no change and make at most 12 sweeps.
This weight-only search uses no activation or text calibration examples.
Repaired~\#1 uses one Strassen level and one classical level; repaired~\#2 uses two levels of the $\Phi$-best seven-product algorithm of Section~\ref{sec:support}, with the coefficient arrays distributed with the implementation.

\subsection{The \NIclrActTwoExemplarsNMoved{} changed answers}\label{app:exemplars}

Table~\ref{tab:exemplars} lists every OpenBookQA item whose chosen answer repaired~\#1 changes when the text after the allowed prefix is replaced by the greedy continuation, and on all \NIclrActTwoExemplarsControlsQuiet{} of them both row-local controls keep their answer.

\begin{table}[h]
\caption{The \NIclrActTwoExemplarsNMoved{} OpenBookQA items on which repaired~\#1 changes its chosen answer when the text after the allowed prefix is replaced by the bf16 model's greedy continuation, on Qwen2.5-14B-Instruct, grouped by how the change affects correctness. Each entry gives the item's number in the test split (counted from 0) and its question, then the answer chosen on the original input $\rightarrow$ the answer chosen under replacement; \checkmark{} marks a chosen answer that is correct, and the correct answer is given on the right when neither is.}
\label{tab:exemplars}
\centering
\small
\setlength{\tabcolsep}{4pt}
\setlength{\exmidw}{\dimexpr\textwidth-\exitemw-\excorrw-4\tabcolsep\relax}
\begin{tabular}{@{}>{\raggedleft\arraybackslash\leavevmode\color{exmuted}}p{\exitemw} >{\raggedright\arraybackslash}p{\exmidw} >{\raggedleft\arraybackslash\leavevmode\color{exmuted}}p{\excorrw}@{}}
\toprule
\multicolumn{2}{@{}l}{\exgroup{Wrong before and after the change} (\NIclrActTwoExemplarsNeitherIsGold{})} & \textit{correct answer} \\[3pt]
37 & \exspan{What is used for sensing visual things?} \\
 & \exans{tibia} $\rightarrow$ \exans{nerves} & \textit{cornea} \\\addlinespace[5pt]
41 & \exspan{What is different about birth in humans and chickens?} \\
 & \exans{Mother} $\rightarrow$ \exans{Fertilization} & \textit{the hard shell} \\\addlinespace[5pt]
50 & \exspan{Some berries may be eaten by} \\
 & \exans{a bear or wolf} $\rightarrow$ \exans{a bear or lion} & \textit{a bear or person} \\\addlinespace[5pt]
53 & \exspan{To improve health, what is a good strategy?} \\
 & \exans{restaurant food} $\rightarrow$ \exans{business trip} & \textit{a spa trip} \\\addlinespace[5pt]
54 & \exspan{A girl and her mom have the same} \\
 & \exans{date of birth} $\rightarrow$ \exans{shirt} & \textit{number of toenails} \\\addlinespace[5pt]
84 & \exspan{Ocean water contains} \\
 & \exans{copious amounts of seltzer} $\rightarrow$ \exans{scant amounts of sodium chloride} & \textit{copious amounts of the combination of Na and Cl} \\\addlinespace[5pt]
93 & \exspan{When looking for good soil for plants, typically what is optimal?} \\
 & \exans{compact and hard} $\rightarrow$ \exans{dry and sandy} & \textit{malleable and nutritious} \\\addlinespace[5pt]
103 & \exspan{A farmer harvests seeds from some plants, such as tomatoes, in order to plant them later on. These seeds, once planted} \\
 & \exans{have their own sunlight} $\rightarrow$ \exans{have their own dirt} & \textit{contain their necessary nutrition} \\\addlinespace[5pt]
227 & \exspan{How is electricity produced from the ocean?} \\
 & \exans{chemical reactions produced from the salt in the water} $\rightarrow$ \exans{decaying organic material from sealife} & \textit{energy is accessed underwater from tides} \\
\midrule
\multicolumn{2}{@{}l}{\exgroup{From correct to wrong} (\NIclrActTwoExemplarsNaturalIsGold{})} &  \\[3pt]
119 & \exspan{Endangered pandas are sometimes} \\
 & \exspan{\exans{confined to enclosures to be viewed by the public}\exok{} $\rightarrow$ \exans{accidentally dropped into volcanoes}} \\\addlinespace[5pt]
132 & \exspan{A desert environment is} \\
 & \exspan{\exans{arid, parched, and sun-baked}\exok{} $\rightarrow$ \exans{lush, green, and tropical}} \\\addlinespace[5pt]
205 & \exspan{There are various creatures that live in forests, such as} \\
 & \exspan{\exans{whitetails}\exok{} $\rightarrow$ \exans{giant fish}} \\
\midrule
\multicolumn{2}{@{}l}{\exgroup{From wrong to correct} (\NIclrActTwoExemplarsBlindIsGold{})} &  \\[3pt]
136 & \exspan{What is the formula of the substance which best helps plants grow} \\
 & \exspan{\exans{CO2} $\rightarrow$ \exans{H2O}\exok{}} \\\addlinespace[5pt]
204 & \exspan{Through DNA, a rabbit will have long ears if} \\
 & \exspan{\exans{parents were also rabbits} $\rightarrow$ \exans{genetic contributors had long ears}\exok{}} \\
\bottomrule
\end{tabular}
\end{table}

\rev{\paragraph{Whether the later text biases the score.}
On the original input every answer token is followed by the rest of its own answer, and the replacement removes that text.
For each scored token of the three comparisons of Table~\ref{tab:support}, we take its log probability on the original input minus its log probability after replacement.
Under both repaired realizations the mean of this shift lies between $\NIclrRebEFiveADmeanMin$ and $\NIclrRebEFiveADmeanMax$ nats per token with every 95\% interval (item bootstrap) containing zero, and of the tokens that change, \NIclrRebEFiveAPosShareMin--\NIclrRebEFiveAPosShareMax\% move up, about as many as move down.
The correct answer's gain from its own text, beyond the gain of the other answers from theirs, and the change in accuracy also have intervals containing zero in each comparison.
A direct test separates presence from absence: for every answer and scored position, the text after the allowed prefix is replaced by the filler or by one answer's own text followed by the filler, so that each answer's text is present in some inputs and absent in others.
Averaged over all scored tokens, the log probability of the answer whose text is present changes by $\NIclrRebEFivePresTokMin$ to $\NIclrRebEFivePresTokMax$ nats per token relative to the other answers, and the choice moves to the present answer about as often as to a relabeled one; every 95\% interval contains zero, and both row-local controls change no logit at all.
We thus find no pull toward the answer whose text is present: what breaks is that the score depends on that text at all.}

\subsection{How a later row reaches an earlier output}\label{app:rowmap}

Section~\ref{sec:reach} traces the dependence to the block sums of the linear layers, and a single multiplication shows it without a model.
We multiply a Gaussian input of 40 rows by a Gaussian $512\times512$ weight matrix, padding the input with zero rows to 48 so that two-level Strassen splits its rows into four blocks of 12.
The block sums are rounded to FP8 as the models' schedule of Appendix~\ref{app:criteria} rounds them, without repairs, and the products are computed in fp32 on the rounded values; a row counts as changed when any of its entries differs at all.
We then replace row 40, the last row before the padding, by a fresh Gaussian row and record which output rows change.
Under two-level Strassen with Strassen's original signs, and under the algorithm of repaired~\#2 without its repairs, output rows 4, 16 and 28 change besides row 40 itself, and no other: the rows at the same offset in the other three blocks, all earlier than row 40.
With the classical product computed in fp32, from the unrounded operands or from operands rounded to FP8 in the same way, only row 40 changes.

\subsection{Other models}\label{app:windows}

The dependence is not particular to multiple-choice scoring: it changes the next-token predictions of Qwen2.5-14B-Instruct, Qwen2.5-3B, Llama-3.1-8B-Instruct, OLMo-2-7B-Instruct and Qwen3-8B.
For each model we take 256 windows of 1,024 tokens, alternating between WikiText-2 and C4, keep each window's first 256 tokens and replace the rest by the corresponding tokens of two other windows, one from each corpus.
Each window and replacement form a pair, and we compare the top-1 predictions at positions 1 to 255, whose predicted tokens lie in the kept quarter.
Table~\ref{tab:reach} gives the percentage of these predictions that change for each model and realization; under the two row-local controls together, none of the \NIclrActTwoCFourSixThreeControlsPositions{} predictions compared on the five models changes (\NIclrActTwoCFourSixThreeControlsFlips{} flips).

\begin{table}[h]
\caption{The dependence on replaced text, per model. Window test: the percentage of next-token predictions in the kept first quarter of a 1,024-token window that change when the rest of the window is replaced. Packed test: two requests in one forward pass, the number of pairs and the percentage of pairs in which the first request's most likely next token changes when the second request is replaced. The packed test ran on three models, under the unrepaired realization and repaired~\#2; both row-local controls change no prediction in either test.}
\label{tab:reach}
\centering
\small
\setlength{\tabcolsep}{5pt}
\begin{tabular}{@{}l r r r r r r@{}}
\toprule
 & \multicolumn{3}{c}{window test (\%)} & \multicolumn{3}{c}{packed test} \\
\cmidrule(lr){2-4}\cmidrule(l){5-7}
model & unrepaired & repaired~\#1 & repaired~\#2 & pairs & unrepaired (\%) & repaired~\#2 (\%) \\
\midrule
Qwen2.5-14B-Instruct & \NIclrActTwoCFourSixThreeQOneFourUnrepairedFlipPct & \NIclrActTwoCFourSixThreeQOneFourRepairedOneFlipPct & \NIclrActTwoCFourSixThreeQOneFourRepairedTwoFlipPct & \NIclrActTwoCFourSixFourQOneFourPairs & \NIclrActTwoCFourSixFourQOneFourUnrepairedGreedyPct & \NIclrActTwoCFourSixFourQOneFourRepairedTwoGreedyPct \\
Qwen2.5-3B & \NIclrActTwoCFourSixThreeQThreeUnrepairedFlipPct & \NIclrActTwoCFourSixThreeQThreeRepairedOneFlipPct & \NIclrActTwoCFourSixThreeQThreeRepairedTwoFlipPct & \NIclrActTwoCFourSixFourQThreePairs & \NIclrActTwoCFourSixFourQThreeUnrepairedGreedyPct & \NIclrActTwoCFourSixFourQThreeRepairedTwoGreedyPct \\
Llama-3.1-8B-Instruct & \NIclrActTwoCFourSixThreeLlEightUnrepairedFlipPct & \NIclrActTwoCFourSixThreeLlEightRepairedOneFlipPct & \NIclrActTwoCFourSixThreeLlEightRepairedTwoFlipPct & -- & -- & -- \\
OLMo-2-7B-Instruct & \NIclrActTwoCFourSixThreeOlSevenUnrepairedFlipPct & \NIclrActTwoCFourSixThreeOlSevenRepairedOneFlipPct & \NIclrActTwoCFourSixThreeOlSevenRepairedTwoFlipPct & \NIclrActTwoCFourSixFourOlSevenPairs & \NIclrActTwoCFourSixFourOlSevenUnrepairedGreedyPct & \NIclrActTwoCFourSixFourOlSevenRepairedTwoGreedyPct \\
Qwen3-8B & \NIclrActTwoCFourSixThreeQThreeEightUnrepairedFlipPct & \NIclrActTwoCFourSixThreeQThreeEightRepairedOneFlipPct & \NIclrActTwoCFourSixThreeQThreeEightRepairedTwoFlipPct & -- & -- & -- \\
\bottomrule
\end{tabular}
\end{table}

\rev{The predictions that flip are a small part of the change.
A kept position counts as changed when the Jensen--Shannon divergence between its two next-token distributions exceeds $10^{-6}$; for the controls, whose two distributions are identical, rounding in computing the divergence stays below this threshold.
Under the repaired realizations the distribution changes at \NIclrRebESixAShareMin\% of the compared positions or more on every model, with a mean divergence of \NIclrRebESixAJsMeanMin--\NIclrRebESixAJsMeanMax{} nats; under the controls it changes at none.}

\subsection{Other requests}\label{app:packed}

Requests that share a forward pass are exposed in the same way.
We split a 1,024-token window into a request A, its first quarter, and a request B, the rest, pack both into one input and build no key--value cache.
The position indices restart at B, which the model's implementation reads as two packed sequences, each with a causal attention mask of its own; in a control run with the bf16 model, replacing A leaves every logit of B unchanged, which confirms that the attention is separated.
We then replace B by the corresponding tokens of two other windows and compare A's next-token distribution at its last position, the token A would emit next.
Over the \NIclrActTwoCFourSixFourPairs{} pairs on Qwen2.5-14B-Instruct, OLMo-2-7B-Instruct and Qwen2.5-3B, that distribution changes on every pair (\NIclrActTwoCFourSixFourFastTvChangedMin{} of \NIclrActTwoCFourSixFourPairs, with a total variation above zero) under the unrepaired realization and under repaired~\#2, and on none under the controls; Table~\ref{tab:reach} gives, per model, how often its most likely token changes.

On Qwen2.5-3B we also record, over 64 such pairs, the key and value projections of A's rows at ten of the model's 36 layers (layers 0 to 5, 8, 16, 24 and 35).
Under the unrepaired realization and repaired~\#2 every recorded projection differs, in every row of A and on every pair, and the differences already appear at the first layer, before any attention has combined positions; under both controls none differs.

\rev{\paragraph{What the first request's output reveals.}
We also measure, on Qwen2.5-14B-Instruct, how much A's output at its last position tells about the request it was packed with.
For each of 16 requests A of 256 tokens, B takes each of \NIclrRebESevenK{} candidate texts of 768 tokens, and we record A's exact logits, its five most likely tokens with their log probabilities rounded to $10^{-3}$ and to $10^{-2}$, and its greedy next token.
We report the entropy of the partition of the candidates by A's observed output, at most 5 bits for \NIclrRebESevenK{} candidates, and how often a lookup table built in one run identifies B from A's output in a second, separate run.
Under either repaired realization, A's exact logits, and even its five most likely tokens with log probabilities rounded to $10^{-2}$, take a different value for each candidate: \NIclrRebESevenRepOneTopFiveBits{} bits, and B is identified in \NIclrRebESevenRepOneTopFiveAcc\% (repaired~\#1) and \NIclrRebESevenRepTwoTopFiveAcc\% (repaired~\#2) of the cases.
A's greedy next token carries \NIclrRebESevenRepOneGreedyBits{} and \NIclrRebESevenRepTwoGreedyBits{} bits (identification \NIclrRebESevenRepOneGreedyAcc\% and \NIclrRebESevenRepTwoGreedyAcc\%, against \NIclrRebESevenBfRowAcc\% by chance).
Under FalconGEMM, run as in Appendix~\ref{app:published}, A's exact logits identify B in \NIclrRebENineIdentFgThreeORowHeldout\% and \NIclrRebENineIdentFgFORowHeldout\% of the cases (its separate and fused variants), its five most likely tokens with log probabilities rounded to $10^{-2}$ in \NIclrRebENineIdentFgThreeTopFiveHeldout\% and \NIclrRebENineIdentFgFTopFiveHeldout\%, and its greedy next token carries \NIclrRebENineIdentFgThreeGreedyBits{} and \NIclrRebENineIdentFgFGreedyBits{} bits.
Under both row-local controls, A's output is the same for every candidate, in the test of the repaired realizations and in that of FalconGEMM, and the controls' exact logits agree between the two tests.}

\rev{\subsection{Generated answers}\label{app:packgen}
The packed test above compares one next token; here the first request generates a whole answer.
Request A is a question from the validation split of NQ-open \citep{kwiatkowski2019natural,lee2019latent} after a fixed five-shot prompt, and request B is 768 tokens of WikiText-2 or C4 text.
A is generated greedily, one token per forward pass over both requests, until a newline, the end-of-sequence token or 16 new tokens, once with each of two texts of B of the same length; attention is separated between the requests and no key--value cache is kept, so until A's two answers first differ, the two generations have the same packed shape at every step and differ only in B's tokens.
Table~\ref{tab:packgen} reports, over \NIclrRebESixBItems{} questions on Qwen2.5-14B-Instruct and on Llama-3.1-8B-Instruct, how often A's answer and its correctness change.
Under both repaired realizations A's first next-token distribution already differs between the two texts of B on \NIclrRebESixBRepOneFirstStep\% of the questions; under repaired~\#2, for example, asked who had a baby at 100 in the Bible, A answers ``Sarah'' next to one text and ``Methuselah'' next to the other.
A generated answer gives the dependence a chance at every token: on Qwen2.5-14B-Instruct under repaired~\#1 and~\#2, the first generated token differs on \NIclrRebESixBRepOneFirstTokenN{} and \NIclrRebESixBRepTwoFirstTokenN{} questions, and the whole answer on \NIclrRebESixBRepOneTextChangedN{} and \NIclrRebESixBRepTwoTextChangedN.
FalconGEMM, run as in Appendix~\ref{app:published}, changes the answer on \NIclrRebENineGenTextChangedMin--\NIclrRebENineGenTextChangedMax\% of the questions (Table~\ref{tab:packgen}).
Under both row-local controls A's two generations are identical token for token in every launch behind Table~\ref{tab:packgen}, and on Qwen2.5-14B-Instruct the control generations of the launch with the repaired realizations and of the one with FalconGEMM agree token for token.

\begin{table}[h]
\caption{\rev{Generated answers to \NIclrRebESixBItems{} NQ-open questions when only the request packed beside them changes: the questions whose answer or exact-match correctness changes (95\% bootstrap intervals over questions), and exact match with each of the two texts of B. FalconGEMM: as in Table~\ref{tab:support}. Under both row-local controls the answer never changes on either model, and the bf16 model's exact match is \NIclrRebESixBCleanEm\% (Qwen2.5-14B-Instruct) and \NIclrRebESixCCleanEm\% (Llama-3.1-8B-Instruct).}}
\label{tab:packgen}
\centering
\small
\setlength{\tabcolsep}{5pt}
\rev{\begin{tabular}{@{}l r@{\hspace{2pt}}l r@{\hspace{2pt}}l r r@{}}
\toprule
 & \multicolumn{2}{c}{answer changed (\%)} & \multicolumn{2}{c}{correctness changed (\%)} & \multicolumn{2}{c}{exact match (\%)} \\
\cmidrule(lr){2-3}\cmidrule(lr){4-5}\cmidrule(l){6-7}
realization & & & & & text 1 & text 2 \\
\midrule
\multicolumn{7}{@{}l}{\textit{Qwen2.5-14B-Instruct}} \\
repaired \#1 & \NIclrRebESixBRepOneTextChanged & {\footnotesize[\NIclrRebESixBRepOneTextChangedLo,\,\NIclrRebESixBRepOneTextChangedHi]} & \NIclrRebESixBRepOneCorrChanged & {\footnotesize[\NIclrRebESixBRepOneCorrChangedLo,\,\NIclrRebESixBRepOneCorrChangedHi]} & \NIclrRebESixBRepOneEmBOne & \NIclrRebESixBRepOneEmBTwo \\
repaired \#2 & \NIclrRebESixBRepTwoTextChanged & {\footnotesize[\NIclrRebESixBRepTwoTextChangedLo,\,\NIclrRebESixBRepTwoTextChangedHi]} & \NIclrRebESixBRepTwoCorrChanged & {\footnotesize[\NIclrRebESixBRepTwoCorrChangedLo,\,\NIclrRebESixBRepTwoCorrChangedHi]} & \NIclrRebESixBRepTwoEmBOne & \NIclrRebESixBRepTwoEmBTwo \\
\midrule
FalconGEMM & \NIclrRebENineGenFgThreeTextChanged & {\footnotesize[\NIclrRebENineGenFgThreeTextChangedLo,\,\NIclrRebENineGenFgThreeTextChangedHi]} & \NIclrRebENineGenFgThreeCorrChanged & {\footnotesize[\NIclrRebENineGenFgThreeCorrChangedLo,\,\NIclrRebENineGenFgThreeCorrChangedHi]} & \NIclrRebENineGenFgThreeEmBOne & \NIclrRebENineGenFgThreeEmBTwo \\
FalconGEMM, fused & \NIclrRebENineGenFgFTextChanged & {\footnotesize[\NIclrRebENineGenFgFTextChangedLo,\,\NIclrRebENineGenFgFTextChangedHi]} & \NIclrRebENineGenFgFCorrChanged & {\footnotesize[\NIclrRebENineGenFgFCorrChangedLo,\,\NIclrRebENineGenFgFCorrChangedHi]} & \NIclrRebENineGenFgFEmBOne & \NIclrRebENineGenFgFEmBTwo \\
\midrule
\multicolumn{7}{@{}l}{\textit{Llama-3.1-8B-Instruct}} \\
repaired \#1 & \NIclrRebESixCRepOneTextChanged & {\footnotesize[\NIclrRebESixCRepOneTextChangedLo,\,\NIclrRebESixCRepOneTextChangedHi]} & \NIclrRebESixCRepOneCorrChanged & {\footnotesize[\NIclrRebESixCRepOneCorrChangedLo,\,\NIclrRebESixCRepOneCorrChangedHi]} & \NIclrRebESixCRepOneEmBOne & \NIclrRebESixCRepOneEmBTwo \\
repaired \#2 & \NIclrRebESixCRepTwoTextChanged & {\footnotesize[\NIclrRebESixCRepTwoTextChangedLo,\,\NIclrRebESixCRepTwoTextChangedHi]} & \NIclrRebESixCRepTwoCorrChanged & {\footnotesize[\NIclrRebESixCRepTwoCorrChangedLo,\,\NIclrRebESixCRepTwoCorrChangedHi]} & \NIclrRebESixCRepTwoEmBOne & \NIclrRebESixCRepTwoEmBTwo \\
\bottomrule
\end{tabular}}
\end{table}}

\rev{\subsection{Existing fast kernels}\label{app:published}
Two existing fast kernels for LLM inference round block sums that mix token rows, FalconGEMM \citep{zhu2026falcongemm} and SubCuber \citep{Jangda_2026}; we run the first in the tests of Sections~\ref{sec:support} and~\ref{sec:reach}, and the second on the kernels alone.

FalconGEMM runs fast matrix multiplication algorithms in BF16 and FP8.
We run its FP8 path as released, with one level of Strassen's algorithm, the activations as the left operand and the weights' block sums formed once per weight matrix.
It splits the token rows of the activations into two halves, and three of its seven block sums of the activations, $A_{11}+A_{22}$, $A_{21}-A_{11}$ and $A_{12}-A_{22}$, combine row $r$ of the first half with row $r$ of the second.
It rounds these sums to FP8, with one scale per row and per 128 channels, before the products, so each rounded row and its scale depend on two tokens: the mechanism of Section~\ref{sec:reach}.
Its weights' block sums carry one scale per $128\times128$ block, so its FP8 format and scale layout are those of our realizations of Section~\ref{sec:support} and of the deployed FP8 kernel, which applies them to operands that mix no rows.
We run both of its variants, which combine the products into the output in a separate kernel or fused into the matrix product.
On an identical input each repeats its logits exactly, so a change is not run-to-run variation.
On Qwen2.5-14B-Instruct every linear layer except the output layer runs FalconGEMM; on OLMo-2-7B-Instruct and Qwen2.5-3B, FalconGEMM does not accept the MLP down projection's input width of 11,008, and that projection runs the deployed FP8 kernel instead.

Beyond the answers of Table~\ref{tab:support}, FalconGEMM changes the target token's log probability under the greedy continuation at \NIclrRebEOneContFgThreeTokpos\% and \NIclrRebEOneContFgFTokpos\% of the scored OpenBookQA positions in its separate and fused variants, against \NIclrTokObqaGreedyRepairOnePct\% and \NIclrTokObqaGreedyRepairTwoPct\% for repaired~\#1 and~\#2 (Table~\ref{tab:tokenchanges}).
In the packed test of Appendix~\ref{app:packed}, on the same pairs, the first request's next-token distribution changes on every pair on all three models, and its most likely token on \NIclrRebEOnePackedFlipMin--\NIclrRebEOnePackedFlipMax\% of them (Table~\ref{tab:published}), against \NIclrActTwoCFourSixFourRepairedGreedyPctMin--\NIclrActTwoCFourSixFourRepairedGreedyPctMax\% under repaired~\#2 (Table~\ref{tab:reach}).
In these runs both row-local controls again change no answer in the tests of Table~\ref{tab:support}, and neither quantity in the packed test.

\begin{table}[h]
\caption{\rev{FalconGEMM in the packed test: per model, the pairs in which the first request's next-token distribution changes, and the percentage in which its most likely token changes, when only the second request is replaced. Both row-local controls change neither.}}
\label{tab:published}
\centering
\small
\setlength{\tabcolsep}{6pt}
\rev{\begin{tabular}{@{}l r rr rr@{}}
\toprule
 & & \multicolumn{2}{c}{separate variant} & \multicolumn{2}{c}{fused variant} \\
\cmidrule(lr){3-4}\cmidrule(l){5-6}
model & pairs & distribution & token (\%) & distribution & token (\%) \\
\midrule
Qwen2.5-14B-Instruct & \NIclrRebEOnePackedQOneFourPairs & \NIclrRebEOnePackedQOneFourFgThreeTvPairs & \NIclrRebEOnePackedQOneFourFgThreeFlip & \NIclrRebEOnePackedQOneFourFgFTvPairs & \NIclrRebEOnePackedQOneFourFgFFlip \\
OLMo-2-7B-Instruct & \NIclrRebEOnePackedOlSevenPairs & \NIclrRebEOnePackedOlSevenFgThreeTvPairs & \NIclrRebEOnePackedOlSevenFgThreeFlip & \NIclrRebEOnePackedOlSevenFgFTvPairs & \NIclrRebEOnePackedOlSevenFgFFlip \\
Qwen2.5-3B & \NIclrRebEOnePackedQThreePairs & \NIclrRebEOnePackedQThreeFgThreeTvPairs & \NIclrRebEOnePackedQThreeFgThreeFlip & \NIclrRebEOnePackedQThreeFgFTvPairs & \NIclrRebEOnePackedQThreeFgFFlip \\
\bottomrule
\end{tabular}}
\end{table}

SubCuber \citep{Jangda_2026} compiles Strassen-like algorithms into fp16 kernels; on Hopper GPUs it ships one level of Winograd's variant of Strassen's algorithm, whose block sums of the activations also combine row $r$ with row $r+M/2$ and are rounded to fp16 before the products.
Its released kernels do not run the tests above: at 1,024 rows none of them runs any of the four shapes of Qwen2.5-14B-Instruct's linear layers below the output layer, and at 4,096 rows the two that compute its query and output projections correctly fail on its other projections.
We therefore measure the mechanism on the kernels themselves, all \NIclrRebETenNVariants{} variants, at a shape they all accept: random activations of 8,192 rows times a random 8,192-by-8,192 weight matrix.
Replacing only the second half of the rows changes \NIclrRebETenSecondHalfChangedMin\% of the output entries of the first half; the largest change is \NIclrRebETenMaxChangePct\% of their mean magnitude.
Replacing one block of 1,024 rows in the second half changes \NIclrRebETenPartnerChangedMin\% of the entries of the first-half rows paired with it, and no entry of the other first-half rows; cuBLAS in fp16 changes no entry in either case.}

\rev{\subsection{Inside a serving engine}\label{app:engine}
The packed test of Section~\ref{sec:reach} also runs inside SGLang \citep{zheng2024sglang}, under the engine's own batching.
Each of 16 prompts A, a segment of 1,024 WikiText-2 test tokens, is sent in eight pairs with prompts B of the same length; a probe confirms that every pair runs as one prefill forward holding exactly the two sequences, with no cached prefix, so that only B's tokens differ between the pairs of an A, and repeating a pair reproduces it exactly.
A comparison sets A's outputs in each later pair against those in its first pair: the log probabilities of all of A's input positions, and A's greedy continuation of 32 tokens.
Besides the bf16 model and SGLang's FP8 path, the arms are the certified path, the configuration of Section~\ref{sec:cost}, and a fast FP8 path, our adapter of one level of Strassen with its block sums rounded to FP8 and no repair.
Under the fast FP8 path A's log probabilities change in all \NIclrRebEEightComparisons{} comparisons of each model and its continuation in \NIclrRebEEightLlEightScCOut{} (Llama-3.1-8B-Instruct) and \NIclrRebEEightQThreeScCOut{} (Qwen3-8B) of them (Table~\ref{tab:engine}); under the other three arms neither changes in any comparison.

\begin{table}[h]
\caption{\rev{Inside SGLang: of \NIclrRebEEightComparisons{} comparisons per model, those in which request A's input log probabilities, or its 32-token greedy continuation, change when only the co-scheduled request B changes.}}
\label{tab:engine}
\centering
\small
\setlength{\tabcolsep}{6pt}
\rev{\begin{tabular}{@{}l rr rr@{}}
\toprule
 & \multicolumn{2}{c}{Llama-3.1-8B-Instruct} & \multicolumn{2}{c}{Qwen3-8B} \\
\cmidrule(lr){2-3}\cmidrule(l){4-5}
arm & log probs & continuation & log probs & continuation \\
\midrule
bf16 & \NIclrRebEEightLlEightBfOneSixCIn & \NIclrRebEEightLlEightBfOneSixCOut & \NIclrRebEEightQThreeBfOneSixCIn & \NIclrRebEEightQThreeBfOneSixCOut \\
SGLang FP8 & \NIclrRebEEightLlEightFpEightCIn & \NIclrRebEEightLlEightFpEightCOut & \NIclrRebEEightQThreeFpEightCIn & \NIclrRebEEightQThreeFpEightCOut \\
certified int8, one level & \NIclrRebEEightLlEightQspkCIn & \NIclrRebEEightLlEightQspkCOut & \NIclrRebEEightQThreeQspkCIn & \NIclrRebEEightQThreeQspkCOut \\
fast FP8 Strassen & \NIclrRebEEightLlEightScCIn & \NIclrRebEEightLlEightScCOut & \NIclrRebEEightQThreeScCIn & \NIclrRebEEightQThreeScCOut \\
\bottomrule
\end{tabular}}
\end{table}}
\section{Construction details}\label{app:construction}

\paragraph{Bitwise checks in the models.}
In the verification runs of Section~\ref{sec:certificate}, the output of two-level Strassen at code bound \NIclrActFourStrassenTwoHeadroomBA{} equals classical int8's bit for bit in all \NIclrTwinQFourteenEqual{} invocations of a replaced linear layer on Qwen2.5-14B-Instruct, \NIclrTwinLlEightEqual{} on Llama-3.1-8B-Instruct, \NIclrTwinOlSevenEqual{} on OLMo-2-7B-Instruct and \NIclrTwinQThreeEightEqual{} on Qwen3-8B.

\paragraph{Quantization specification for the model runs.}
The quantizer takes fp32 inputs and computes scales in fp32, grouping 128 consecutive entries along the inner dimension, with a separate activation scale for each token row and weight scale for each output column within a group.
For a group $x$, its scale is $d=\max|x|/b$, with $d=1$ for an all-zero group, and its codes are $x/d$ rounded to nearest with ties to even, with no additional clipping.
The inner dimension is zero-padded to a multiple of 128.
For each group in ascending inner-dimension order, each entry $P_{tj}$ of the exact integer group product is rescaled as $(P_{tj}d_{A,t})d_{B,j}$ using two separate fp32 multiplications and added to a zero-initialized fp32 output.
After the last group the output is cast to the model input dtype, bf16 here, and the bias is then added.
The model runs at $b=\NIclrActFourStrassenTwoHeadroomBA{}$ and $b=127$ differ only in the code bound.

\paragraph{The code bound in context.}
Quantized fast convolution meets the same growth in its transformed tiles and handles it by clipping them and retraining \citep{mori2024wino}; condition~(i) of Proposition~\ref{prop:certificate} instead keeps the codes small enough that no block sum overflows.

\paragraph{The code bounds of condition (i).}
The code bounds $\lfloor127/L_A\rfloor$ and $\lfloor127/L_B\rfloor$ that condition~(i) admits are sharp, since codes whose signs follow the coefficients reach them; \citet[Theorem~3.1]{dumas2009fflas} prove a sharp bound of this kind for the integer intermediates of Winograd's variant of Strassen's algorithm before reduction modulo a prime.
For two-level Strassen the bound \NIclrActFourStrassenTwoHeadroomBA{} makes the quantization step \NIclrActFourStepRatioStrassenSqPredicted{} times as coarse as that of full-range int8.

\paragraph{Other algorithms with 49 multiplications.}
Other algorithms with the same number of multiplications pay more than two-level Strassen: each of the \NIclrActFourCatalogueN{} algorithms with 49 multiplications for $4\times4$ blocks released with AlphaTensor \citep{Fawzi_2022} has $L_A\ge\NIclrActFourCatalogueLAMin$, so none admits a code bound above \NIclrActFourCatalogueBABest.

\paragraph{Dense uniform codebooks on FP8.}
The code bound is set by the need for every block sum to stay exact, and an 8-bit floating-point format meets that need with far fewer evenly spaced codes than int8.
In a binary floating-point format with $p$ significant bits, the implicit bit included, a dense uniform codebook $\{0,\pm\delta,\dots,\pm J\delta\}$ whose every sum of $L$ codes is representable in the same format has $J\le\lfloor2^p/L\rfloor$.
Such sums include every multiple $q\delta$ with $|q|\le LJ$, so it suffices that one multiple beyond $2^p$ is not representable: writing $\delta=a\cdot2^s$ with $a$ odd, $(2^p+1)\delta$ has the odd part $(2^p+1)a\ge2^p+1$, which needs more than $p$ significant bits.
Hence $LJ\le2^p$, and a finite exponent range can only lower the bound.
For e4m3, with four exponent bits and three stored significand bits ($p=4$), and the $L_A=\NIclrActFourStrassenTwoHeadroomLA$ codes that a block sum of two-level Strassen adds, at most nine evenly spaced codes remain (for example $-4,\dots,4$), against the $2\cdot\NIclrActFourStrassenTwoHeadroomBA+1=63$ codes that condition~(i) admits in int8 at the same eight bits.

\paragraph{The corrected computation's certificate.}
At code bound 127, with groups of 128 and blocks of inner length $h=32$, so that a call spans one group as condition~(iii) requires, two-level Strassen with the overflow correction of Section~\ref{sec:correction} computes classical int8 bit for bit, provided its integer arithmetic runs in int32 and the quantization, the rescaling and the order of the fp32 additions across groups are those of classical int8.
A block sum of codes, formed exactly in a wider integer type, satisfies $|\hat A|,|\hat B|\le\NIclrActFourStrassenTwoHeadroomLA\cdot127=508$, which int8 cannot hold, so the unsplit computation violates condition~(i) of Proposition~\ref{prop:certificate}; the following split instead keeps every multiplicand in int8.
With $R_A=\lfloor(\hat A+128)/256\rfloor$ and $A_0=\hat A-256R_A$, and likewise for $\hat B$, the parts satisfy $A_0,B_0\in[-128,127]$ and $|R_A|,|R_B|\le2$, so each of the four products in~(\ref{eq:correction}) multiplies int8 values.
For one entry of a reconstructed block product over a block of inner length $h$, the absolute values of all the weighted scalar terms of~(\ref{eq:correction}) add up to at most $h\,(128+256\cdot2)^2=13{,}107{,}200$.
The combination $C_{ij}=\sum_r w_{r,ij}M_r$ of~(\ref{eq:bilinear}) adds at most $L_W=\NIclrActFourStrassenTwoHeadroomLW$ reconstructed products for an output block, so every partial sum stays below $16\cdot13{,}107{,}200<2^{31}$ and int32 computes each block of the output exactly, in any order of its integer additions.
By~(\ref{eq:correction}) and the algorithm's exact identity, each group's integer result is then the classical group product at code bound 127.
That product is bounded by $128\cdot127^2=2{,}064{,}512<2^{24}$, so its conversion to fp32 is exact, and the unchanged rescaling and fp32 order give classical int8's output bits.
The three layer slices of Appendix~\ref{app:cost} confirm this identity on each of them.

\paragraph{Model quality.}
Table~\ref{tab:quality} lists the excess values that Figure~\ref{fig:quality} plots, with the bf16 model's own negative log-likelihood beside them.
Table~\ref{tab:obqa-quality} gives ordinary task accuracy on the same first \NIclrActTwoQOneFourObqaItems{} OpenBookQA test items, using the natural inputs and likelihood scoring of Section~\ref{sec:support}.

\begin{table}[h]
\caption{Increase in negative log-likelihood per token over the bf16 model on eight 2048-token chunks of WikiText-2 ($10^{-3}$ nats; lower is better), raw and with every operator and the bf16 model at its own best temperature on the tested grid, chosen for each chunk on the other seven (adjusted). The bf16 NLL is raw. Each certified column is measured on the classical int8 operator its realization computes bit for bit, the one at $b=127$ with the overflow correction; all three operators use groups of 128 along the inner dimension. \rev{The int8 operators have one weight scale per output column in each group and the FP8 kernel one per $128\times128$ block; Table~\ref{tab:quality-matched} compares them at matched weight scales.}}
\label{tab:quality}
\centering
\small
\setlength{\tabcolsep}{4pt}
\newcolumntype{Q}{>{\raggedleft\arraybackslash}p{1.3cm}}
\begin{tabular}{@{}l c QQ QQ QQ@{}}
\toprule
 & bf16 NLL & \multicolumn{2}{c}{deployed FP8 kernel} & \multicolumn{2}{c}{certified, $b=\NIclrActFourStrassenTwoHeadroomBA$} & \multicolumn{2}{c}{certified, $b=127$} \\
\cmidrule(lr){3-4}\cmidrule(lr){5-6}\cmidrule(l){7-8}
 & (nats) & raw & adjusted & raw & adjusted & raw & adjusted \\
\midrule
Qwen2.5-14B-Instruct & \NIclrWThreeEightNineQFourteenNllCleanTOne & $\NIclrQualityQFourteenCubicRaw$ & $\NIclrQualityQFourteenCubicAdj$ & $\NIclrQualityQFourteenCintThirtyOneRaw$ & $\NIclrQualityQFourteenCintThirtyOneAdj$ & $\NIclrQualityQFourteenCintOneTwentySevenRaw$ & $\NIclrQualityQFourteenCintOneTwentySevenAdj$ \\
Llama-3.1-8B-Instruct & \NIclrWThreeEightNineLlEightNllCleanTOne & $\NIclrQualityLlEightCubicRaw$ & $\NIclrQualityLlEightCubicAdj$ & $\NIclrQualityLlEightCintThirtyOneRaw$ & $\NIclrQualityLlEightCintThirtyOneAdj$ & $\NIclrQualityLlEightCintOneTwentySevenRaw$ & $\NIclrQualityLlEightCintOneTwentySevenAdj$ \\
OLMo-2-7B-Instruct & \NIclrWThreeEightNineOlSevenNllCleanTOne & $\NIclrQualityOlSevenCubicRaw$ & $\NIclrQualityOlSevenCubicAdj$ & $\NIclrQualityOlSevenCintThirtyOneRaw$ & $\NIclrQualityOlSevenCintThirtyOneAdj$ & $\NIclrQualityOlSevenCintOneTwentySevenRaw$ & $\NIclrQualityOlSevenCintOneTwentySevenAdj$ \\
Qwen3-8B & \NIclrWThreeEightNineQThreeEightNllCleanTOne & $\NIclrQualityQThreeEightCubicRaw$ & $\NIclrQualityQThreeEightCubicAdj$ & $\NIclrQualityQThreeEightCintThirtyOneRaw$ & $\NIclrQualityQThreeEightCintThirtyOneAdj$ & $\NIclrQualityQThreeEightCintOneTwentySevenRaw$ & $\NIclrQualityQThreeEightCintOneTwentySevenAdj$ \\
\bottomrule
\end{tabular}
\end{table}

\begin{table}[h]
\caption{OpenBookQA accuracy (\%) on the first \NIclrActTwoQOneFourObqaItems{} test items. The int8 columns are measured on the classical reference operators that the certified realizations compute, with groups of 128 and the quantization specification above. Each answer is chosen by its summed token log probabilities, with the first candidate selected in an exact tie.}
\label{tab:obqa-quality}
\centering
\small
\setlength{\tabcolsep}{6pt}
\begin{tabular}{@{}l r r r r@{}}
\toprule
 & & & \multicolumn{2}{c}{classical int8 reference} \\
\cmidrule(l){4-5}
model & bf16 & deployed FP8 & $b=\NIclrActFourStrassenTwoHeadroomBA$ & $b=127$ \\
\midrule
Qwen2.5-14B-Instruct & \NIclrActTwoQOneFourObqaPlausibleAccPctClean & \NIclrActTwoQOneFourObqaPlausibleAccPctCubic & \NIclrLadderQfourteenAccPctCintThirtyOne & \NIclrAccQfourteenIntFull \\
Llama-3.1-8B-Instruct & \NIclrAccLlamaBf & \NIclrAccLlamaFp & \NIclrAccLlamaIntLow & \NIclrAccLlamaIntFull \\
OLMo-2-7B-Instruct & \NIclrAccOlmoBf & \NIclrAccOlmoFp & \NIclrAccOlmoIntLow & \NIclrAccOlmoIntFull \\
Qwen3-8B & \NIclrAccQthreeBf & \NIclrAccQthreeFp & \NIclrAccQthreeIntLow & \NIclrAccQthreeIntFull \\
\bottomrule
\end{tabular}
\end{table}

\paragraph{Qwen3-8B at code bound \NIclrActFourStrassenTwoHeadroomBA.}
On Qwen3-8B the raw excess at code bound \NIclrActFourStrassenTwoHeadroomBA{} is negative; its bf16 model fits this text best at temperature \NIclrWThreeEightNineQThreeEightTstarClean{} on the tested grid.

\rev{\paragraph{The deployed configuration.}
Section~\ref{sec:cost} deploys one level of Strassen at code bound 63 under a quantization specification of its own: each activation row is rotated by an online block Hadamard transform \citep{ashkboos-quarot-2024} and quantized with one scale per token over the whole inner dimension, each weight row is a GPTQ code \citep{frantar2023gptq} of the rotated weight with one scale per output channel, calibrated on PG-19 text \citep{rae2019compressive} and never on WikiText-2, and the query, key and value projections run classical int8 on activations at code bound 127.
Its certified realization computes the classical int8 operator of this specification bit for bit, and its fast calls matched that operator in all \NIclrRebETwoPTwoRSixThreeMainCalls{} checked calls.
On the protocol of Table~\ref{tab:quality}, its excess NLL is at or below the deployed FP8 kernel's on all four models, raw and adjusted (Table~\ref{tab:quality-deployed}).

\begin{table}[h]
\caption{\rev{Increase in negative log-likelihood per token over the bf16 model for the configuration that Section~\ref{sec:cost} deploys, on the protocol of Table~\ref{tab:quality} ($10^{-3}$ nats; lower is better).}}
\label{tab:quality-deployed}
\centering
\small
\setlength{\tabcolsep}{6pt}
\rev{\begin{tabular}{@{}l rr rr@{}}
\toprule
 & \multicolumn{2}{c}{deployed FP8 kernel} & \multicolumn{2}{c}{certified, one level, $b=63$} \\
\cmidrule(lr){2-3}\cmidrule(l){4-5}
model & raw & adjusted & raw & adjusted \\
\midrule
Qwen2.5-14B-Instruct & $\NIclrRebETwoQOneFourFpEightRaw$ & $\NIclrRebETwoQOneFourFpEightAdj$ & $\NIclrRebETwoQOneFourRSixThreeMainRaw$ & $\NIclrRebETwoQOneFourRSixThreeMainAdj$ \\
Llama-3.1-8B-Instruct & $\NIclrRebETwoLlEightFpEightRaw$ & $\NIclrRebETwoLlEightFpEightAdj$ & $\NIclrRebETwoLlEightRSixThreeMainRaw$ & $\NIclrRebETwoLlEightRSixThreeMainAdj$ \\
OLMo-2-7B-Instruct & $\NIclrRebETwoOlSevenFpEightRaw$ & $\NIclrRebETwoOlSevenFpEightAdj$ & $\NIclrRebETwoOlSevenRSixThreeMainRaw$ & $\NIclrRebETwoOlSevenRSixThreeMainAdj$ \\
Qwen3-8B & $\NIclrRebETwoQThreeEightFpEightRaw$ & $\NIclrRebETwoQThreeEightFpEightAdj$ & $\NIclrRebETwoQThreeEightRSixThreeMainRaw$ & $\NIclrRebETwoQThreeEightRSixThreeMainAdj$ \\
\bottomrule
\end{tabular}}
\end{table}}
\rev{\paragraph{Matched weight scales.}
The int8 columns of Table~\ref{tab:quality} give each output column its own weight scale within a group, while the deployed FP8 kernel shares one weight scale over each $128\times128$ block.
Table~\ref{tab:quality-matched} matches the two in both directions: classical int8 with the kernel's $128\times128$ weight blocks, and e4m3 codes with one weight scale per output column within a group, rescaled and accumulated as classical int8.
The certified realizations compute the block-scaled int8 operators bit for bit as well, since each scale is constant over a group of 128 inner entries and multiplies the group's exact integer product; on one chunk per model, two-level Strassen at code bound \NIclrActFourStrassenTwoHeadroomBA{} returned the output bits of classical int8 with the same blocks in all \NIclrRebEElevenWitnessSitesMin--\NIclrRebEElevenWitnessSitesMax{} calls of the replaced linear layers.
With the kernel's blocks, classical int8 at code bound 127 has a lower excess than the kernel on three of the four models, raw and adjusted.
Under the same emulated rescaling and accumulation, e4m3 codes at the kernel's own blocks differ from the kernel's raw excess by $\NIclrRebEElevenSchedDiffMin$ to \NIclrRebEElevenSchedDiffMax, so we compare the e4m3 codes with the int8 columns of Table~\ref{tab:quality}, which share that rescaling and accumulation, rather than with the kernel.
At per-column scales, their excess stays above that of classical int8 at code bound 127 on three of the four models raw and on all four adjusted.
Per-column scales store one fp32 value for every 128 weights, against one for every 16,384 with $128\times128$ blocks, and the rescaling performs the same multiplications with either.
The configuration that Section~\ref{sec:cost} deploys uses neither: it applies its scale per token and its scale per output channel once, to the int32 sum over the whole inner dimension.

\begin{table}[h]
\caption{\rev{Excess NLL per token over bf16 at matched weight scales, on the protocol of Table~\ref{tab:quality} ($10^{-3}$ nats; lower is better). The int8 columns use the deployed FP8 kernel's $128\times128$ weight blocks; the e4m3 column uses the per-column weight scales of Table~\ref{tab:quality}'s int8 columns, with the same rescaling and accumulation.}}
\label{tab:quality-matched}
\centering
\small
\setlength{\tabcolsep}{4pt}
\newcolumntype{Q}{>{\raggedleft\arraybackslash}p{1.3cm}}
\rev{\begin{tabular}{@{}l QQ QQ QQ@{}}
\toprule
 & \multicolumn{2}{c}{e4m3, per-column scales} & \multicolumn{2}{c}{int8, $b=\NIclrActFourStrassenTwoHeadroomBA$, blocks} & \multicolumn{2}{c}{int8, $b=127$, blocks} \\
\cmidrule(lr){2-3}\cmidrule(lr){4-5}\cmidrule(l){6-7}
model & raw & adjusted & raw & adjusted & raw & adjusted \\
\midrule
Qwen2.5-14B-Instruct & $\NIclrRebEElevenQOneFourCfpEightRaw$ & $\NIclrRebEElevenQOneFourCfpEightAdj$ & $\NIclrRebEElevenQOneFourCintThirtyOneWbRaw$ & $\NIclrRebEElevenQOneFourCintThirtyOneWbAdj$ & $\NIclrRebEElevenQOneFourCintOneTwentySevenWbRaw$ & $\NIclrRebEElevenQOneFourCintOneTwentySevenWbAdj$ \\
Llama-3.1-8B-Instruct & $\NIclrRebEElevenLlEightCfpEightRaw$ & $\NIclrRebEElevenLlEightCfpEightAdj$ & $\NIclrRebEElevenLlEightCintThirtyOneWbRaw$ & $\NIclrRebEElevenLlEightCintThirtyOneWbAdj$ & $\NIclrRebEElevenLlEightCintOneTwentySevenWbRaw$ & $\NIclrRebEElevenLlEightCintOneTwentySevenWbAdj$ \\
OLMo-2-7B-Instruct & $\NIclrRebEElevenOlSevenCfpEightRaw$ & $\NIclrRebEElevenOlSevenCfpEightAdj$ & $\NIclrRebEElevenOlSevenCintThirtyOneWbRaw$ & $\NIclrRebEElevenOlSevenCintThirtyOneWbAdj$ & $\NIclrRebEElevenOlSevenCintOneTwentySevenWbRaw$ & $\NIclrRebEElevenOlSevenCintOneTwentySevenWbAdj$ \\
Qwen3-8B & $\NIclrRebEElevenQThreeEightCfpEightRaw$ & $\NIclrRebEElevenQThreeEightCfpEightAdj$ & $\NIclrRebEElevenQThreeEightCintThirtyOneWbRaw$ & $\NIclrRebEElevenQThreeEightCintThirtyOneWbAdj$ & $\NIclrRebEElevenQThreeEightCintOneTwentySevenWbRaw$ & $\NIclrRebEElevenQThreeEightCintOneTwentySevenWbAdj$ \\
\bottomrule
\end{tabular}}
\end{table}}

\rev{\paragraph{FP8 with the rotation and GPTQ codes.}
The deployed configuration rotates the activations and computes its weight codes with GPTQ, while the deployed FP8 kernel of Table~\ref{tab:quality-deployed} has neither.
Table~\ref{tab:quality-fp8tools} gives FP8 the same tools on the protocol of Table~\ref{tab:quality}, in two layouts: the deployed FP8 kernel, with its $128\times128$ weight blocks, and the per-token FP8 format that SGLang serves, with one weight scale per output channel.
Each gets the rotation, GPTQ e4m3 weight codes calibrated as our int8 codes are, or both.
GPTQ sets each weight scale in one of three ways: by the mean-squared-error search of our int8 codes (MSE), as the largest magnitude of the scale's weights divided by 448, the largest e4m3 value (amax), or as 1.25 times that, which leaves room for GPTQ's updates (headroom).
On Qwen2.5-14B-Instruct and Llama-3.1-8B-Instruct, each of the fourteen variants has a larger raw excess than the deployed configuration, and on OLMo-2-7B-Instruct and Qwen3-8B some have a smaller one.
Eight chunks resolve none of the differences between the deployed configuration and FP8's best variant on each model, the one with the smallest raw excess, a choice that favors FP8: a two-sided exact sign-flip test over the chunks gives $p=\NIclrRebEOneTwobQOneFourSignflipPVsFbestTwo$, \NIclrRebEOneTwobLlEightSignflipPVsFbestTwo, \NIclrRebEOneTwobOlSevenSignflipPVsFbestTwo{} and \NIclrRebEOneTwobQThreeEightSignflipPVsFbestTwo.
Nor do they resolve the differences between the deployed configuration and the deployed FP8 kernel in Table~\ref{tab:quality-deployed} ($p=\NIclrRebEOneTwobQOneFourSignflipPVsCubic$, \NIclrRebEOneTwobLlEightSignflipPVsCubic, \NIclrRebEOneTwobOlSevenSignflipPVsCubic{} and \NIclrRebEOneTwobQThreeEightSignflipPVsCubic).
We therefore repeat the comparison on the two 8B models in SGLang, with the per-token NLL protocol of Appendix~\ref{app:cost}, where each of the 32 books gives a paired difference (Table~\ref{tab:sglang-fp8tools}).
There, FP8's best variant has \NIclrRebEOneThreebQThreeFpEightbestOverQspk{} times the deployed configuration's per-token NLL increase on Qwen3-8B and \NIclrRebEOneThreebLlEightFpEightbestOverQspk{} times on Llama-3.1-8B, and the paired 95\% bootstrap intervals of the deployed configuration's increase minus that of FP8's best variant, [$\NIclrRebEOneThreebQThreeQspkMinusFpEightbestCiNineFiveLo$, $\NIclrRebEOneThreebQThreeQspkMinusFpEightbestCiNineFiveHi$] and [$\NIclrRebEOneThreebLlEightQspkMinusFpEightbestCiNineFiveLo$, $\NIclrRebEOneThreebLlEightQspkMinusFpEightbestCiNineFiveHi$], exclude zero.
With GPTQ codes and no rotation, SGLang's FP8 path runs its own FP8 kernels on the same shapes, and the rotation adds an online transform before them.

\begin{table}[h]
\caption{\rev{Excess NLL per token over bf16 for FP8 with the deployed configuration's rotation and GPTQ weight codes, on the protocol of Table~\ref{tab:quality} ($10^{-3}$ nats; lower is better). Kernel: the deployed FP8 kernel, with $128\times128$ weight blocks; per-token: the per-token FP8 format, with one weight scale per output channel. MSE, amax and headroom name GPTQ's scale rule. The first row of each block uses neither tool; the last row is the deployed configuration of Table~\ref{tab:quality-deployed}.}}
\label{tab:quality-fp8tools}
\centering
\footnotesize
\setlength{\tabcolsep}{3.5pt}
\rev{\begin{tabular}{@{}l rr rr rr rr@{}}
\toprule
 & \multicolumn{2}{c}{Qwen2.5-14B} & \multicolumn{2}{c}{Llama-3.1-8B} & \multicolumn{2}{c}{OLMo-2-7B} & \multicolumn{2}{c}{Qwen3-8B} \\
\cmidrule(lr){2-3}\cmidrule(lr){4-5}\cmidrule(lr){6-7}\cmidrule(l){8-9}
 & raw & adjusted & raw & adjusted & raw & adjusted & raw & adjusted \\
\midrule
kernel & $\NIclrRebEOneTwoQOneFourCubicRaw$ & $\NIclrRebEOneTwoQOneFourCubicAdj$ & $\NIclrRebEOneTwoLlEightCubicRaw$ & $\NIclrRebEOneTwoLlEightCubicAdj$ & $\NIclrRebEOneTwoOlSevenCubicRaw$ & $\NIclrRebEOneTwoOlSevenCubicAdj$ & $\NIclrRebEOneTwoQThreeEightCubicRaw$ & $\NIclrRebEOneTwoQThreeEightCubicAdj$ \\
kernel, rotation & $\NIclrRebEOneTwoQOneFourFpEightblkRotRaw$ & $\NIclrRebEOneTwoQOneFourFpEightblkRotAdj$ & $\NIclrRebEOneTwoLlEightFpEightblkRotRaw$ & $\NIclrRebEOneTwoLlEightFpEightblkRotAdj$ & $\NIclrRebEOneTwoOlSevenFpEightblkRotRaw$ & $\NIclrRebEOneTwoOlSevenFpEightblkRotAdj$ & $\NIclrRebEOneTwoQThreeEightFpEightblkRotRaw$ & $\NIclrRebEOneTwoQThreeEightFpEightblkRotAdj$ \\
kernel, GPTQ (MSE) & $\NIclrRebEOneTwoQOneFourFpEightblkGptqRaw$ & $\NIclrRebEOneTwoQOneFourFpEightblkGptqAdj$ & $\NIclrRebEOneTwoLlEightFpEightblkGptqRaw$ & $\NIclrRebEOneTwoLlEightFpEightblkGptqAdj$ & $\NIclrRebEOneTwoOlSevenFpEightblkGptqRaw$ & $\NIclrRebEOneTwoOlSevenFpEightblkGptqAdj$ & $\NIclrRebEOneTwoQThreeEightFpEightblkGptqRaw$ & $\NIclrRebEOneTwoQThreeEightFpEightblkGptqAdj$ \\
kernel, GPTQ (amax) & $\NIclrRebEOneTwobQOneFourFpEightblkGptqARaw$ & $\NIclrRebEOneTwobQOneFourFpEightblkGptqAAdj$ & $\NIclrRebEOneTwobLlEightFpEightblkGptqARaw$ & $\NIclrRebEOneTwobLlEightFpEightblkGptqAAdj$ & $\NIclrRebEOneTwobOlSevenFpEightblkGptqARaw$ & $\NIclrRebEOneTwobOlSevenFpEightblkGptqAAdj$ & $\NIclrRebEOneTwobQThreeEightFpEightblkGptqARaw$ & $\NIclrRebEOneTwobQThreeEightFpEightblkGptqAAdj$ \\
kernel, GPTQ (headroom) & $\NIclrRebEOneTwobQOneFourFpEightblkGptqHRaw$ & $\NIclrRebEOneTwobQOneFourFpEightblkGptqHAdj$ & $\NIclrRebEOneTwobLlEightFpEightblkGptqHRaw$ & $\NIclrRebEOneTwobLlEightFpEightblkGptqHAdj$ & $\NIclrRebEOneTwobOlSevenFpEightblkGptqHRaw$ & $\NIclrRebEOneTwobOlSevenFpEightblkGptqHAdj$ & $\NIclrRebEOneTwobQThreeEightFpEightblkGptqHRaw$ & $\NIclrRebEOneTwobQThreeEightFpEightblkGptqHAdj$ \\
kernel, rotation + GPTQ (MSE) & $\NIclrRebEOneTwoQOneFourFpEightblkRotGptqRaw$ & $\NIclrRebEOneTwoQOneFourFpEightblkRotGptqAdj$ & $\NIclrRebEOneTwoLlEightFpEightblkRotGptqRaw$ & $\NIclrRebEOneTwoLlEightFpEightblkRotGptqAdj$ & $\NIclrRebEOneTwoOlSevenFpEightblkRotGptqRaw$ & $\NIclrRebEOneTwoOlSevenFpEightblkRotGptqAdj$ & $\NIclrRebEOneTwoQThreeEightFpEightblkRotGptqRaw$ & $\NIclrRebEOneTwoQThreeEightFpEightblkRotGptqAdj$ \\
kernel, rotation + GPTQ (amax) & $\NIclrRebEOneTwobQOneFourFpEightblkRotGptqARaw$ & $\NIclrRebEOneTwobQOneFourFpEightblkRotGptqAAdj$ & $\NIclrRebEOneTwobLlEightFpEightblkRotGptqARaw$ & $\NIclrRebEOneTwobLlEightFpEightblkRotGptqAAdj$ & $\NIclrRebEOneTwobOlSevenFpEightblkRotGptqARaw$ & $\NIclrRebEOneTwobOlSevenFpEightblkRotGptqAAdj$ & $\NIclrRebEOneTwobQThreeEightFpEightblkRotGptqARaw$ & $\NIclrRebEOneTwobQThreeEightFpEightblkRotGptqAAdj$ \\
kernel, rotation + GPTQ (headroom) & $\NIclrRebEOneTwobQOneFourFpEightblkRotGptqHRaw$ & $\NIclrRebEOneTwobQOneFourFpEightblkRotGptqHAdj$ & $\NIclrRebEOneTwobLlEightFpEightblkRotGptqHRaw$ & $\NIclrRebEOneTwobLlEightFpEightblkRotGptqHAdj$ & $\NIclrRebEOneTwobOlSevenFpEightblkRotGptqHRaw$ & $\NIclrRebEOneTwobOlSevenFpEightblkRotGptqHAdj$ & $\NIclrRebEOneTwobQThreeEightFpEightblkRotGptqHRaw$ & $\NIclrRebEOneTwobQThreeEightFpEightblkRotGptqHAdj$ \\
\midrule
per-token & $\NIclrRebEOneTwoQOneFourFpEightptRaw$ & $\NIclrRebEOneTwoQOneFourFpEightptAdj$ & $\NIclrRebEOneTwoLlEightFpEightptRaw$ & $\NIclrRebEOneTwoLlEightFpEightptAdj$ & $\NIclrRebEOneTwoOlSevenFpEightptRaw$ & $\NIclrRebEOneTwoOlSevenFpEightptAdj$ & $\NIclrRebEOneTwoQThreeEightFpEightptRaw$ & $\NIclrRebEOneTwoQThreeEightFpEightptAdj$ \\
per-token, rotation & $\NIclrRebEOneTwoQOneFourFpEightptRotRaw$ & $\NIclrRebEOneTwoQOneFourFpEightptRotAdj$ & $\NIclrRebEOneTwoLlEightFpEightptRotRaw$ & $\NIclrRebEOneTwoLlEightFpEightptRotAdj$ & $\NIclrRebEOneTwoOlSevenFpEightptRotRaw$ & $\NIclrRebEOneTwoOlSevenFpEightptRotAdj$ & $\NIclrRebEOneTwoQThreeEightFpEightptRotRaw$ & $\NIclrRebEOneTwoQThreeEightFpEightptRotAdj$ \\
per-token, GPTQ (MSE) & $\NIclrRebEOneTwoQOneFourFpEightptGptqRaw$ & $\NIclrRebEOneTwoQOneFourFpEightptGptqAdj$ & $\NIclrRebEOneTwoLlEightFpEightptGptqRaw$ & $\NIclrRebEOneTwoLlEightFpEightptGptqAdj$ & $\NIclrRebEOneTwoOlSevenFpEightptGptqRaw$ & $\NIclrRebEOneTwoOlSevenFpEightptGptqAdj$ & $\NIclrRebEOneTwoQThreeEightFpEightptGptqRaw$ & $\NIclrRebEOneTwoQThreeEightFpEightptGptqAdj$ \\
per-token, GPTQ (amax) & $\NIclrRebEOneTwobQOneFourFpEightptGptqARaw$ & $\NIclrRebEOneTwobQOneFourFpEightptGptqAAdj$ & $\NIclrRebEOneTwobLlEightFpEightptGptqARaw$ & $\NIclrRebEOneTwobLlEightFpEightptGptqAAdj$ & $\NIclrRebEOneTwobOlSevenFpEightptGptqARaw$ & $\NIclrRebEOneTwobOlSevenFpEightptGptqAAdj$ & $\NIclrRebEOneTwobQThreeEightFpEightptGptqARaw$ & $\NIclrRebEOneTwobQThreeEightFpEightptGptqAAdj$ \\
per-token, GPTQ (headroom) & $\NIclrRebEOneTwobQOneFourFpEightptGptqHRaw$ & $\NIclrRebEOneTwobQOneFourFpEightptGptqHAdj$ & $\NIclrRebEOneTwobLlEightFpEightptGptqHRaw$ & $\NIclrRebEOneTwobLlEightFpEightptGptqHAdj$ & $\NIclrRebEOneTwobOlSevenFpEightptGptqHRaw$ & $\NIclrRebEOneTwobOlSevenFpEightptGptqHAdj$ & $\NIclrRebEOneTwobQThreeEightFpEightptGptqHRaw$ & $\NIclrRebEOneTwobQThreeEightFpEightptGptqHAdj$ \\
per-token, rotation + GPTQ (MSE) & $\NIclrRebEOneTwoQOneFourFpEightptRotGptqRaw$ & $\NIclrRebEOneTwoQOneFourFpEightptRotGptqAdj$ & $\NIclrRebEOneTwoLlEightFpEightptRotGptqRaw$ & $\NIclrRebEOneTwoLlEightFpEightptRotGptqAdj$ & $\NIclrRebEOneTwoOlSevenFpEightptRotGptqRaw$ & $\NIclrRebEOneTwoOlSevenFpEightptRotGptqAdj$ & $\NIclrRebEOneTwoQThreeEightFpEightptRotGptqRaw$ & $\NIclrRebEOneTwoQThreeEightFpEightptRotGptqAdj$ \\
per-token, rotation + GPTQ (amax) & $\NIclrRebEOneTwobQOneFourFpEightptRotGptqARaw$ & $\NIclrRebEOneTwobQOneFourFpEightptRotGptqAAdj$ & $\NIclrRebEOneTwobLlEightFpEightptRotGptqARaw$ & $\NIclrRebEOneTwobLlEightFpEightptRotGptqAAdj$ & $\NIclrRebEOneTwobOlSevenFpEightptRotGptqARaw$ & $\NIclrRebEOneTwobOlSevenFpEightptRotGptqAAdj$ & $\NIclrRebEOneTwobQThreeEightFpEightptRotGptqARaw$ & $\NIclrRebEOneTwobQThreeEightFpEightptRotGptqAAdj$ \\
per-token, rotation + GPTQ (headroom) & $\NIclrRebEOneTwobQOneFourFpEightptRotGptqHRaw$ & $\NIclrRebEOneTwobQOneFourFpEightptRotGptqHAdj$ & $\NIclrRebEOneTwobLlEightFpEightptRotGptqHRaw$ & $\NIclrRebEOneTwobLlEightFpEightptRotGptqHAdj$ & $\NIclrRebEOneTwobOlSevenFpEightptRotGptqHRaw$ & $\NIclrRebEOneTwobOlSevenFpEightptRotGptqHAdj$ & $\NIclrRebEOneTwobQThreeEightFpEightptRotGptqHRaw$ & $\NIclrRebEOneTwobQThreeEightFpEightptRotGptqHAdj$ \\
\midrule
certified, deployed configuration & $\NIclrRebEOneTwoQOneFourRSixThreeMainRaw$ & $\NIclrRebEOneTwoQOneFourRSixThreeMainAdj$ & $\NIclrRebEOneTwoLlEightRSixThreeMainRaw$ & $\NIclrRebEOneTwoLlEightRSixThreeMainAdj$ & $\NIclrRebEOneTwoOlSevenRSixThreeMainRaw$ & $\NIclrRebEOneTwoOlSevenRSixThreeMainAdj$ & $\NIclrRebEOneTwoQThreeEightRSixThreeMainRaw$ & $\NIclrRebEOneTwoQThreeEightRSixThreeMainAdj$ \\
\bottomrule
\end{tabular}}
\end{table}

\begin{table}[h]
\caption{\rev{Per-token NLL increase over bf16 in SGLang, mean over 32 books (nats; lower is better), for SGLang's FP8 path with the deployed configuration's GPTQ codes and rotation, and for the deployed configuration. MSE, amax and headroom name GPTQ's scale rule, as in Table~\ref{tab:quality-fp8tools}.}}
\label{tab:sglang-fp8tools}
\centering
\small
\rev{\begin{tabular}{@{}l rr@{}}
\toprule
 & Qwen3-8B & Llama-3.1-8B-Instruct \\
\midrule
SGLang's FP8 path & \NIclrRebEOneThreeQThreeFpEightMean & \NIclrRebEOneThreeLlEightFpEightMean \\
GPTQ (MSE) & \NIclrRebEOneThreeQThreeFpEightgMean & \NIclrRebEOneThreeLlEightFpEightgMean \\
GPTQ (amax) & \NIclrRebEOneThreebQThreeFpEightgAMean & \NIclrRebEOneThreebLlEightFpEightgAMean \\
GPTQ (headroom) & \NIclrRebEOneThreebQThreeFpEightgHMean & \NIclrRebEOneThreebLlEightFpEightgHMean \\
rotation + GPTQ (MSE) & \NIclrRebEOneThreeQThreeFpEightrgMean & \NIclrRebEOneThreeLlEightFpEightrgMean \\
rotation + GPTQ (amax) & \NIclrRebEOneThreebQThreeFpEightrgAMean & \NIclrRebEOneThreebLlEightFpEightrgAMean \\
rotation + GPTQ (headroom) & \NIclrRebEOneThreebQThreeFpEightrgHMean & \NIclrRebEOneThreebLlEightFpEightrgHMean \\
\midrule
certified, deployed configuration & \NIclrRebEOneThreeQThreeQspkMean & \NIclrRebEOneThreeLlEightQspkMean \\
\bottomrule
\end{tabular}}
\end{table}}
\FloatBarrier

\section{Cost details}\label{app:cost}

\paragraph{\rev{The correction on three layer slices.}}
\rev{We check the corrected computation of Section~\ref{sec:correction}, two-level Strassen with the overflow correction at code bound 127,} on one slice from each of three linear layers: the down projection of layer 4 of Qwen2.5-14B-Instruct, the gate projection of layer 2 of Qwen3-8B and the down projection of layer 2 of Qwen2.5-3B, each with 512 token rows and 1024 output channels, at groups of 128 and blocks of inner length 32.
On each slice the corrected computation reproduces classical int8 at code bound 127 bit for bit.
Every reconstructed product and group output also matches an exact integer-valued float64 reference, with the integer bounds below $2^{53}$.
The entrywise bound $\sum_r |w_{r,ij}M_r|<2^{24}$ keeps every decoded partial sum exact in fp32 on these slices, and the final rescaled outputs have identical bit patterns\ifhmode\unskip\fi.
\rev{On these slices, its} correction terms, the products with $R_A$ or $R_B$ in~(\ref{eq:correction}), are counted in two ways.
Counting only the nonzero entries of $R_A$ and $R_B$ gives \NIclrCostCorrOneTwoSevenSparseMin--\NIclrCostCorrOneTwoSevenSparseMax{} of classical int8's multiplications in total; counting every row of $R_A$ and column of $R_B$ that holds a nonzero entry, as a dense product over those rows and columns would, gives \NIclrCostCorrOneTwoSevenGatheredMin--\NIclrCostCorrOneTwoSevenGatheredMax{} times as many.

\rev{\paragraph{Every layer of the four models.}
Section~\ref{sec:cost} reports the same counts over every linear layer except the output layer of the four models of Section~\ref{sec:quality}, on the activations of the bf16 model on its eight 2048-token WikiText-2 chunks and with all \NIclrActFourStrassenTwoGaugesN{} sign variants (Table~\ref{tab:overflow}).
Entrywise, the corrected computation performs fewer multiplications than classical int8 on every one of the \NIclrRebEFourLayersN{} layers (from \NIclrRebEFourLayerEntryMin{} to \NIclrRebEFourLayerEntryMax{} of them, median \NIclrRebEFourLayerEntryMedian); a dense product over the rows that hold an overflow performs more on every layer (median \NIclrRebEFourLayerRowsMedian{} times as many).
The overflows are spread thinly: on the median layer the top 1\% of the activations' encoded columns hold \NIclrRebEFourConcTopOneColsA\% of their overflowing entries and the top 1\% of rows \NIclrRebEFourConcTopOneRowsA\%.
The three slices above come from layers that overflow little: each lies below the 10th percentile of its model's layers (\NIclrRebEFourSlicesPTenMin--\NIclrRebEFourSlicesPTenMax, entrywise; for Qwen2.5-3B, which is not one of the four, of their layers pooled).
Since every sign variant computes the same output bits, the variant can be chosen per layer for the smallest count, but this lowers the entrywise count by a median of \NIclrRebEFourGaugeGainMedian{} per layer (at most \NIclrRebEFourGaugeGainMax).

\begin{table}[h]
\caption{\rev{The overflow correction at code bound 127 over every linear layer, per model: the percentage of block-sum entries that overflow int8 (median layer), and the corrected computation's multiplications relative to classical int8 (layers weighted by size), entrywise or dense over the rows or columns that hold an overflow.}}
\label{tab:overflow}
\centering
\small
\setlength{\tabcolsep}{5pt}
\rev{\begin{tabular}{@{}l rr rrr@{}}
\toprule
 & \multicolumn{2}{c}{overflowing entries (\%)} & \multicolumn{3}{c}{multiplications / classical int8} \\
\cmidrule(lr){2-3}\cmidrule(l){4-6}
model & activations & weights & entrywise & rows & columns \\
\midrule
Qwen2.5-14B-Instruct & \NIclrRebEFourQOneFourOfA & \NIclrRebEFourQOneFourOfB & \NIclrRebEFourQOneFourEntry & \NIclrRebEFourQOneFourRows & \NIclrRebEFourQOneFourCols \\
Llama-3.1-8B-Instruct & \NIclrRebEFourLlEightOfA & \NIclrRebEFourLlEightOfB & \NIclrRebEFourLlEightEntry & \NIclrRebEFourLlEightRows & \NIclrRebEFourLlEightCols \\
OLMo-2-7B-Instruct & \NIclrRebEFourOlSevenOfA & \NIclrRebEFourOlSevenOfB & \NIclrRebEFourOlSevenEntry & \NIclrRebEFourOlSevenRows & \NIclrRebEFourOlSevenCols \\
Qwen3-8B & \NIclrRebEFourQThreeEightOfA & \NIclrRebEFourQThreeEightOfB & \NIclrRebEFourQThreeEightEntry & \NIclrRebEFourQThreeEightRows & \NIclrRebEFourQThreeEightCols \\
\bottomrule
\end{tabular}}
\end{table}}

\paragraph{\rev{The first kernels' timings.}}
\rev{The timings of our first kernels in Section~\ref{sec:cost} multiply} two $4096\times4096$ matrices on one H20 GPU from prequantized codes and scales.
The slowdowns relative to classical int8 at the same code bound and group size compare the fastest configurations found for the two kernels in the respective tile sweeps of one run.
Only configurations whose fast and classical outputs have identical bit patterns are timed; the same check must reject a deliberate sign flip in one decoding coefficient.
The classical kernel is separately checked against a grouped fp32 reference at relative maximum error below $10^{-5}$.
Each configuration receives eight warmup runs followed by 30 CUDA-event measurements, each preceded by zeroing a 96\,MB buffer to evict the cache; we report their upper median.
In that sweep, one level of Strassen at code bound 63 takes \NIclrCostStrassenOneQGOneTwoEight, \NIclrCostStrassenOneQGTwoFiveSix{} and \NIclrCostStrassenOneQGFiveOneTwo{} times as long as classical int8 at groups of 128, 256 and 512, and two levels at code bound \NIclrActFourStrassenTwoHeadroomBA{} take \NIclrCostTwoLevelVsClassical{} times as long at groups of 256; in every configuration a call spans one group, within the conditions of Proposition~\ref{prop:certificate}.
Part of this cost comes from splitting a call into short block products: when the executor of our fast kernels runs the \NIclrActFourStrassenTwoMultsClassical{} block products of the classical algorithm, with nothing of the fast algorithm in it, it already takes \NIclrCostExecutorFloor{} times as long as our classical int8 kernel at groups of 256.
Our classical int8 kernel, which applies the group scales as well, takes \NIclrCostAnchorClassicalUsMin--\NIclrCostAnchorClassicalUsMax{}~$\mu$s at this shape across the three group sizes, while one call of the cuBLASLt int8 matrix multiplication that PyTorch exposes, which applies none, takes \NIclrCostAnchorVendorUs{}~$\mu$s in the same run.

\paragraph{The order of the products.}
The spreads reported here are over the orders we timed at one fixed kernel configuration for each algorithm, for two $4096\times4096$ matrices at groups of 256 and code bound \NIclrActFourStrassenTwoHeadroomBA: the slowest order takes \NIclrCostOrderBandSxSPooled{} times as long as the fastest for two-level Strassen, and \NIclrCostOrderBandCxSPooled{} times for a hybrid of one Strassen level and one classical level.
In two further pools, the fastest order could be found without timing.
For each algorithm we drew \NIclrCostRegretNFresh{} orders absent from the timings behind those spreads, and compiled and timed each, recording the local memory per thread that the compiler allocated to it.
The rule, fixed before the orders were drawn, picks the order with the least local memory, which needs only the compilation.
In both pools it picked the fastest of the \NIclrCostRegretNFresh{}: the picked order's time over the pool's best is \NIclrCostRegretSxS{} for two-level Strassen and \NIclrCostRegretCxS{} for the hybrid.
Since every order computes the same bits, such a search cannot change what the model computes; \citet{yang2026taming} likewise confine GPU autotuning to configurations that compute the same bits.

\rev{\paragraph{Where the time went.}
To separate what the construction costs from what our first kernels cost, we re-timed them at the shape above, on a 12-configuration tile search that reproduces the submitted ratios, together with variants that drop one step at a time; the variants without block sums or without the decoding, which adds each product into the output blocks it feeds, serve only as timing instruments and do not compute the product.
Table~\ref{tab:ladder} gives the ladder.
For two levels, forming the block sums inside the kernel multiplies the time by \NIclrRebEThreeSTwoFullOverNoenc, with \NIclrRebEThreeSTwoFullSpills{} registers spilled in the fastest configuration, and that decoding by \NIclrRebEThreeSTwoFullOverNodec; splitting a call into short block products costs \NIclrRebEThreeSTwoCxc{} times classical int8's time even for the classical algorithm, while the \NIclrActFourStrassenTwoMultsFast{} short products alone take \NIclrRebEThreeSTwoNoencNodec{} times.
Our optimized kernels, described below, form the block sums in a separate pass; at the same shape they take \NIclrRebEThreeOctTwoOverCl{} (two levels, including that pass) and \NIclrRebEThreeOctOneOverCl{} (one level) of the time of classical int8 with the same scales, one per token row and per output channel.

\begin{table}[h]
\caption{\rev{The submitted kernels re-timed at $4096\times4096$, groups of 256: time over classical int8 for the classical algorithm's short products through the same executor, the fast algorithm's short products alone, without block sums, without the decoding, and in full. The middle three are timing instruments that do not compute the product.}}
\label{tab:ladder}
\centering
\small
\setlength{\tabcolsep}{5pt}
\rev{\begin{tabular}{@{}l rrrrr@{}}
\toprule
 & classical, short products & products alone & no block sums & no decoding & full \\
\midrule
two levels & \NIclrRebEThreeSTwoCxc & \NIclrRebEThreeSTwoNoencNodec & \NIclrRebEThreeSTwoNoenc & \NIclrRebEThreeSTwoNodec & \NIclrRebEThreeSTwoFull \\
one level & \NIclrRebEThreeSOneCxc & \NIclrRebEThreeSOneNoencNodec & \NIclrRebEThreeSOneNoenc & \NIclrRebEThreeSOneNodec & \NIclrRebEThreeSOneFull \\
\bottomrule
\end{tabular}}
\end{table}}
\paragraph{\rev{The optimized kernels.}}
\rev{The one-level kernel is written in Gluon, the lower-level language of Triton 3.7.1. Each thread block computes one $64\times32$ output tile over the whole inner dimension in steps of 256, three blocks share a streaming multiprocessor, and the seven block products run back to back in a fixed order, fed by one asynchronous load pipeline that prefetches across their boundaries; each output quadrant is written as soon as its integer value is complete.
In the engine it reads the weight block sums and the seven activation block sums, which are formed right after the rotation and quantization. We time it alone at 8,192 token rows on \NIclrRebKOneKTwoFiveSixSustNshapes{} projection shapes of three of the models, Qwen3.8-27B, Llama-3.1-8B-Instruct and Llama-3.3-70B-Instruct, each in a 15-second loop whose last two thirds are timed, against the same GPU's classical int8 kernel run before and after it. It takes \NIclrRebKOneKTwoFiveSixSustMin--\NIclrRebKOneKTwoFiveSixSustMax{} of classical int8's time, returns classical int8's output bits on every shape, and no longer does once the sign of one decoding coefficient is flipped.

The two-level kernel computes the \NIclrActFourStrassenTwoMultsFast{} block products of two-level Strassen with the same tile, step and occupancy; it keeps the inner level's partial sums in registers and the outer level's nine in a scratch buffer in global memory that belongs to its thread block.
A complete call also quantizes the activations and writes their block sums, in one pass for the scales and one for the sums; we time it against quantization followed by classical int8 at code bound \NIclrActFourStrassenTwoHeadroomBA{} and 8,192 token rows on the projection shapes of all four models, summing each layer's projections. Per layer it takes \NIclrRebKTwoLayerMin--\NIclrRebKTwoLayerMax{} of classical int8's time, and all \NIclrRebKTwoNcells{} timed cells return classical int8's output bits.}

\paragraph{\rev{On an H100.}}
\rev{Whether one level of Strassen saves time on a GPU depends on how efficiently the GPU runs half-size products: one level replaces a full-size product by seven of them, one fewer than the classical algorithm's eight.
On an NVIDIA H100 SXM GPU, eight half-size int8 products on independent random operands take \NIclrRebHOneZeroZeroSplitEightTriton{} times as long as the full-size product with our classical int8 kernel in Triton and \NIclrRebHOneZeroZeroSplitEightCublaslt{} times with the cuBLASLt int8 matrix multiplication, so little is left of the eighth that one level saves.
We time them at 8,192 token rows on one GPU's share of a Llama-3.3-70B-Instruct layer under two-way tensor parallelism, summed over its four projections, replaying each call from a CUDA graph.
Our earlier one-level kernels, one in Gluon and one in Triton, step through the inner dimension by 128.
A complete call also quantizes the activations and forms their block sums, and we time the faster of the two on each shape against quantization followed by the faster of our classical int8 kernels.
On the H20 the complete calls take a median of \NIclrRebKOneKOneTwoEightCSevenZeroSixMedian{} of that time over \NIclrRebHOneZeroZeroKOneNcells{} projection shapes of Llama-3.1-8B-Instruct, Qwen3-8B and Qwen3.8-27B at 8,192 token rows.
Run unchanged on the H100, they take a median of \NIclrRebHOneZeroZeroKOneMedian{} times that time on the same shapes and \NIclrRebHOneZeroZeroKOneLSevenZero{} times on the 70B layer, and every call returns classical int8's output bits.
Since the two compute the same bits, classical int8 is the faster choice on this GPU.}

\paragraph{\rev{Inside SGLang.}}
\rev{We run SGLang 0.5.11.dev7003 \citep{zheng2024sglang} on H20 GPUs: Qwen3-8B and Llama-3.1-8B-Instruct on one GPU per server, Qwen3.8-27B and Llama-3.3-70B-Instruct on two by tensor parallelism.
The query--key--value projection runs classical int8 on activations at code bound 127, and every other projection runs the one-level kernel at code bound 63 with $64\times32$ tiles on the 8B and 27B models and $64\times64$ tiles on the 70B model; the 8B models store the weight block sums once, the larger models form them in every call.
The weights are GPTQ codes \citep{frantar2023gptq} of the rotated weights at code bound 63, calibrated on the first 2,048 tokens of 16 books, and one kernel applies the online block Hadamard rotation and quantizes each token's activations.
Every server prefills in chunks of 8,192 tokens and caps the key--value cache at 160,000 tokens, and the comparison path is SGLang's own FP8 path.

The workloads send random token ids: prefill gives each request 6,144 input tokens and asks for one output token, long asks for 128, and chat gives 768 and asks for 256,
with 192, 192 and 512 requests on the 8B models (96, 96 and 256 on the others) and 16, 16 and 64 in flight;
in the agent workload eight agents run six turns, each adding 4,096 new tokens to its cached 8,192-token prefix and generating 200.
Each arm runs twice on one node, or four times on two nodes in opposite orders for the 27B model, and every int8 run reproduces the output tokens and log-probabilities of the arm's first run on fixed check prompts.
The per-token NLL increase over bf16 is measured in the engine on 32 books of the PG-19 training split \citep{rae2019compressive}, tokens 2,048 to 10,367 of each, disjoint from the calibration books, as the mean over books of each book's mean drop in log-probability.
Table~\ref{tab:sglang} lists the results; in its lower block both paths also run SGLang's FP8 attention on an FP8 key--value cache.

To see where the time goes, we profile one 8,192-token prefill batch of each 8B model under the certified path, under its twin that runs the same codes with classical int8, and under SGLang's FP8 path: the fast projections take \NIclrRebKOneBatchEightBFastprojMin--\NIclrRebKOneBatchEightBFastprojMax{} of the GPU time of the same projections in the twin, the preparation of the activations (rotation, quantization and block sums) takes \NIclrRebKOneBatchEightBPrepExcessPctMin--\NIclrRebKOneBatchEightBPrepExcessPctMax\% of FP8's batch time beyond FP8's own quantizer, and the whole batch takes \NIclrRebKOneBatchEightBTotalVsFpEightMin--\NIclrRebKOneBatchEightBTotalVsFpEightMax{} of FP8's GPU time.
In a separate run of the two 8B models with the same workloads, the certified path's throughput exceeds that of its twin on all four workloads, by \NIclrRebEOneFivePrefillLongGainPctMin--\NIclrRebEOneFivePrefillLongGainPctMax\% on prefill and long and by \NIclrRebEOneFiveChatAgentGainPctMin--\NIclrRebEOneFiveChatAgentGainPctMax\% on chat and agent, where the two launches of either path differ by up to \NIclrRebEOneFiveChatAgentIntEightSpreadPctMax\%.}

\begin{table}[t]
\centering
\small
\caption{\rev{SGLang on H20 GPUs: throughput gain of the certified int8 path over SGLang's FP8 path per workload, and the per-token NLL increase over bf16 (nats) of each path. Upper block: bf16 attention; lower block: both paths with FP8 attention and an FP8 key--value cache.}}
\label{tab:sglang}
\rev{\begin{tabular}{lrrrrrr}
\toprule
model & prefill & long & chat & agent & int8 NLL & FP8 NLL \\
\midrule
Qwen3-8B & \NIclrRebSglBfOneSixAttnQThreePrefillGainPct\% & \NIclrRebSglBfOneSixAttnQThreeLongGainPct\% & \NIclrRebSglBfOneSixAttnQThreeChatGainPct\% & \NIclrRebSglBfOneSixAttnQThreeAgentGainPct\% & \NIclrRebNllBfOneSixAttnQThreeIntEight & \NIclrRebNllBfOneSixAttnQThreeFpEight \\
Llama-3.1-8B-Instruct & \NIclrRebSglBfOneSixAttnLlEightPrefillGainPct\% & \NIclrRebSglBfOneSixAttnLlEightLongGainPct\% & \NIclrRebSglBfOneSixAttnLlEightChatGainPct\% & \NIclrRebSglBfOneSixAttnLlEightAgentGainPct\% & \NIclrRebNllBfOneSixAttnLlEightIntEight & \NIclrRebNllBfOneSixAttnLlEightFpEight \\
Qwen3.8-27B & \NIclrRebSglBfOneSixAttnQTwoSevenPrefillGainPct\% & \NIclrRebSglBfOneSixAttnQTwoSevenLongGainPct\% & \NIclrRebSglBfOneSixAttnQTwoSevenChatGainPct\% & \NIclrRebSglBfOneSixAttnQTwoSevenAgentGainPct\% & \NIclrRebNllBfOneSixAttnQTwoSevenIntEight & \NIclrRebNllBfOneSixAttnQTwoSevenFpEight \\
Llama-3.3-70B-Instruct & \NIclrRebSglBfOneSixAttnLlSevenZeroPrefillGainPct\% & \NIclrRebSglBfOneSixAttnLlSevenZeroLongGainPct\% & \NIclrRebSglBfOneSixAttnLlSevenZeroChatGainPct\% & \NIclrRebSglBfOneSixAttnLlSevenZeroAgentGainPct\% & \NIclrRebNllBfOneSixAttnLlSevenZeroIntEight & \NIclrRebNllBfOneSixAttnLlSevenZeroFpEight \\
\midrule
Qwen3-8B & \NIclrRebSglFpEightAttnQThreePrefillGainPct\% & \NIclrRebSglFpEightAttnQThreeLongGainPct\% & \NIclrRebSglFpEightAttnQThreeChatGainPct\% & \NIclrRebSglFpEightAttnQThreeAgentGainPct\% & \NIclrRebNllFpEightAttnQThreeIntEight & \NIclrRebNllFpEightAttnQThreeFpEight \\
Llama-3.1-8B-Instruct & \NIclrRebSglFpEightAttnLlEightPrefillGainPct\% & \NIclrRebSglFpEightAttnLlEightLongGainPct\% & \NIclrRebSglFpEightAttnLlEightChatGainPct\% & \NIclrRebSglFpEightAttnLlEightAgentGainPct\% & \NIclrRebNllFpEightAttnLlEightIntEight & \NIclrRebNllFpEightAttnLlEightFpEight \\
Llama-3.3-70B-Instruct & \NIclrRebSglFpEightAttnLlSevenZeroPrefillGainPct\% & \NIclrRebSglFpEightAttnLlSevenZeroLongGainPct\% & \NIclrRebSglFpEightAttnLlSevenZeroChatGainPct\% & \NIclrRebSglFpEightAttnLlSevenZeroAgentGainPct\% & \NIclrRebNllFpEightAttnLlSevenZeroIntEight & \NIclrRebNllFpEightAttnLlSevenZeroFpEight \\
\bottomrule
\end{tabular}}
\end{table}

\section{Reproducibility details}\label{app:repro}

\paragraph{From stored outputs to printed numbers.}
Every result the paper prints as a numeral is generated from the stored output of the program that produced it, by a rule that reads the value from that output or computes it; the \NIclrActOneQOneFourOrbitSpread-fold range of Section~\ref{sec:magnitude}, for example, is read from the stored summary of the runs of the sign variants and rounded to one decimal.
Table~\ref{tab:repro} gives, for each family of results, the arithmetic that produced it and the hardware it ran on.

\begingroup
\setlength{\tabcolsep}{4pt}
\setlength{\LTcapwidth}{\textwidth}
\begin{longtable}{@{}>{\small\raggedright\arraybackslash}p{0.31\textwidth} >{\small\raggedright\arraybackslash}p{0.15\textwidth} >{\small\raggedright\arraybackslash}p{0.37\textwidth} >{\small\raggedright\arraybackslash}p{0.09\textwidth}@{}}
\caption{The arithmetic and hardware behind each family of results, with the sections and appendices where the results appear. Outside the layers a row names, the models run in bf16, and likelihoods are computed from their logits in fp32\rev{; inside SGLang, attention and the key--value cache run in bf16 or FP8 as Appendix~\ref{app:cost} states, and likelihoods are the engine's log probabilities}.}\label{tab:repro}\\
\toprule
results & where & arithmetic & hardware \\
\midrule
\endfirsthead
\caption[]{(continued)}\\
\toprule
results & where & arithmetic & hardware \\
\midrule
\endhead
\bottomrule
\endfoot
\rev{Tests of prefix invariance on OpenBookQA and ARC-Easy, windows, packed requests, generated answers and identification} & \ref{sec:support}, \ref{sec:reach}; \ref{app:scoring}, \ref{app:exemplars}, \ref{app:windows}, \ref{app:packed}\rev{, \ref{app:packgen}} & fast realizations: block sums rounded to FP8 and multiplied by DeepGEMM's FP8 kernel, products combined in fp32; controls: the bf16 model and DeepGEMM's FP8 kernel on the whole product & GPU \\
\addlinespace[2pt]
\rev{Existing fast kernels: FalconGEMM in these tests, SubCuber on random operands} & \rev{\ref{sec:support}, \ref{sec:reach}; \ref{app:packed}--\ref{app:published}} & \rev{FalconGEMM's FP8 path in place of the fast realizations, its block sums rounded to FP8 with one scale per token row and per 128 channels; SubCuber's fp16 kernels, against cuBLAS in fp16} & \rev{GPU} \\
\addlinespace[2pt]
\rev{Packed requests inside SGLang} & \rev{\ref{sec:reach}; \ref{app:engine}} & \rev{the engine's own batching, with the bf16 model, SGLang's FP8 path, the certified int8 path and our fast FP8 Strassen adapter} & \rev{GPU} \\
\addlinespace[2pt]
How a later row reaches an earlier output & \ref{sec:reach}; \ref{app:rowmap} & FP8 rounding simulated by casting to FP8 (e4m3) and back; products in fp32 & CPU \\
\addlinespace[2pt]
Sign variants on seven models, four variants rescored with the later text replaced by filler, summation orders and repeated runs & \ref{sec:magnitude}; \ref{app:models}, \ref{app:order} & block sums rounded to FP8 and multiplied by DeepGEMM's FP8 kernel, products combined in fp32 & GPU \\
\addlinespace[2pt]
The outlier ratio of the counterexample & \ref{app:models} & the bf16 model's activations, row norms in fp32 & GPU \\
\addlinespace[2pt]
\rev{The cause of the spread on Qwen2.5-14B-Instruct, the rescaling on Qwen2.5-3B and the single-layer controls on Qwen3-8B} & \rev{\ref{app:mechanism}} & \rev{the FP8 schedule of Section~\ref{sec:support}, with one row's error injected, a row permutation or a per-row rescaling; the attribution recomputes that error in an fp32 emulation with either operand's FP8 rounding disabled} & \rev{GPU} \\
\addlinespace[2pt]
Stability criteria, code bounds and headroom, and the code bounds of the AlphaTensor algorithms & \ref{sec:magnitude}, \ref{sec:certificate}, \ref{sec:correction}; \ref{app:criteria}\rev{, \ref{app:construction}} & exact arithmetic on the coefficients & CPU \\
\addlinespace[2pt]
Synthetic ensembles and the rotation & \ref{app:criteria}, \ref{app:order} & FP8 rounding simulated by casting to FP8 (e4m3) and back, products in fp32, against a float64 reference & CPU \\
\addlinespace[2pt]
Certified construction on tiles: exact trials, sign variants, pairs of an earlier and a later row, and the control that breaks condition~(iii) & \ref{sec:certificate} & integer products in int64, rescaled in float64 and accumulated in fp32; the FP8 comparison and the control round their block sums in simulated floating point & CPU \\
\addlinespace[2pt]
Checks of replaced layers inside four models: every invocation compared bit for bit with classical int8 & \ref{sec:certificate}; \ref{app:construction} & int8 arithmetic emulated exactly in fp32, with TF32, the GPU's reduced-precision fp32 mode, disabled & GPU \\
\addlinespace[2pt]
The OpenBookQA test on four models and model quality at code bounds \NIclrActFourStrassenTwoHeadroomBA{} and 127 & \ref{sec:certificate}, \ref{sec:quality}; \ref{app:construction} & the classical int8 operator the certified realization computes, emulated exactly in fp32; references: the bf16 model and DeepGEMM's FP8 kernel & GPU \\
\addlinespace[2pt]
\rev{Model quality of the deployed configuration} & \rev{\ref{sec:quality}; \ref{app:construction}} & \rev{the classical int8 operator of its specification, with rotated activations and GPTQ weight codes, and every fast call checked against it bit for bit; references as above} & \rev{GPU} \\
\addlinespace[2pt]
\rev{FP8 with the deployed configuration's rotation and GPTQ codes} & \rev{\ref{app:construction}} & \rev{DeepGEMM's FP8 kernel and the per-token FP8 format with rotated activations, GPTQ e4m3 weight codes or both, on the quality protocol; SGLang's FP8 path with the same codes, in the engine} & \rev{GPU} \\
\addlinespace[2pt]
The overflow correction on three layer slices: its bitwise check and its multiplication counts & \ref{sec:correction}; \ref{app:construction}, \ref{app:cost} & the corrected and the classical computation in fp32, compared bit for bit; multiplications counted over the entries, or the rows and columns, that overflow & CPU \\
\addlinespace[2pt]
\rev{The correction's multiplications over every linear layer except the output layer of four models} & \rev{\ref{sec:cost}; \ref{app:cost}} & \rev{the bf16 model's activations and the weights quantized at code bound 127 as in the model runs; multiplications counted over the entries, or the rows and columns, that overflow} & \rev{GPU} \\
\addlinespace[2pt]
\rev{Timings of the first and the optimized kernels}, configuration sweeps and product orders & \ref{sec:cost}; \ref{app:cost} & int8 Triton \rev{and Gluon} kernels with int32 accumulation, timed against \rev{classical int8 at the same specification}, with PyTorch's int8 matrix product as a timing reference & GPU \\
\addlinespace[2pt]
\rev{Throughput, per-token likelihood and a profiled prefill batch inside SGLang} & \rev{\ref{sec:cost}; \ref{app:cost}} & \rev{SGLang with our plugin: the one-level certified kernel at code bound 63 and classical int8 for the query--key--value projection, against SGLang's FP8 path and against classical int8 on the same codes} & \rev{GPU} \\
\addlinespace[2pt]
Multiplication counts without the correction & \ref{sec:certificate}, \ref{sec:cost} & counted from the algorithms & -- \\
\end{longtable}
\endgroup

\paragraph{What holds on any hardware.}
Three kinds of result are exact and so hold on any hardware: the coefficient quantities, the multiplication counts, and the certified operator identities (the bitwise equality of the certified construction and classical int8 on the tiles and in every checked invocation inside the models).
Proposition~\ref{prop:certificate} guarantees that identity for any implementation whose integer arithmetic is exact and whose quantization, rescaling and accumulation follow the same specification $c$, and the models' fp32 emulation is such an implementation because every integer intermediate stays below $2^{24}$.

\paragraph{Software.}
The available environment records list the following software: the runs of the sign variants used PyTorch 2.11.0 \citep{paszke2019pytorch} with CUDA 12.8, the repeated runs PyTorch 2.13.0 with CUDA 12.9 and DeepGEMM 0.1.5.post3, the quality measurements PyTorch 2.13.0 with CUDA 13.0, the packed requests transformers 5.3.0 \citep{wolf2019huggingface}, the timings PyTorch 2.13.0 with CUDA 12.9 and Triton 3.7.1\rev{, and the serving runs SGLang 0.5.11.dev7003 with sgl\_kernel 0.4.7}.
The timings and the quality measurements record NVIDIA H20 GPUs\rev{, apart from the timings on an H100 in Appendix~\ref{app:cost}, which used CUDA 13.0}.

\paragraph{Models and data.}
The experiments on one model use Qwen2.5-14B-Instruct, except the record of key and value projections in Appendix~\ref{app:packed}, which uses Qwen2.5-3B.
The window test adds Qwen2.5-3B, Llama-3.1-8B-Instruct, OLMo-2-7B-Instruct (the 1124 release) and Qwen3-8B, and the packed requests use Qwen2.5-14B-Instruct, Qwen2.5-3B and OLMo-2-7B-Instruct; the sign variants on seven models add Qwen2.5-7B-Instruct and Yi-6B to the window test's five; Section~\ref{sec:construction}'s runs use Qwen2.5-14B-Instruct, Llama-3.1-8B-Instruct, OLMo-2-7B-Instruct and Qwen3-8B; the overflow correction's slices come from Qwen2.5-14B-Instruct, Qwen3-8B and Qwen2.5-3B; and the counterexample involving the outlier ratio uses Qwen2.5-32B-Instruct, Mistral-7B-v0.1 and Llama-2-7b-hf.
\rev{The generated answers add Llama-3.1-8B-Instruct; the test inside SGLang uses Llama-3.1-8B-Instruct and Qwen3-8B, and the serving measurements Qwen3-8B, Llama-3.1-8B-Instruct, Qwen3.8-27B and Llama-3.3-70B-Instruct; the deployed configuration's quality and the overflow counts over every linear layer except the output layer use the four models of Section~\ref{sec:construction}.}
The multiple-choice tests use the first \NIclrActTwoQOneFourObqaItems{} test items of OpenBookQA and ARC-Easy, with C4 validation text as the filler (Appendix~\ref{app:scoring}); the window test (Appendix~\ref{app:windows}) and the rescore of four variants in Section~\ref{sec:magnitude} alternate WikiText-2 test and C4 validation text; and the perplexity and quality measurements over whole sequences use 2048-token chunks of the WikiText-2 test set (Appendices~\ref{app:models} and~\ref{app:construction}).
\rev{The generated answers use questions from the NQ-open validation split with WikiText-2 or C4 text as the other request, the test inside SGLang uses WikiText-2 test text, and the GPTQ weight codes are calibrated on books of PG-19 and the per-token likelihood in the engine is measured on other books of it, while the throughput workloads send random token ids (Appendices~\ref{app:construction} and~\ref{app:cost}).}

\section{\texorpdfstring{\rev{Terms}}{Terms}}\label{app:terms}

\begin{table}[h]
\centering
\caption{\rev{Terms used throughout the paper, with the section that defines each.}}
\label{tab:terms}
\small
\rev{\begin{tabular}{@{}p{0.27\linewidth}p{0.57\linewidth}l@{}}
\toprule
term & meaning & defined in \\
\midrule
fast algorithm $(U,V,W)$ & computes $C=AB$ from $R$ products of block sums of $A$ and $B$, combined with weights $W$; the classical algorithm uses $R=mkn$ & \S\ref{sec:signvariants} \\
block sum & a combination $\sum_{i,l}u_{r,il}A_{il}$ of blocks; when the blocks hold different token rows, it mixes tokens & \S\ref{sec:signvariants} \\
sign variant & the coefficients obtained by flipping whole block rows or columns with diagonal sign matrices; it computes the same exact product & \S\ref{sec:signvariants} \\
realization & an algorithm's coefficients together with how it is quantized and executed: which values are rounded, to which format, with which scales, and in which order & \S\ref{sec:signvariants}, \S\ref{sec:spec} \\
quantization specification $c$ & the rules that turn operands into integer codes (code bounds $b_A$, $b_B$, group length $g$, scale ownership, scale rule, rounding) and group results into the output (precision and order of rescaling and summation, output format) & \S\ref{sec:spec} \\
classical int8 operator at $c$ & the reference: it multiplies each group's codes exactly and rescales and sums the group results as $c$ prescribes & \S\ref{sec:spec} \\
row-local & output row $t$ depends only on row $t$ of $A$, its codes and its scales & \S\ref{sec:spec} \\
allowed prefix & the text of a position's own sequence up to that position & \S\ref{sec:prefix} \\
prefix invariance & two inputs that agree on a position's allowed prefix give that position the same outputs & \S\ref{sec:prefix} \\
unrepaired, repaired \#1 and \#2 & the FP8 realizations of Section~\ref{sec:support}: two-level Strassen; one level of Strassen; two levels of the algorithm that $\Phi$ ranks most accurate. All round token-mixing block sums to FP8 and are our compositions of DeepGEMM's routines & \S\ref{sec:support} \\
row-local controls & the bf16 model and the deployed FP8 kernel, which scales each token row on its own and forms no block sums & \S\ref{sec:support} \\
call & one computation of the integer product of a tile of $A$ and a tile of $B$ with the fast algorithm & \S\ref{sec:certificate} \\
certified realization & a realization that meets conditions (i)--(iii) of Proposition~\ref{prop:certificate}; it computes the classical int8 operator at $c$ bit for bit & \S\ref{sec:certificate} \\
code bound $b$ & the largest code magnitude that $c$ allows & \S\ref{sec:spec}, \S\ref{sec:certificate} \\
\bottomrule
\end{tabular}}
\end{table}

\end{document}